\documentclass[10pt,a4paper]{article}
\usepackage[margin=1.16in]{geometry}

\usepackage{amssymb, amsfonts, amsbsy, latexsym}
\usepackage{graphicx}
\usepackage{tikz}
\usepackage{color}
\usepackage{colortbl}
\usepackage[english]{babel}
\usepackage{url}
\usepackage{amsmath}

\usepackage{bm}

\usepackage{scalerel}

\usepackage{relsize}

\usepackage{amsthm}

\usepackage{csquotes}

\newtheorem{theorem}{Theorem}
\newtheorem{definition}[theorem]{Definition}
\newtheorem{proposition}[theorem]{Proposition}
\newtheorem{remark}[theorem]{Remark}
\newtheorem{lemma}[theorem]{Lemma}
\newtheorem{example}[theorem]{Example}

\newtheorem{approximation problem}[theorem]{Approximation problem}

\newtheorem{corollary}[theorem]{Corollary}

\newtheorem{claim}[theorem]{Claim}

\newcommand{\ip}[2]{\left\langle#1,#2\right\rangle}
\newcommand{\abs}[1]{\left|#1\right|}
\newcommand{\norm}[1]{\left\|#1\right\|}

\def\tf{\tilde{f}}

\def\d{\delta}
\def\e{\epsilon}

\def\l{\lambda}
\def\k{\kappa}
\def\a{\alpha}

\def\NN{\mathbb{N}}

\def\cF{\mathcal{F}}

\def\cL{\mathcal{L}}

\def\CC{\mathbb{C}}
\def\NN{\mathbb{N}}
\def\ZZ{\mathbb{Z}}

\def\RR{\mathbb{R}}

\def\DD{\boldsymbol{\Delta}}

\DeclareMathOperator*{\argmax}{argmax}

\title{Transferability of Spectral Graph Convolutional Neural Networks}

\author{%
Ron Levie$^*$, Wei Huang$^{\dag}$, Lorenzo Bucci$^{\dag}$, Michael Bronstein$^{\dag \S \ddag}$, Gitta Kutyniok$^{* \P}$ \vspace{2mm} \\ 
 $*$ Ludwig Maximilian University of Munich, $\dag$Università della Svizzera italiana,  \\
$\P$University of Troms{\o} , ${\ddag}$Imperial College London, $\S$Twitter
}

\date{\vspace{-5ex}}

\begin{document}

\maketitle

\begin{abstract}
This paper focuses on spectral graph convolutional neural networks (ConvNets), where filters are defined as elementwise multiplication in the frequency domain of a graph. In machine learning settings where the dataset consists of signals defined on many different graphs, the trained ConvNet should generalize to signals on graphs unseen in the training set. It is thus important to transfer ConvNets between graphs. Transferability, which is a certain type of generalization capability, can be loosely defined as follows: if two graphs describe the same phenomenon, then a single filter or ConvNet should have similar repercussions on both graphs. This paper aims at debunking the common misconception that spectral filters are not transferable. We show that if two graphs discretize the same ``continuous'' space, then a spectral filter or ConvNet has approximately the same repercussion on both graphs. Our analysis is more permissive than the standard analysis. Transferability is typically described as the robustness of the filter to small graph perturbations and re-indexing of the vertices. Our analysis accounts also for large graph perturbations. We prove transferability between graphs that can have completely different dimensions and topologies, only requiring that both graphs discretize the same underlying space in some generic sense.
\end{abstract}

\section{Introduction}

The success of convolutional neural networks (ConvNets) on Euclidean domains ignited an interest in recent years in extending these methods to graph structured data.  In a standard ConvNet, the network receives as input a signal defined over a Euclidean rectangle, and at each layer applies a set of convolutions/filters on the outputs of the previous layer, a non linear activation function, and, optionally, pooling. 
 A graph ConvNet has the same architecture, with the only difference that now signals are defined over the vertices of graph domains, and not Euclidean rectangles. Graph structured data is ubiquitous in a range of applications, and can represent 3D shapes, molecules, social networks, point clouds, and citation networks to name a few. 
 
 In a machine learning setting, the general architecture of the ConvNet is fixed, but the specific filters to use in each layer are free parameters. In training, the filter coefficients are optimized to minimize some loss function. In some situations, both the graph and the signal defined on the graph are variables in the input space of the ConvNet. Namely, the dataset consists of many different graphs, and many different signals on these graphs. We call such a scenario a \textbf{multi-graph setting}. In multi-graph settings, if two graphs represent the same underlying phenomenon, and the two signals given on the two graphs are similar in some sense, the output of the ConvNet on both signals should be similar as well. This property is typically termed transferability, and is an essential requirement if we wish the ConvNet to generalize well on the test set, which in general consists of graphs unseen in the training set. In fact, transferability can be seen as a special type of generalization capability.
Analyzing and proving transferability is the focus of this paper.

\subsection{Convolutional neural networks}

A classical 1D convolution neural network, as described above, can be written explicitly as follows. We call each application of filters, followed by the activation function and pooling a \textbf{layer}. We consider discrete input signals ${\bf f}\in \RR^{d_1}$, seen as the samples of a continuous signal $f:\RR\rightarrow\RR$ at $d_1$ sample points.
In each Layer $l=1,\ldots,L$ there are $K_l\in\NN$ signal channels. The convolution-operators/filters of the ConvNet map the signal channels of each Layer $l-1$ to the signal channels of Layer $l$. Moreover, as the layers increase, we consider coarser discrete signals. Namely, signals of Layer $l$ consist of $d_l$ samples, where $d_1\geq d_2 \geq\ldots\geq d_L$. 
Consider the affine-linear filters 
\[\{g^l_{k'k}\ |\ 
  {\scriptstyle k=1\ldots K_{l-1},\ 
  k'=1\ldots K_{l}}
\}\]
 of Layer $l-1$, and the matrix $A^l=\{a^l_{k'k}\}_{k'k}\in\RR^{K_{l}\times K_{l-1}}$ that mixes the $K_{l-1}\times K_{l}$ resulting output signals to the $K_{l}$ channels of Layer $l$. \textcolor{black}{Note that each $g_{k'k}^l$ denotes a convolution operator plus constant.} Denote the signals at Layer $l$ by $\{{\bf f}^{l}_{k'}\}_{k'=1}^{K_l}$. The ConvNet maps Layer $l-1$ to Layer $l$ by 
\[
\{{\bf f}^{l}_{k'}\}_{k'=1}^{K_l}= Q^{l}\Big(\rho\Big\{\sum_{k=1}^{K_{l-1}}a^l_{k'k}\ g^l_{k'k}({\bf f}^{l-1}_k)\Big\}_{k'=1}^{K_l}\Big),
\]
where $\rho:\RR\rightarrow\RR$, called the \textbf{activation function}, operates pointwise on vectors, and the \textbf{pooling operator} $Q^{l}:\RR^{d_{l-1}}\rightarrow \RR^{d_l}$ sub-samples signals from $\RR^{d_{l-1}}$ to $\RR^{d_l}$. A typical choice for $\rho$ is the ReLU function $\rho(x)=\max\{0,x\}$. The output of the ConvNet are the signals $\{{\bf f}^{L}_{k'}\}_{k'=1}^{K_L}$ at Layer $L$.

When generalizing this architecture to graphs, there is a need to extend the convolution, activation function, and pooling to graph structured data. Here, graph signals are mappings that assign to each vertex of a graph a value. The activation function operates pointwise on signals, and generalizes trivially to graph signals. For pooling, graph signals are sub-sampled to signals over coarsened graphs, typically via the Graclus algorithm \cite{Graclus} (see also \cite[Subsection 2.2]{defferrard2016convolutional}). Next, we explain how filters are generalized to graphs.

\subsection{Convolution operators on graphs}

There are generally two approaches to defining convolution operators on graphs, both generalizing the standard convolution on Euclidean domains \cite{review_new,wu2019comprehensive}. Spatial approaches generalize the idea of a sliding window to graphs. Here, the main challenge is to define a way to translate a filter kernel along the vertices of the graph, or to aggregate feature information from the neighbors of each node. Some popular examples of spatial methods are \cite{GNN_1,GNN_2,Monti2017GeometricDL}. Spectral methods are inspired by the convolution theorem in Euclidean domains, that states that convolution in the spatial domain is equivalent to pointwise multiplication in the frequency domain. The challenge here is to define the frequency domain and the Fourier transform of graphs. The basic idea is to define the graph Laplacian, or some other graph operator that we interpreted as a shift operator, and to use its eigenvalues as frequencies and its eigenvectors as the corresponding pure harmonics \cite{ortega2018graph}. Decomposing a graph signal to its pure harmonic coefficients is by definition the graph Fourier transform, and filters are defined by multiplying the different frequency components by different values, see Subsection \ref{Spectral convolution operators} for more details. For some examples of spectral methods we refer to
\cite{bruna2013spectral,defferrard2016convolutional,Cayley1,gama2018convolutional}.
Additional references for both methods can be found in \cite{wu2019comprehensive}.

One typical motivation for favoring spatial methods is the claim that spectral methods are not transferable, and thus do not generalize well on graphs unseen in the training set. The goal in this paper is to debunk this misconception, and to show that state-of-the-art spectral graph filtering methods are transferable.
This paper does not argue against spatial methods, but shows the potential of spectral approaches to cope with datasets having varying graphs. We would like to encourage researches to reconsider spectral methods in such situations. 
Interestingly, \cite{ARMA_ConvNet} obtained state-of-the-art results using spectral graph filters on variable graphs, without any modification to compensate for the ``non-transferability''.

\subsection{Stability of spectral methods}

A necessary condition of any reasonable definition of transferability is stability. Namely, given a filter, if the topology of a graph is perturbed, then the filter on the perturbed graph is close to the filter on the un-perturbed graph. Without stability it is not even possible to transfer a filter from a graph to another very close graph, and thus stability is necessary for transferability. 
Previous work studied the behavior of graph filters with respect to variations in the graph. \cite{segarra2017optimal} provided numerical results on the robustness of polynomial graph filters to additive Gaussian perturbations of the eigenvectors of the graph Laplacian. Since the eigendecomposition is not stable to perturbations in the topology of the graph, this result does not prove robustness to such perturbations. \cite{isufi2017filtering}  showed that the expected graph filter under random edge losses is equal to the accurate output. However, \cite{isufi2017filtering} did not bound the error in the output in terms of the error in the graph topology. \cite{Stab_Diff} studied the stability with respect to diffusion distance of diffusion scattering transforms on graphs, a graph version of the popular scattering transforms, which are pre-defined Euclidean domain ConvNets \cite{Mallat2}. \cite{Stab_Diff2} also studied stability of graph scattering transforms, in terms of perturbations in the Laplacian eigenvectors and vertex permutations. Recently, \cite{Trans00} studied stability properties of spectral graph filters of a fixed number of vertices. However, in \cite[Theorems 2 and 3]{Trans00} the assumption that the relative error matrix is normal and is close to a scaled identity matrix is restrictive, and not satisfied in the generic case. In particular, only perturbations which are approximately a multiplication of all of the edge weights by the same scalar are considered in these theorems. A similar restriction is implicit in the analysis of \cite{Stab_Diff3}, which studied stability of graph scattering transforms.
%
\cite{surface_CVPR} analyzed the stability of a special type of ConvNet on triangle meshes, where filtering is pre-defined via propagating information from vertices to faces and back using the Dirac operator. The error of the ConvNet between two polygon meshes discretizing the same surface was bounded, assuming the two meshes consist of the same number of vertices. This approach to stability is reminiscent of our approach, but in our analysis we do not assume that the two graphs consist of the same number of vertices. Moreover, we consider general spectral graph ConvNets.

\subsection{Our contribution}

In the following we summarize our contribution.

\subsubsection{Theoretical settings of transferability}
\label{Theoretical settings of transferability}
We prove in this paper the stability of graph spectral filters to general perturbations in the topology. 
In fact, we present a more permissive framework of transferability, allowing to compare graphs of incompatible sizes and topologies.  We consider spectral filters as they are, and do not enhance them with any computational machinery for transferring filters. Thus, one of the main conceptual challenges is to find a way to compare two different graphs, with incompatible graph structures, from a theoretical stance. To accommodate the comparison of incompatible graphs, our approach resorts to non-graph theoretical considerations, assuming that graphs are observed from some underlying non-graph spaces. 
In our approach, graphs are regarded as discretizations of underlying corresponding ``continuous'' metric spaces. This makes sense, since a weighted graph can be interpreted as a set of points (vertices) and a decreasing function of their distances (edge weights). We can actually relax the assumption that the ``continuous space'' is metric, and consider more general topological spaces\footnote{A topological space is a generalization of a metric space, where distances are no longer defined, but continuity is defined. In metric spaces $\mathcal{M}$, continuity of functions $f:\mathcal{M}\rightarrow\RR$ is defined via an ``$\epsilon$--$\delta$'' formulation: $f$ is continuous at $x\in\mathcal M$, if for every open interval $B_{\e,f(x)}=\big(f(x)-\e,f(x)+\e\big)$ about $f(x)$, the inverse set $\{y\in\mathcal{M}\ |\ f(y)\in B_{\e,f(x)}\}$ contains some open ball $B_{\delta,x}=\{y\in\mathcal{M}\ |\ {\rm dist}(y,x)<\delta\}$ about $x$.  Topological spaces generalize the ``$\epsilon$--$\delta$'' notion of continuity by directly specifying which sets are open, without defining a notion of distance.}.
Two graphs are comparable, or represent the same phenomenon, if both discretize the same space.
This approach allows us to prove transferability under small perturbations of the adjacency matrix, but more generally, allows us to prove transferability between graphs with incompatible \textcolor{black}{sizes}. 

More generally, we consider graph sampled from general measure spaces\footnote{A measure space is informally a space in which it is possible to compute the volume of a rich collection of subsets. Using the notion of volume, it is then possible to define integration of functions defined on the measure space, and thus the root mean square error between functions is well defined.}, where the \textbf{sampling operator} is a linear mapping that takes a signal on the measure space and returns a signal on the graph. We consider a corresponding \textbf{interpolation operator}, a linear mapping that takes a signal on the graph and returns a signal on the measure space. This setting is general, and can be used to describe graphs sampled from topological spaces at sample points, graphs coarsened to smaller graphs via, e.g., the Graclus algorithm \cite{Graclus}, and graph perturbations, as discussed in Subsection \ref{Examples of transferability settings}.

The way to compare two graphs is to consider their embeddings to the ``continuous'' space they both discretize. 
For intuition, consider the special case where the ``continuous'' space is a manifold. Any manifold can be discretized to a graph/polygon-mesh in many different ways, resulting in different graph topologies. A filter designed/learned on one polygon-mesh should have approximately the same repercussion
on a different polygon-mesh discretizing the same manifold. 
\textcolor{black}{The informal term ``repercussion'' means ``the effect that a network/filter has on data.'' Choosing a rigorous definition for this term is a mathematical modeling challenge that we address as follows.}
To compare the filter on the two graphs, we consider a generic signal defined on the continuous space, and sampled to both graphs. After applying the graph filter on the sampled signal on both graphs, we interpolate the results back to two continuous signals. In our analysis we show that these two interpolated continuous signals are approximately equal (see Figure \ref{fig1} for illustration of this procedure). 

For the case of graphs sampled from topological spaces, we develop a digital signal processing (DSP) framework akin to the classical Nyquist--Shannon approach, where now analog domains are topological spaces, and digital domains are graphs. 

\subsubsection{The basic assumption of graphs discretizing topological spaces}
\label{The basic assumption of graphs discretizing metric spaces}

In the DSP setting of transferability, the assumption that graphs are discretizations of topological spaces is an ansatz, and it is important to clarify the philosophy behind this choice.
One of the fundamental challenges in studying transferability is to determine to which graph changes a network should be sensitive/discriminative and to which changes the network should generalize, or be transferable. The later changes are sometimes termed nuisances in the machine learning jargon, since the network should be designed/trained to ignore them. A network should not be transferable to all graph changes, since then the network cannot be used to discriminate between different types of graphs. On the other hand, the network should be transferable between different graphs that represent the same underlying phenomenon, even if these two graphs are not close to each other in standard measures of graph distance. The ansatz that two graphs represent the same phenomenon if both discretize the same topological space, gives us a theoretical starting point: we know to which graph changes the network should be transferable, so the problem of transferability can be formulated mathematically. What we show is that spectral graph ConvNets always generalize between graphs discretizing the same topological space, regardless of the specific form of their filters. Namely, this type of generalization is built-in to spectral graph ConvNets, and requires no training. 

The validity of this ansatz from a modeling stance is justifiable to different extents, depending on the situation. As noted above, it is natural to think of graphs as discretizations of metric spaces. Certainly, this is the case for geometric datasets like meshes, or 3D solids like molecules. 
\textcolor{black}{There is also evidence that real life networks, like World Wide Web, social networks,  protein interaction networks, and biological cellular networks, have underlying geometric structures. For example, in \cite{SelfSim} it was shown that such networks are self similar, in the sense that the coarsened version of the network has the same probability distribution of links as the fine network. Hence, a network and its coarsened version both represent the same underlying phenomenon. It is thus desirable for graph ConvNets to have the same effect on both the original and the corasened graph in some sense. Follow up works showed that networks can be seen as sampled from a latent underlying geometric space, e.g., a hyperbolic space \cite{Hyp1}, or a circle \cite{SelfSim2}. For a comprehensive survey on the underlying geometry of networks we refer the reader to \cite{NetGeo}.
}

One might even stretch the interpretation further, and consider examples like citation networks\footnote{A citation network is a graph, where each node represents a paper. Two nodes are connected by an edge if there is a citation between the papers. A graph signal is constructed by mapping the content of each paper to a vector representing this content.}, which seem non-geometric. The idea is to view citation networks as discretizations of some hypothetical underlying metric space. This metric space is the continuous limit of citation networks, where the number of papers tends to infinity. Intuitively, in the limit there is a continuum of papers, and the distance between papers models the probability for the two papers to be linked by a citation. Namely, the distance decreases to zero as the probability increases to one. We do not attempt to study or characterize this hypothetical continuous citation network, but only postulate its existence as a metric space.
In practice, the computations in training and applying filters do not use any knowledge of the underlying continuous metric space. Its existence is used only for approximation theoretic analysis. 

\textcolor{black}{Other notions of graphs approximating continuous latent spaces are possible. For example, in graphon analysis, simple graphs approximate graphons if the homomorphism densities of the graph and of the graphon are close \cite{graphon0}. In this paper we focus however on the sampling approach, leaving the graphon approach for future research.}

\subsubsection{Concept-based and principle transferability}
Graph ConvNets can manage transferability in different ways. 
First, when a graph ConvNet is shown a multi-graph training set, it can learn ``concepts'' that promote transferability. Let us call this approach \textbf{concept-based transferability}. Second, it may be the case that transferability is a mathematical law: a built-in capability of certain types of graph ConvNets, independent of their specific filters, which requires no training. This approach, that we call \textbf{principle transferability}, is the focus of this paper. 

We believe that the success of spectral graph ConvNets in multi-graph settings relies on both types of transferability. We call the accumulative effect of concept-based transferability and principle transferability \textbf{total transferability}. In this paper we prove theoretically that spectral graph ConvNets have principle transferability. We moreover demonstrate principle transferability by concocting experiments that isolate principle transferability from concept-based transferability. This is done by zero shot learning: training the network on one single graph, which prevents it from learning concepts for dealing with varying graphs, and testing the resulting network on other graphs. The performance of such a network on the new graphs only partially degrades, illustrating the effect size of principle transferability in total transferability. Moreover, in our isolated principle transferability experiment, spectral methods outperform spatial methods, which indicates that spectral methods have competitive transferability capabilities. 

\subsubsection{Overview of our transferability results}

In the following we give a high-level overview of our results.

\paragraph{The transferability inequality.}

In the transferability theory there is always an original space with an original Laplacian, from which we sample a graph and a graph Laplacian. As explained in Subsection \ref{Theoretical settings of transferability}, the original space may be a ``continuous'' measure space or a discrete graph. Let us call the original space the \textbf{continuous space}, and the original Laplacian the \textbf{continuous Laplacian}.
A \textbf{transferability error} is the error between the continuous object and the discretized object.
In Section \ref{The transferability inequality} we introduce the \textbf{transferability inequality} (Theorem \ref{main00}), a generic inequality that bounds the \textbf{transferability error of filters} in terms of the \textbf{transferability error of Laplacians} and the error entailed by sampling-interpolating, called the \textbf{consistency error}. Informally, the transferability inequality reads
\[
  transferability\ of\ filter \ \ \leq \  \  transferability\ of\ Laplacian \ \ +\ \   consistency\ error. 
  \]
%
The transferability inequality 
 asserts that if sampling and interpolation is chosen well, in the sense that sampling a continuous signal and then interpolating it results in a small error, and if the graph Laplacian approximates the continuous Laplacian, then also any graph spectral filter approximates the corresponding filter on the continuous space.

\paragraph{Sufficient conditions for transferability.}
The transferability inequality states that the transferability of a filter is small if the transferability of the Laplacian and the consistency error are small. 
In Section \ref{Transferability of spectral graph filters and ConvNets} we introduce general conditions under which the transferability of the Laplacian and the consistency error are small. 

\paragraph{Tansferability of graph spectral ConvNets.}

In Subsection \ref{Transferability of graph ConvNets} we extend the transferability results of filters to transferability of spectral ConvNets. We prove the transferability of graph spectral ConvNets under the assumption of small transferability error of the Laplacian and small consistency error in each coarsened version of the graph in the network (Theorem \ref{main_conv00} and Corollary \ref{main_conv0}). This implies that graph spectral ConvNets are appropriate in multi-graph settings. We support this claim both with basic experiments and by recalling other papers that demonstrate transferability of spectral methods in practice.

\paragraph{Transferability of graphs sampled from topological spaces.}
In Section \ref{Sampling and interpolation} we prove that the sufficient conditions for transferability are satisfied  for graphs discretizing topological spaces via sampling. To this end, we develop a digital signal processing (DSP) framework akin to the classical Nyquist--Shannon approach, where now analog domains are topological spaces, and digital domains are graphs.
Graphs are sampled from topological spaces by evaluation at sample points. We prove that graph Laplacians approximate topological space Laplacians in case the sample points satisfy some quadrature assumptions, namely, if certain integrals over the topological space can be approximated by sums over the sample points. 

\paragraph{Transferability of graphs randomly sampled from topological spaces.}
An important question that arises from the transferability inequality is if it is reasonable to assume that the right hand side of the transferability inequality is small. Another question is if the assumptions of the  DSP setting of transferability are  reasonable. The answer to these questions depends on the situation. A universal mathematical analysis is not possible, since the answer depends on how the graph dataset was constructed, how graphs were sampled and from what model, and how the graph Laplacians were chosen. To give a mathematical solution to this question, in Subsection \ref{Random sample sets satisfy the quadrature definitions} we consider a controlled setting of the data acquisition step. 
We prove that the quadrature assumptions of our DSP framework are satisfied in high probability in case the sample points of the discrete graphs are drawn randomly from the corresponding topological space (Theorem \ref{Theo:probMC}). In this scenario, spectral ConvNets are transferable in high probability.

\paragraph{Main message.} 
\emph{The concept that spectral graph ConvNets are not appropriate in situations where the data consists of many different graphs and many different signals on these graphs is a misconception. Graph spectral ConvNets are transferable both in practice and theory. If your data consists of many graphs, among other methods, you should consider spectral graph ConvNets.}

$ $

All proofs are given in the appendix.
We wish to remark that some preliminary results on stability of spectral convolutions of graphs of a fixed size were reported in \cite{tran0}.

\begin{figure}[!ht]
\centering
\includegraphics[width=0.7\linewidth]{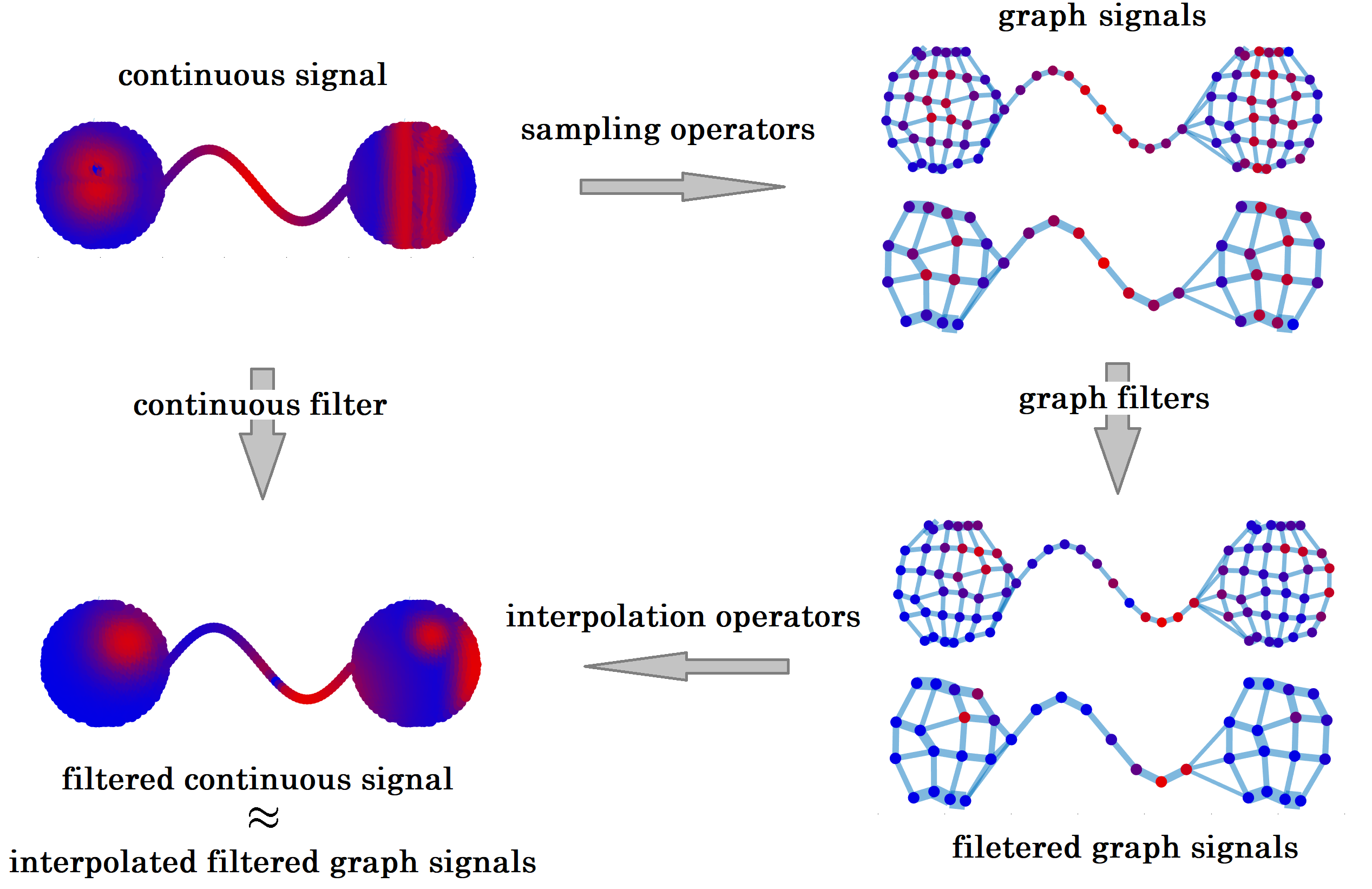}
\caption{ Diagram of the approximation procedure, illustrating how a fixed filter/ConvNet operates on a ``continuous'' topological space and two graphs discretizing it. Top left: a continuous signal on the topological space. Top right: the sampling of the continuous signal to the two graphs that discretize the topological space. Bottom right: the filter applied on both graph signals. Bottom left: the filter applied on the continuous topological space signal is approximated by the interpolation of either of the two filtered graph signals. As a result, the interpolations of the two filtered graph signals are approximately identical.}
\label{fig1}
\end{figure}

\section{Theoretical framework of graph spectral methods}

In this section we recall the theory of graph spectral methods. We show that state-of-the-art graph spectral methods are based on a \textbf{functional calculus} implementation of convolution operators, and explain the misconception of non-transferability of spectral graph filters. We last show how to use graph spectral methods for directed graphs.

\subsection{Spectral convolution operators}
\label{Spectral convolution operators}

Consider an undirected weighted graph ${\cal G}=\{\mathcal{E},\mathcal{V},{\bf W}\}$, with vertices $\mathcal{V}=\{1,\ldots,N\}$, edges $\mathcal{E}\subset \mathcal{V}^2$, and adjacency matrix ${\bf W}$. The adjacency matrix ${\bf W}=(w_{n,m})_{n,m=1}^N$ is symmetric and represents the weights of the edges, where $w_{n,m}$ is nonzero only if vertex $n$ is connected to vertex $m$ by an edge. Consider the degree matrix ${\bf D}$, defined as the diagonal matrix with entries $d_{n,n}=\sum_{m=1}^N w_{n,m}$.

The frequency domain of a graph is determined by choosing a shift operator, namely a self-adjoint operator $\DD$ that respects the connectivity of the graph. 
As a prototypical example, we consider the unnormalized Laplacian $\DD ={\bf D}-{\bf W}$, which depends linearly on ${\bf W}$.
Other examples of common shift operators are the normalized Laplacian $\DD_{\rm n}={\bf I}-{\bf D}^{-1/2}{\bf W}{\bf D}^{-1/2}$, and the adjacency matrix itself. In this paper we call a generic self-adjoint shift operator \textbf{Laplacian}, and denote it by $\DD$.
Denote the eigenvalues of $\DD$ by $\{\l_n\}_{n=1}^N$, and the eigenvectors by $\{\phi_n:V\rightarrow\CC\}_{n=1}^N$. The Fourier transform of a graph signal $f:V\rightarrow\CC$ is given by the vector of frequency intensities
\[\cF f = (\ip{f}{\phi_n})_{n=1}^N,\]
where $\ip{u}{v}$ is an inner product in $\CC^N$, e.g., the standard dot product.
The inverse Fourier transform of the vector $(v_n)_{n=1}^N$ is given by
\[\cF^* (v_n)_{n=1}^N = \sum_{n=1}^Nv_n\phi_n.\]
Since $\{\phi_n\}_{n=1}^N$ is an orthonormal basis, $\cF^*$ is the inverse of $\cF$. 
A spectral graph filter ${\bf G}$ based on the coefficients $(g_n)_{n=1}^N$ is defined by
\begin{equation}
{\bf G} f = \sum_{n=1}^Ng_n\ip{f}{\phi_n}\phi_n.
\label{filt0}
\end{equation}
Any spectral filter defined by (\ref{filt0}) is \textbf{permutation equivariant}, namely, does not depend on the indexing of the vertices. Re-indexing the vertices in the input results in the same re-indexing of vertices in the output. 

Spectral filters implemented by (\ref{filt0}) have two disadvantages. First, as shown in Subsection \ref{The misconception of non-transferability of spectral graph filters}, they are not transferable. Second, they entail high computational complexity. Formula (\ref{filt0}) requires the computation of the eigendecomposition of the Laplacian $\DD$, which is computationally demanding and can be unstable when the number of vertices $N$ is large. Moreover, there is no general ``graph FFT'' algorithm for computing the Fourier transform of a signal $f\in L^2(V)$, and (\ref{filt0}) requires computing the frequency components $\ip{f}{\phi_n}$ and their summation directly. 

\subsection{Functional calculus implementation of spectral convolution operators}

To overcome the above two limitations, state-of-the-art methods, like \cite{defferrard2016convolutional,art:ARMA,Cayley1,gama2018convolutional},  are implemented via \textbf{functional calculus}. Functional calculus is the theory of applying functions $g:\CC\rightarrow\CC$ on normal operators in Hilbert spaces $\mathcal{H}$. In the special case of a self-adjoint or unitary operator $\mathbf{T}$ in the space $\mathcal{H}$, with a discrete spectrum, $g(\mathbf{T})$ is defined by
\begin{equation}
g(\mathbf{T})f=\sum_{n} g(\l_n)\ip{f}{\phi_n}\phi_n,
\label{eq:FC1}
\end{equation}
for any vector $f$ in the Hilbert space, where $\{\l_n,\phi_n\}$ is the eigendecomposition of the operator $\mathbf{T}$. The operator $g(\mathbf{T})$ is normal for general $g:\CC\rightarrow\CC$, self-adjoint for $g:\CC\rightarrow\RR$, and unitary for $g:\CC\rightarrow e^{i\RR}$ (where $e^{i\RR}$ is the unit complex circle). 

Definition (\ref{eq:FC1}) is canonical in the following sense. In the special case where 
\[g(\l)=\frac{\sum_{l=0}^{L}c_l\l^l}{\sum_{l=0}^{L}d_l\l^l}\]
is a rational function, $g(\mathbf{T})$ can be defined in two ways. First, by (\ref{eq:FC1}), and second by compositions, linear combinations, and inversions, as 
\begin{equation}
g(\mathbf{T}) = \Big(\sum_{l=0}^{L}c_l\mathbf{T}^l\Big)\Big(\sum_{l=0}^{L}d_l\mathbf{T}^l\Big)^{-1}
\label{eq:FC_ratio}
\end{equation}
It can be shown that (\ref{eq:FC1}) and (\ref{eq:FC_ratio}) are equivalent. 

Moreover, definition (\ref{eq:FC1}) is also canonical in regard to non-rational functions. Loosely speaking, if a polynomial $p$ approximates the function $g$, then the operator $p(\mathbf{T})$ approximates the operator $g(\mathbf{T})$.  This is formulated as follows. Consider the space $PW(\l_M)$ of vectors $f$ comprising finite eigenbasis expansions
\[f=\sum_{n=0}^{M}b_n\phi_n,\]
for a fixed $M$.
If a sequence of polynomials $\{g_k\}_k$ converges to a continuous function $g$ in the sense
\[\lim_{k\rightarrow\infty}\sup_{\l\leq\abs{\l_M}}\abs{g(\l)-g_k(\l)}=0,\]
then also
\begin{equation}
\lim_{k\rightarrow\infty}\norm{g(\mathbf{T}) -g_k(\mathbf{T}) }=0,
\label{eq:FC_conv1}
\end{equation}
where the operator norm in (\ref{eq:FC_conv1}) is defined by
\[\norm{g(\mathbf{T}) -g_k(\mathbf{T}) } := \sup_{0\neq f\in PW(\l_M)}\frac{\norm{g(\mathbf{T})f -g_k(\mathbf{T})f }}{\norm{f}}.\]

When filters are defined via (\ref{eq:FC1}) with polynomial or rational function $g$, implementing spectral filters via (\ref{eq:FC_ratio}) overcomes the limitation of definition (\ref{filt0}). By relying on the spatial operations of compositions, linear combinations, and inversions, the computation of a spectral filter is carried out entirely in the spatial domain, without ever resorting to spectral computations. Thus, no eigendecomposition and Fourier transforms are ever computed. The inversions in $g(\mathbf{T})f$ involve solving systems of linear equations, which can be computed directly if $N$ is small, or by some iterative approximation method for large $N$.
Methods like \cite{defferrard2016convolutional,Kipf2017,ortega2018graph,gama2018convolutional} use polynomial filters, and \cite{art:ARMA,Cayley1,ARMA_ConvNet} use rational function filters. We term spectral methods based on functional calculus \textbf{functional calculus filters}.

\subsection{The misconception of non-transferability of spectral graph filters}
\label{The misconception of non-transferability of spectral graph filters}

The non-transferability claim is formulated based on the sensitivity of the Laplacian eigendecomposition to small perturbations in $\mathbf{W}$, or equivalently in $\DD$. Namely, a small perturbation of $\DD$ can result in a large perturbation of the eigendecomposition $\{\l_n,\phi_n\}_{n=1}^N$, which results in a large change in the filter defined via (\ref{filt0}). This claim was stated in \cite{review_new} only for spectral filters implemented via (\ref{filt0}), for which it is true. However, later papers misinterpreted this claim and applied it to functional calculus filters. \textcolor{black}{This misconception can be found in prominent surveys \cite{wu2019comprehensive}, as well as research papers, e.g.,  \cite{fey2018splinecnn, maron2018invariant, te2018rgcnn, cai2019exploiting, chen2019dagcn, bi2019graph} (the list if far from exhaustive).}
The instability argument does not prove non-transferability, since state-of-the-art spectral methods do not explicitly use the eigenvectors, and do not parametrize the filter coefficients $g_n$ via the index $n$ of the eigenvalues. Instead, state-of-the-art methods are based on functional calculus, and define the filter coefficients using a function $g:\RR\rightarrow\CC$, as $g(\l_n)$. 
The parametrization of the filter coefficients by $g$ is indifferent to the specifics of how the spectrum is indexed, and instead represents an overall response in the frequency domain, where the \textbf{value} of each frequency determines its response, and not its index.
When functional calculus filters are defined by (\ref{eq:FC1}), a small perturbation of $\DD$ that results in a perturbation of $\l_n$, also results in a perturbation of the coefficients $g(\l_n)$. It turns out that the perturbation in $g(\l_n)$ implicitly compensates for the instability of the eigendecomposition, and functional calculus spectral filters are stable.  This is seen by using the transferability inequality in a graph perturbation setting (see Subsection \ref{Examples of transferability settings}).

\textcolor{black}{As a toy example, consider the graph Laplacian on a graph with three nodes, defined via its eigendecomposition, where $\l_1=1$ has a 2D eigenspace, spanned by the eigenvectors $\phi_1,\phi_2$, and $\l_2=2$ has a 1D eigenspace spanned by $\phi_3$.
Implementation (\ref{filt0}) with $g_1\neq g_2$ is not even uniquely defined by $\DD$, as the basis $\{\phi_1,\phi_2\}$ of the eigenspace of $\l_1$ is not uniquely defined by $\DD$. On the other hand, the functional calculus implementation (\ref{eq:FC1}) imposes that the frequency response is one constant for the whole eigenspace of $\l_1$, and the non-uniqueness problem is avoided. More generally, in \cite{tran0} the stability of functional calculus filters was proved.}

\subsection{Spectral graph filters on directed graphs}
\label{Laplacians of directed graphs as normal operators0}

In Appendix \ref{Laplacians of directed graphs as normal operators} we explain how functional calculus applies as-is to non-normal matrices, even though the theory is defined only for normal operators. As a result, spectral filters can be defined on directed graphs represented by non-symmetric adjacency matrices.

There is an inner product structure in $\CC^{N}$ under which general diagonalizable matrices can be seen as normal operators.
Given an $N\times N$ diagonalizable matrix ${\bf A}$ with eigenvectors $\{\boldsymbol{\gamma}_k\}_{k=1}^N$, consider the matrix $\boldsymbol{\Gamma}$ comprising the eigenvectors as columns. Define the inner product 
\begin{equation}
\ip{{\bf u}}{{\bf v}}= {\bf v}^{\rm H} {\bf B} {\bf u},
\label{eq:self1}
\end{equation}
where ${\bf B}=\boldsymbol{\Gamma}^{-{\rm H}}\boldsymbol{\Gamma}^{-1}$ is symmetric,  ${\bf u}$ and ${\bf v}$ are given as column vectors, and for a matrix ${\bf C}=(c_{m,k})_{n,m}\in \CC^{N\times N}$, the Hermitian transpose ${\bf C}^{\rm H}$ is the matrix consisting of entries $c^{\rm H}_{m,k}=\overline{c_{k,m}}$. Under the inner product (\ref{eq:self1}),  ${\bf A}$ is normal. Consider an operator $A$ represented by the matrix ${\bf A}$. The adjoint $A^*$ of an operator $A$ is defined to be the unique operator such that
\[\forall {\bf u},{\bf v}\in\CC^{d}, \quad \ip{A{\bf u}}{{\bf v}} = \ip{{\bf u}}{A^*{\bf v}}.\]
 The matrix representation of the adjoint $A^*$ is given by 
\begin{equation}
{\bf A}^*={\bf B}^{-1}{\bf A}^{\rm H}{\bf B}.
\label{eq:mat_adj1}
\end{equation}
Thus, an operator is self-adjoint if ${\bf B}^{-1}{\bf A}^{\rm H}{\bf B}={\bf A}$, and unitary if  ${\bf B}^{-1}{\bf A}^{\rm H}{\bf B}={\bf A}^{-1}$.

The above results are proved in Appendix \ref{Laplacians of directed graphs as normal operators}.

\section{The transferability inequality}
\label{The transferability inequality}

In this section we derive the \textbf{transferability inequality}, a generic inequality that bounds the \textbf{transferability error of filters} by the \textbf{transferability error of Laplacians} plus the error entailed by sampling-interpolating, called the \textbf{consistency error}.

\subsection{The general setting of transferability}

For a graph discretizing a ``continuous'' topological space, as described in Subsections \ref {Theoretical settings of transferability} and \ref{The basic assumption of graphs discretizing metric spaces}, the transferability error between the graph and the topological space is defined as follows. Given a generic signal in the topological space, on the one hand, the signal is sampled to the graph, the discrete filter is applied on the sampled signal, and the filtered signal is interpolated back to the topological space. On the other hand, the filter is applied on the signal directly in the topological space. The error between these two output signals is called the transferability error of the filter.  For two graphs with small transferability error between each of the graphs and the topological space, the transferability error of the filter between the two graphs is also small by the triangle inequality. We thus focus on transferability between graphs and topological spaces.

Topological space signals are functions that assign to every point in the topological space a value. The error between pairs of signals is defined as the root mean square error (RMSE). To define RMSE in this abstract setting we must be able to integrate over the topological space, and thus we always assume that the topological space comes with some notions of volume, namely a Borel measure\footnote{A measure is a generalization of the notion of volume. A Borel measure in a topological space is a notion of volume that respects in some sense the topological structure. For example, open sets must have well defined volumes.}. 

Instead on focusing on graphs discretizing topological spaces via sampling, we consider a more general setting. In the general setting we study transferability between two domains, $\mathcal{M}$ and the finite domain $G$.
The domains $\mathcal{M}$ and $G$ are assumed to be measure spaces, and we consider the two spaces of signals\footnote{Using the notion of volume of a measure space $\mathcal{M}$ it is possible to define integration, and thus define the Lebesgue space of square integrable functions $L^2({\cal M})$.} $L^2(\mathcal{M})$ and $L^2(G)$. We assume that the spaces $L^2({\cal M})$ and $L^2(G)$ are separable, namely, there exist orthonormal bases of $L^2({\cal M})$ and $L^2(G)$. Since filtering is seen as a procedure of increasing certain frequencies, and decreasing others, we need a notion of oscillation of signals in the spaces $L^2(\mathcal{M})$ and $L^2(G)$. For that, we endow the signal spaces with additional structure.
In each of the signal spaces $L^2(\mathcal{M})$ and $L^2(G)$ we consider a special normal linear operators (typically self-adjoint) that we call the \textbf{Laplacian} of the space. For $L^2(\mathcal{M})$ we denote the Laplacian by $\mathcal{L}$, and for $L^2(G)$ we denote the Laplacian by $\DD$. We suppose that $\mathcal{L}$ and $\DD$ have discrete spectra in the following sense.
\begin{definition}
\label{discreteL}
Consider the normal operator $T$ with spectrum consisting only of eigenvalues, and denote the eigendecomposition of $T$ by $\{\l_j,P_j\}_{j=1}^{\infty}$, with eigenvalues $\l_j$ and projections $P_j$ upon the corresponding eigenspaces $W_j$. We say that $T$ has \textbf{discrete spectrum}
if in each bounded disc in $\CC$ there are finitely many eigenvalues of $T$, and the eigenspace of each eigenvalue is finite-dimensional. We consider the eigenvalues in increasing order of $\abs{\l_j}$, and denote $\Lambda(T)=\{\l_j\}_{j=1}^{\infty}$.
\end{definition}
For example, Laplace-Beltrami operators on compact Riemannian manifolds satisfy Definition \ref{discreteL} by Weyl’s law \cite[Chapter 11]{WeylLaw}.

As discussed in Subsection \ref{Laplacians of directed graphs as normal operators}, the Laplacian $\DD$ need not be a normal matrix. If $\DD$ is not a normal matrix,
we consider an inner product structure on each $L^2(V_n)$ for which $\DD$ is a normal operator.

The Laplacians $\mathcal{L}$ and $\DD$ define the notion of oscillation on $L^2(\mathcal{M})$ and $L^2(G)$. Namely, the eigenvectors of the Laplacians are seen as the pure harmonics, or Fourier modes. The eigenvalues are seen as an ordering of the pure harmonics, where the larger the eigenvalue corresponding to an eigenvector, the more oscillatory the eigenvector is. 
Filters are defined as measurable functions $f:\CC\rightarrow\CC$. Each filter can be manifested in both spaces via functional calculus, where the filter in $L^2(\mathcal{M})$ is defined as $f(\mathcal{L})$,  and the filter in $L^2(G)$ is defined as $f(\DD)$.

We suppose that the space $G$ is finite, and thus $L^2(G)$ is finite-dimensional.
When $\mathcal{M}$ is infinite, the signal space $L^2(\mathcal{M})$ is infinite-dimensional in general. We consider the finite-dimensional subspace of signal of $L^2(\mathcal{M})$ spanning all of the eigenvectors of $\mathcal{L}$ up to some eigenvalue, as defined next.
\begin{definition}
Let $\mathcal{L}$ be a normal operator in $L^2(\mathcal{M})$ with discrete spectrum.
Denote the eigenvalues, eigenspaces, and projections upon the eigenspaces of $\mathcal{L}$ by $\{\l_j,W_j,P_j\}_{j\in\NN}$.
For each $\l>0$, we define the $\l$'th \textbf{Paley-Wiener} space of $\mathcal{M}$ as 
\[PW(\l)=\oplus_{j\in\NN}\{W_j\ |\ \abs{\l_j}\leq\l\}.\]
We denote by $P(\l)$ the \textbf{spectral projection} upon $PW(\l)$, given by 
\[P(\l) = \sum_{\l_j\in \Lambda(\DD),\ \abs{\l_j}\leq\l}P_j.\]
\end{definition}

A Paley-Wiener space is interpreted as the space of band-limited signals in the band $\l$. 
When $L^2(\mathcal{M})$ is infinite-dimensional we restrict the analysis to a generic Paley-Wiener space $PW(\l)\subset L^2(\mathcal{M})$. Namely, transferability is analyzed on signals which are not too oscillatory.

To accommodate a transferability analysis, we consider two mappings  that transfer signals from $L^2(\mathcal{M})$ to signals in $L^2(G)$ and back. For each fixed band $\l$, consider the linear operators
\[S^{\l}:  PW(\l)\rightarrow L^2(G) , \quad R^{\l}:L^2(G)\rightarrow PW(\l).\]
We typically think of $S^{\l}$ as down-sampling or discretization, and $R^{\l}$ as up-sampling. We thus call $S^{\l}$ \textbf{sampling} and $R^{\l}$ \textbf{interpolation}.
\begin{definition}
The \textbf{transferability error of the filter} $f$ \textcolor{black}{(at the band $\l$)}, on the signal $s\in PW(\l)$, is defined by
\[\norm{f(\mathcal{L})s - R^{\l}f(\DD)S^{\l}s},\]
the \textbf{transferability error of the Laplacian} \textcolor{black}{(at the band $\l$)} is defined by
\[\norm{\mathcal{L}s - R^{\l}\DD S^{\l}s},\]
and the \textbf{consistency error} \textcolor{black}{(at the band $\l$)} is defined by
\[\norm{s - R^{\l}S^{\l}s}.\]
\end{definition}
What we prove in this section is the following inequality
\[\norm{f(\mathcal{L})s - R^{\l}f(\DD)S^{\l}s} \leq C_1\norm{\mathcal{L}s - R^{\l}\DD S^{\l}s} + C_2\norm{s - R^{\l}S^{\l}s}\]
up to some constants $C_1$ and $C_2$.

\subsection{Examples of transferability settings}
\label{Examples of transferability settings}

Before we formulate the transferability inequality theorem, let us give three concrete settings of the above transferability analysis. In the first example, which was introduced in Subsections \ref {Theoretical settings of transferability} and \ref{The basic assumption of graphs discretizing metric spaces}, $\mathcal{M}$ is a topological space with a Borel measure. 
The space $G$ is a graph, where the nodes of $G$ are seen as sample points in $\mathcal{M}$.
Sampling general signals in the Lebesgue space $L^2(\mathcal{M})$ is not well defined (unless $\mathcal{M}$ is discrete), since signals in $L^2(\mathcal{M})$ are defined up to a subset of $\mathcal{M}$ of measure zero. 
To be able to define sampling properly we consider the Paley-Wiener spaces, an approach that generalizes the standard Nyquist--Shannon theory in signal processing in $L^2(\RR)$. For that we further assume that the Paley-Wiener spaces associated with $\mathcal{L}$ consist of continuous functions (see Definition \ref{def:respectCont} for more details). 
 Sampling is the operator $S^{\l}$ that evaluates signals $s\in PW(\l)\subset L^2(\mathcal{M})$ at the sample points to obtain a signal on the graph. Similarly to classical digital signal processing, we define the interpolation $R^{\l}$ as the adjoint of the sampling operator, namely $R^{\l}=S^{\l *}$ (see Subsection for more information). Transferability between $\mathcal{M}$ and $G$ is thus seen as the error entailed by operating in the digital domain $G$ instead of the analog domain $\mathcal{M}$.

As a second example, we consider transferability under graph coarsening. Here, $\mathcal{M}$ is a graph, and $G$ is a coarse version of $M$.  In the Graclus algorithm for coarsening \cite{Graclus},  pairs of neighboring nodes in $\mathcal{M}$ with strong weights are collapsed to single nodes in $G$. Since both $\mathcal{M}$ and $G$ are finite, we consider the whole space $L^2(\mathcal{M})$ as the Paley-Wiener space, and omit the superscript $\l$ in $R$ and $S$. Given a signal $s$,
coarsening, $S$, is the operator that assigns the value
\begin{equation}
    \label{eq:coarsening_op}
    [Ss](q_{1,2}) =(s(q_1)+s(q_2))/\sqrt{2}
\end{equation} 
to the node $q_{1,2}$ of $G$ with parent nodes from $\mathcal{M}$,  $q_1$ and $q_2$, that have the signal values $s(q_1)$ and $s(q_2)$ respectively. Piecewise constant interpolation is defined to be $R= S^*$.

The last example is graph perturbation. Here, $\mathcal{M}$ is a graph, and $G$ is a perturbation of $\mathcal{M}$, that is obtain by adding or deleting random edges from $\mathcal{M}$ or perturbing the edge weights. Here we take $S=R=I$. Transferability in this case is called stability.

\subsection{Theorem of transferability inequality}

For the transferability inequality we need the following notations. For a continuous $g:\CC\rightarrow\CC$ and $M\in\NN$ denote 
\begin{equation}
    \label{eq:glm}
    \norm{g}_{\mathcal{L},M}:=\max_{0\leq m \leq M}\{\abs{g(\lambda_m)}\}.
\end{equation}
	For each $\l_m\in \Lambda(\mathcal{L})$ denote
	\begin{equation}
	{\rm V}_g(\lambda_m):=\max_{\kappa\in \Lambda(\DD)}\abs{\frac{g(\kappa)-g(\l_m)}{\kappa-\l_m}}.
	\label{eq:V_fl}
	\end{equation}
	Note that for a Lipschitz continuous $g$ with Lipschitz constant $D$, it follows from $\abs{\frac{g(x)-g(y)}{x-y}}\leq D$ that ${\rm V}_g(\lambda_m) \leq D$. Denote by $\#\{\l_j\leq\l\}_j$ the number of eigenvalues of $\mathcal{L}$ less or equal to $\l$, and note that
\[\#\{\l_j\leq\l\}_j\leq {\rm dim}PW(\l),\]
where ${\rm dim}PW(\l)$ is the dimension of $PW(\l)$.
\begin{example}
For the Laplacian on the $d$-dimensional torus, we have $\#\{\l_j\leq\l\}_j = O(\l^{1/2})$.
For compact Riemannian manifolds and the Laplace-Beltrami operator, by Weyl’s law, $\#\{\l_j\leq\l\}_j\leq{\rm dim}PW(\l) = O((2\pi)^{-d}\l^{d/2})$ where $d$ is the dimension of the manifold \cite[Chapter 11]{WeylLaw}.
\end{example}
We are now ready to formulate five versions of the transferability inequality.

\begin{theorem}
\label{main00}
Consider the above setting, and let $\lambda_M>0$ be a band with 
 $\norm{R^{\lambda_M}}<C$. 
   Let $g:\mathbb{R}\rightarrow\mathbb{C}$ be a Lipschitz continuous function with Lipschitz constant $D$.  Let $q=\sum_{m=0}^M c_m\phi_m\in PW(\lambda_M)\subset L^2(\mathcal{M})$ have normalized eigenspace components $\phi_m\in W_m$, $m=0,\ldots,M$.
	Then the following bounds are satisfied.
	\begin{enumerate}
	\item 
	\label{eq:TE_Fmode_G}
	Transferability of $\mathcal{L}$-Fourier modes evaluated in $G$:
	\[
\begin{split}
\norm{S^{\lambda_M} g(\mathcal{L}) \phi_m - g(\boldsymbol{\Delta})S^{\lambda_M}\phi_m} \leq 
{\rm V}_g(\lambda_m)\norm{\boldsymbol{\Delta}S^{\lambda_M}\phi_m - S^{\lambda_M} \l_m \phi_m}.
\end{split}
\]
\item 
\label{eq:TE_point_G}
Pointwise transferability evaluated in $G$:
\[
\norm{S^{\lambda_M}g(\mathcal{L})q - g(\boldsymbol{\Delta}) S^{\lambda_M} q}  \leq    \sum_{m=0}^M {\rm V}_g(\lambda_m) \abs{c_m}\norm{S^{\lambda_M}\mathcal{L} \phi_m -\boldsymbol{\Delta} S^{\lambda_M} \phi_m}.
\]
\item
\label{eq:TE_worst_G}
Worst-case transferability evaluated in $G$:
\[
\begin{split}
 & \norm{S^{\lambda_M}g(\mathcal{L})P(\l_M) - g(\DD) S^{\l_M}P(\l_M)}  \\
 & \leq D \sqrt{\#\{\l_j\leq\l_M\}_j} \norm{S^{\l_M}\mathcal{L} P(\l_M) -\DD S^{\l_M}P(\l_M)}.
\end{split}
\]
	\item 
	\label{eq:TE_point_M}
	Pointwise transferability evaluated in $\mathcal{M}$:
\[
\begin{split}
\norm{g(\mathcal{L})q - R^{\lambda_M}g(\boldsymbol{\Delta}) S^{\lambda_M} q}  \leq  & C \sum_{m=0}^M {\rm V}_f(\lambda_m) \abs{c_m}\norm{S^{\lambda_M}\mathcal{L} \phi_m -\boldsymbol{\Delta} S^{\lambda_M} \phi_m} \\ &  + \norm{g}_{\mathcal{L},M}\norm{q - R^{\lambda_M} S^{\lambda_M} q},
\end{split}
\]
\item 
\label{eq:TE_worst_M}
Worst-case transferability evaluated in  $\mathcal{M}$:
		\[
\begin{split}
 & \norm{g(\mathcal{L})P(\lambda_M) - R^{\lambda_M}g(\boldsymbol{\Delta}) S^{\lambda_M} P(\lambda_M)}  \\
 & \leq DC \sqrt{\#\{\l_j\leq\l_M\}_j} \norm{S^{\lambda_M}\mathcal{L} P(\lambda_M) -\boldsymbol{\Delta} S^{\lambda_M} P(\lambda_M)} \\
 & \ \ \ \ + \norm{g}_{\mathcal{L},M}\norm{P(\lambda_M) - R^{\lambda_M} S^{\lambda_M}P(\lambda_M)},
\end{split}
\]
	\end{enumerate} 
\end{theorem}

\textcolor{black}{Theorem \ref{main00} can be seen as a family of bounds, for different choices of Paley-Wiener spaces. Typically, if we choose a small cut-off frequency $\lambda_M$, the Laplacian has a lower transferability error (see, e.g., Lemma \ref{T:quad1}), so we can prove a low approximation error. However, the bounds are true only for the low frequency content of the signal, namely, for the ``smooth content.'' If we choose high  $\lambda_M$, we can also model “non-smooth” signals, but, on account of a typically higher transferability error of the Laplacian. This principle in choosing the Paley-Wiener space is true for all results presented in this paper which depend on a choice of the cut-off frequency. }

\textcolor{black}{We note that at the time of writing this paper, it is still not clear whether the dependency on $\sqrt{\#\{\l_j\leq\l\}_j}$ in the operator norm bounds 3 and 5 is tight, or just an artifact of the proof.}


Let us now study the transferability between two graphs. Consider two graphs $G_1$ and $G_2$, with corresponding graph Laplacians $\boldsymbol{\Delta}_1$ and $\boldsymbol{\Delta}_2$, that represent the same phenomenon. Adopting our basic assumption, we thus suppose that both graphs approximate the space $\mathcal{M}$ in the sense that the transferabiliy errors of the Laplacians and the consistency errors are small.
\begin{corollary}
Consider a fixed Paley-Wiener space $PW(\lambda_M)$, and for each $n=1,2$, suppose 
$\norm{\mathcal{L}P(\l_M)  - R_n^{\lambda_M} \boldsymbol{\Delta}_n S_n^{\lambda_M}P(\l_M)  } \leq \d$ and $\norm{P(\l_M)  - R_n^{\lambda_M} S_n^{\lambda_M}P(\l_M)  }\leq \d$ for some small $\d>0$.
Then
\begin{equation}
\norm{R_{1}^{\lambda_M}f(\boldsymbol{\Delta}_1) S_1^{\lambda_M}P(\l_M)  - R_{2}^{\lambda_M}f(\boldsymbol{\Delta}_2) S_2^{\lambda_M}P(\l_M) } = O(\d).
\label{eq:trans0}
\end{equation}
\end{corollary}
 \begin{proof}
 By the triangle inequality we have
\begin{equation}
\begin{split}
& \norm{R_{1}^{\lambda_M}f(\boldsymbol{\Delta}_1) S_1^{\lambda_M}P(\l_M)  - R_{2}^{\lambda_M}f(\boldsymbol{\Delta}_2) S_2^{\lambda_M}P(\l_M) } \\
 & \leq 
\norm{f(\mathcal{L})P(\l_M) - R_{1}^{\lambda_M}f(\boldsymbol{\Delta}_1) S_1^{\lambda_M}P(\l_M)  }
+
\norm{f(\mathcal{L})P(\l_M) -  R_{2}^{\lambda_M}f(\boldsymbol{\Delta}_2) S_2^{\lambda_M}P(\l_M) }
.
\end{split}
\label{eq:err11}
\end{equation}
Thus, (\ref{eq:trans0}) follows fromThm.\ref{main00}(\ref{eq:TE_worst_M}) .
 \end{proof}

A similar result can be obtained in the pointwise analysis.

\section{Transferability of spectral graph filters and ConvNets}
\label{Transferability of spectral graph filters and ConvNets}

In this section we study the transferability of spectral graph filters and ConvNets. We formulate general conditions guaranteeing transferability of filters, and then extend the analysis to full convolutional networks. We also give some numerical experiments that showcases transferability.

\subsection{Sufficient conditions for transferability}
\label{Transferability for discretizing Laplacians}

In this subsection we consider sufficient conditions for the right-hand-side of the transferability inequality to be small (Theorem \ref{main00}). The idea is to formulate general conditions, and to later on prove that they are satisfied in the specific setting of graphs discretizing topological spaces.
We denote by $\RR_+$ the set of non-negative real numbers.

Consider a measure space $\mathcal{M}$ for which $L^2(\mathcal{M})$ is a separable Hilbert space, and let the Laplacian $\mathcal{L}$ be a normal operator in $L^2(\mathcal{M})$ with discrete spectrum. Denote the eigenvalues of $\mathcal{L}$ by $\l_j$, the eigenprojections by $P_j$, and the Paley-Wiener spaces $PW(\l)$. 
 To accommodate the approximation analysis, we consider a sequence of graphs $G_n$ with $d_n$ vertices and graph Laplacians $\DD_n$, such that ``$\DD_n \xrightarrow[n \to \infty]{} \mathcal{L}$'' in a sense that will be clarified in Definition \ref{As_convS}.
 
 By abuse of notation, we denote the set of vertices of $G_n$ also by $G_n$. We consider an inner product structure on each $L^2(G_n)$ for which $\DD_n$ is a normal operator. Denote the eigendecomposition of $\DD_n$ by $\{\k^n_j,Q^n_j\}_j$, and denote $\Lambda(\DD_n):= \{\k^n_j\}_j$. For any $\k>0$, denote by $Q_n(\k)$ the spectral projection of $\DD_n$ defined by 
\[Q_n(\k) := \sum_{\k^n_j\in \Lambda(\DD_n),\ \abs{\k^n_j}\leq\k}Q^n_j.\]
%
Generic sampling and interpolation operators are only required to satisfy mild conditions, as defined next.
%
%
\begin{definition}
Under the above construction, the two mappings 
\[\{S_n^{\l}\}_{n,\l}:(n,\l) \mapsto S_n^{\l} , \quad  \{R_n^{\l}\}_{n,\l}:(n,\l) \mapsto R_n^{\l}\]
from $\NN\times \RR_+$ to the space of linear operators $L^2(\mathcal{M})\rightarrow L^2(G_n)$ and $L^2(G_n)\rightarrow L^2(\mathcal{M})$ respectively, satisfying for each $\l\geq 0$  \[S_n^{\l}:PW(\l)\rightarrow L^2(G_n), \quad R_n^{\l}:L^2(G_n)\rightarrow PW(\l),\]
are called \textbf{sampling sequence} and \textbf{interpolation sequence} respectively, if the following condition is held.:
for every $\l'>\l\geq 0$
\begin{equation}
    S_n^{\l'}P(\l)=S_n^{\l}, \quad P(\l)R_n^{\l'}=R_n^{\l}.
    \label{eq:SampIntSeq}
\end{equation}

Operators $S_n^{\l}$ from a sampling sequence are called \textbf{sampling operators}, and similarly, $R_n^{\l}$ are called \textbf{interpolation operators}.
\end{definition}
For $\l'>\l$, (\ref{eq:SampIntSeq}) means that sampling a signal from $PW(\l)$ using $S_n^{\l'}$ is exactly the same as sampling it using $S_n^{\l}$, and in this sense the different sampling operators of a sampling sequence are related to each other. Interpolation operators of an interpolation sequence have a similar interpretation.
Given sampling and interpolation operator sequences, by the fact that $PW(\l)$ is finite dimensional,  $S_n^{\l}$ and $R_n^{\l}$ must be bounded for each $n\in\NN$ and $\l\geq 0$.

In the following, we fix a sampling and interpolation sequence.
Next, we define general conditions on $\DD_n$, $S_n^{\l}$ and $R_n^{\l}$, and show that these conditions guarantee transferability of spectral graph filters.
 In Section \ref{Sampling and interpolation} we give an explicit construction of the sampling and interpolation operators in the DSP setting, where $S_n^{\l}f$ evaluates the signal $f\in PW(\l)$ at a set of sample points, viewed as the vertices of $G_n$. Under that construction, we show in Subsection \ref{Sampling and interpolation} that the conditions underlying Definitions \ref{As_RS}--\ref{As_convS} are satisfied. 
 
\begin{definition}
\label{As_RS}
The sequence $\{\{R_n^{\l},  S_n^{\l}\}_n\ |\ \l\in\RR \}$ is called \textbf{asymptotically reconstructive} if
for any fixed band $\l$,
\begin{equation}
\lim_{n\rightarrow\infty} R_n^{\l} S_n^{\l}P(\l)= P(\l).
\label{eq:RS}
\end{equation}
\end{definition}

Note that since $PW(\l)$ is a finite-dimensional space, the operator norm topology and the strong topology are equivalent, namely
\begin{equation}
\lim_{n\rightarrow\infty}\max_{f\in PW(\l)}\frac{\norm{f-R_n^{\l} S_n^{\l}f}}{\norm{f}}=0 \ \ \Longleftrightarrow \ \   \forall f\in PW(\l),\ \lim_{n\rightarrow\infty}\norm{f-R_n^{\l} S_n^{\l}f}=0,
\label{eq:2limits}
\end{equation}
and the limit in (\ref{eq:RS}) can be defined in either way.

\begin{definition}
\label{As_RS2}
The sequence $\{\{R_n^{\l},  S_n^{\l}\}_n\ |\ \l\in\RR \}$ is called \textbf{bounded} if
there exists a global constant $C\geq 1$ such that
for any fixed band $\l$, 
\begin{equation}
\limsup_{n\in\NN}\norm{S_n^{\l}} \leq C, \quad \limsup_{n\in\NN}\norm{R_n^{\l}}\leq C.
\label{eq:RS2}
\end{equation}
where the induced operator norms are with respect to the vector norms in $PW(\l)$ and in $L^2(V_n)$.
\end{definition}
Boundedness (Definition \ref{As_RS2}) is a necessary condition for sampling and interpolation to approximate isometries as  the resolution of sampling $d_n$ becomes finer, and we typically consider $C=1$.

\begin{definition}
\label{As_convS}
	The set of sequences $\{\{ \DD_n, S_n^{\l}\}_n\ |\ \l\in\RR \}$ are called \textbf{convergent} to $\mathcal{L}$ if for every fixed band $\l$,
\begin{equation}
\lim_{n\rightarrow\infty}\norm{S_n^{\l}\mathcal{L} P(\l)-\DD_nS_n^{\l}P(\l)}=0.
\label{eq:convS}
\end{equation}
where the norm in (\ref{eq:convS}) is with respect to $L^2(V_n)$.
\end{definition}

In the DSP setting of transferability, for $S_n^{\l}$ that evaluates the signal at sample points and corresponding $R^{\l_n}$ and $\DD_n$,
boundedness and asymptotic reconstruction (Definitions \ref{As_RS},\ref{As_RS2}) are proved in Proposition \ref{P:quad1}. Convergence (Definition \ref{As_convS}) is proved in Proposition \ref{PropDef3l} in the DSP setting. We can also treat sampling and interpolation abstractly, allowing other constructions for transforming signals in $L^2(\mathcal{M})$ to graph signals in $L^2(V_n)$.  In the abstract setting, sampling and interpolation are assumed to be bounded, asymptotically reconstructive, and graph Laplacians are assumed to be convergent to $\mathcal{L}$.
Assuming boundedness, asymptotic reconstruction, and convergence of Laplacians, is permissive in a sense, since we only demand asymptotic properties on the finite-dimensional Paley-Wiener spaces. However, under these assumptions, we are able to prove convergence of spectral filters on band-unlimited signals.

The following proposition proves asymptotic perfect tansferability, and is a direct result of the transferability inequality. 
\begin{proposition}
\label{Prop10}
consider the above setting, and a fixed band $\l>0$. Let
 $S_n^\l,R_n^{\l}$ and $\DD_n$, $n=1,\ldots,\infty$, be bounded,  asymptotically reconstructive, and convergent (Definitions \ref{As_RS}-\ref{As_convS}). 
Let $g:\mathbb{R}\rightarrow\mathbb{C}$ be a Lipschitz continuous function. 
Then
\[\begin{split}
\norm{g(\mathcal{L})P(\lambda) - R_n^{\lambda}g(\boldsymbol{\Delta}_n) S_n^{\lambda} P(\lambda)} = & O\Big(\norm{S_n^{\l}\mathcal{L} P(\l) -\DD_n S_n^{\l}P(\l)} + \norm{P(\l) - R_n^{\l}S_n^{\l}P(\l)}\Big) \\
 & \xrightarrow[n \to \infty]{} 0.
\end{split}\]
\label{Prop:Asymp_trans0}
\end{proposition}
\begin{proof}
Denote by $M_{\l}$ the largest index $M\in\NN$ such that $\l_M\leq \l$. Then by Thr\ref{main00}.(\ref{eq:TE_worst_M}), by (\ref{eq:SampIntSeq}) and by the fact that $P(\l_m)=P(\l)$
	\[
\begin{split}
\norm{g(\mathcal{L})P(\lambda) - R_n^{\lambda}g(\boldsymbol{\Delta}_n) S_n^{\lambda} P(\lambda)} 
 =  &  
 \norm{P(\lambda_M)g(\mathcal{L})P(\lambda_M) - P(\lambda_M)R_n^{\lambda}g(\boldsymbol{\Delta}_n) S_n^{\lambda} P(\lambda_M)} \\
 &  + \norm{g}_{\mathcal{L},M}\norm{P(\lambda_M) - R_n^{\lambda_M} S_n^{\lambda_M}P(\lambda_M)}\\
 \leq  & DC \sqrt{\#\{\l_j\leq\l\}_j} \norm{S_n^{\lambda_M}\mathcal{L} P(\lambda_M) -\boldsymbol{\Delta} S_n^{\lambda_M} P(\lambda_M)} \\
 &  + \norm{g}_{\mathcal{L},M}\norm{P(\lambda_M) - R_n^{\lambda_M} S_n^{\lambda_M}P(\lambda_M)}\\
\end{split}
\]
\end{proof}
%
%
%

Next, we show how to treat band-unlimited signals. Under the conditions of Theorem \ref{Prop:Asymp_trans0}, for each band $\l\in\NN$, there exists $N_{\l}\in\NN$ such that for any $n>N_{\l}$ we have
\[\norm{g(\mathcal{L})P(\l) - R_n^{\l}g(\DD_n) S_n^{\l}P(\l)}<\frac{1}{\l}\]
We may choose the sequence $\{N_{\l}\}_{\l\in\NN}$ increasing.
We construct a sequence of bands $\{\psi_n\}$, starting from some index $n_0>0$, as follows.
For each $\l\in\NN$, consider $N_{\l}$ and $N_{\l+1}$. For each $N_{\l}<n\leq N_{\l+1}$ we define $\psi_n = \l$. This gives the following corollary.

\begin{corollary}
\label{main1}
Under the conditions of Proposition \ref{Prop:Asymp_trans0}, there exists a sequence of bands $0<\psi_n \xrightarrow[n \to \infty]{} \infty$ such that
\begin{equation}
\lim_{n\rightarrow\infty} \norm{g(\mathcal{L}) - R_n^{\psi_n}g(\DD_n) S_n^{\psi_n}P(\psi_n)}=0.
\label{eq:main11c}
\end{equation}
\end{corollary}
A direct result of Corollary \ref{main1} is that
\begin{equation}
\lim_{n>m\rightarrow\infty} \norm{R_n^{\psi_m}g(\DD_n) S_n^{\psi_m}P(\psi_m) - R_m^{\psi_m}g(\DD_m) S_m^{\psi_m}P(\psi_m)}=0.
\label{eq:main11c2}
\end{equation}
Loosely speaking, the better both $\DD_m$ and $\DD_j$ approximate $\cL$, the larger the band where $g(\DD_m)$ and $g(\DD_j)$ have approximately the same repercussion.

Last, for the transferability analysis of convolution networks, we also need to assume that sampling approximately commutes with the activation function $\rho$, in the following sense. 
\begin{definition}
\label{def:pointwise convergent}
Consider a measure space $\mathcal{M}$ with a Laplacian $\mathcal{L}$ having a discrete spectrum. Let $P(\l)$ be the Paley-Wiener projections corresponding to $\mathcal{L}$. Consider a sequence of graphs $G_n$, sampling operators $S_n^{\l}$ from a sampling sequence, and an activation function $\rho$. 
	Sampling \textbf{asymptotically commutes with $\rho$} if
\begin{equation}
 \lim_{\l\rightarrow\infty} \lim_{\l'\rightarrow\infty}\lim_{n\rightarrow\infty}\sup_{f\neq 0}\frac{\norm{\rho(S_n^{\l}P(\l)f)-S_n^{\l'}P(\l')\rho(P(\l)f)}}{\norm{f}} =0.
\label{eq:tyutu7}
\end{equation}
\end{definition}
In Proposition \ref{Prop_comm4} we prove that in the DSP setting, under natural conditions, sampling asymptotically commutes with $\rho$ for a class of activation functions that include ReLU and the absolute value.

\subsection{Transferability of graph ConvNets}
\label{Transferability of graph ConvNets}

In this subsection we extend the transferability results of the previous subsection from filters to complete ConvNets.
Consider two graphs $G^j$, $j=1,2$ and two graph Laplacians $\DD_1,\DD_2$ approximating the same Laplacian $\mathcal{L}$ in a measure space. 
Consider a ConvNet with $L$ layers, with or without pooling. In each layer where pooling is performed, the signal is mapped to a signal over a coarsened graph. If pooling is not performed, we define the coarsened graph $G^{j,l}$ at Layer $l$ as the graph of the previous layer. Suppose that each coarsened version of each of the two graphs $G^{j,l}$, where $l$ is the layer, approximates the continuous space in the sense
\begin{equation}
    \norm{P(\psi_l) - R_{j,l}^{\psi_l}S_{j,l}^{\psi_l} P(\psi_l)} < \d
     \label{eq:temp66776g02}
\end{equation}
\begin{equation}
 \norm{S_{j,l}^{\psi_l}\mathcal{L} P(\psi_l) - \DD_{j,l}S_{j,l}^{\psi_l} P(\psi_l)} < \d  
 \label{eq:temp66776g0}
\end{equation}
for some $\d< 1$. Here, $\DD_{j,l}$ is the Laplacian of graph $j$ at Layer $l$, $S_{j,l}^{\psi_l},R_{j,l}^{\psi_l}$ are the sampling and interpolation operators of Layer $l$, and we consider the band $\psi^l$ at each Layer $l$. Equations (\ref{eq:temp66776g02}) and (\ref{eq:temp66776g0}) are non-asymptotic versions of Definition \ref{As_RS} and \ref{As_convS}.

In each Layer $l$ consider $K_l$ channels.
Let 
\[\{g^l_{k'k}\ |\ 
  {\scriptstyle k=1\ldots K_{l-1},\ 
  k'=1\ldots K_{l}}
\}\]
 denote the filters of Layer $l$, and consider the matrix $A^l=\{a^l_{k'k}\}_{k'k}\in\RR^{K_{l}\times K_{l-1}}$. 
We denote the bias at channel $k'$ and Layer $l$ by $b_{k'}^l$. Here, $b_{k'}^l$ is a scalar signal, namely the signal that assigns the constant  value $b_{k'}^l$ to each node. Note that, by abuse of notation, $b_{k'}^l$ denotes both a scalar and a signal. In most common graph ConvNet methods there are no biases, so we typically assume $b_{k'}^l=0$.


Denote the signals/feature-map at Layer $l$ of the graph ConvNet of graph $G^j$, by $\{\tf^{j,l}_{k'}\}_{k'=1}^{K_l}$. The CovnNet maps Layer $l-1$ to Layer $l$ by 
\begin{equation}
\{\tf^{j,l}_{k'}\}_{k'=1}^{K_l}= Y^{j,l}\Big(\rho\Big\{b_{k'}^l+\sum_{k=1}^{K_{l-1}}a^l_{k'k}\ g^l_{k'k}(\DD_{j,l-1})\tf^{j,l-1}_k\Big\}_{k'=1}^{K_l}\Big),
\label{eq:Conv1}
\end{equation}
where $\rho$ is an activation function, and $Y^{j,l}:L^2(G^{j,l-1})\rightarrow L^2(G^{j,l})$ is pooling. For the graph ConvNets, the inputs of Layer 1 are $S_{j,1}^{\psi_0}P(\psi_0)f$ for $j=1,2$, where $f\in L^2(\mathcal{M})$ is a measure space signal. 
In the continuous case, we define the measure space ConvNet by
\begin{equation}
\{f^l_{k'}\}_{k'=1}^{K_l}= P(\psi^l)\Big(\rho\big\{ P(\psi^{l-1})b_{k'}^l+\sum_{k=1}^{K_{l-1}}a^l_{k'k}\ g^l_{k'k}(\mathcal{L})f^{l-1}_k\big\}_{k'=1}^{K_l}\Big),
\label{eq:Conv2}
\end{equation}
where $\{f^{l}_{k'}\}_{k'=1}^{K_l}$ is the signal at Layer $l$. Here, the input $P(\psi_0)f$ of Layer 1 is in $PW(\psi_0)$.
To understand the role of the projection $P(\psi_l)$ in (\ref{eq:Conv2}), note that spaces $PW(\psi_l)$ are not invariant under the activation function $\rho$ in general. Thus, as part of the definition of the ConvNet on $L^2(\mathcal{M})$, after each application of $\rho$ we project the result to $PW(\psi_l)$. Moreover, for typical choices of $\mathcal{L}$, like Laplace-Beltrami operator, the constant signal $b_{k'}^l$ is in $P(0)$, and thus $P(\psi_{l-1})b^l_{k'}=b^l_{k'}$. More generally, we project $b^l_{k'}$ to the Paley-Wiener space by $P(\psi_{l-1})b^l_{k'}$. 
 
The graph and measure space ConvNets are defined by iterating formulas (\ref{eq:Conv1}) and (\ref{eq:Conv2}) respectively along the layers.
We denote the mapping from the input of Layer 1 to Channel $k$ of Layer $l$ of the ConvNet by $\mathcal{N}^l_k$ for the measure space ConvNet, and by $\mathcal{N}^{j,l}_k$ for the graphs ConvNets $j=1,2$. Namely
\begin{equation}
f^l_k=\mathcal{N}^l_k P(\psi_0)f,  \quad
\tf^{j,l}_k=\mathcal{N}^{j,l}_k S_{j,1}^{\l} P(\psi_0)f.
\label{eq:Conv12}
\end{equation}
	
	We restrict ourselves to contractive activation functions, as defined next.
	\begin{definition}
The activation function $\rho$ is called \textbf{contractive} if for every $y,z\in\CC$ \newline
$\abs{\rho(y)-\rho(z)}\leq \abs{y-z}$.
\label{def:contractive}
	\end{definition}
	 The contraction property also carries to $L^p(\mathcal{M})$ spaces. Namely, if $\rho$ is contractive, then for every two signals $p,g$, $\norm{\rho(p)-\rho(g)}_p\leq \norm{p-g}_p$. For example, the ReLU and the absolute value activation functions are contractive. 
	 
	 We consider a generic pooling operator $Y^{j,l}:L^2(G^j,l-1)\rightarrow L^2(G^j,l)$. Typically, $Y^{j,l}$ is the max pooling or $l^2$ average pooling.
	 \begin{definition}
	 Suppose that coarsening is done by collapsing sequences of nodes of $G$ to single nodes in $G'$.  
\textbf{Max-pooling} is the non-linear operator that assigns the value
\begin{equation}
    \label{eq:maxp_pool_op}
    [Ys](y) =\max\{s(q_1),\ldots,s(q_K)\}/\sqrt{K}
\end{equation} 
to the node $y$ of $G'$ with parent nodes $q_1,\ldots,q_K$ from $G$, that have the signal values $s(q_1),\ldots,s(q_K)$ respectively, where $s\in L^2(G)$ is the $\RR_+$ valued signal.
\textbf{Average pooling} is defined similarly by
\begin{equation}
    \label{eq:l2p_pool_op}
    [Ys](y) =\sqrt{\sum_{k=1}^K s^2(q_k)/K}
\end{equation} 
	 \end{definition}
Note that in standard ConvNets of 2D images, pooling is defined via (\ref{eq:maxp_pool_op}) without division by $\sqrt{K}$. We divide max pooling by $\sqrt{K}$ since in the transferability setting it makes sense to normalize the $L^2$ norm in the coarse graph $G^{j,l}$ (see for example the DSP setting of Subsection \ref{Sampling and interpolation0}), while in standard ConvNets of 2D images the grid is not normalized.
Max pooling is norm-reducing, as defined next.
	 \begin{definition}
	 \label{def:reduce_norm}
	 Pooling, $Y:L^2(G)\rightarrow L^2(G')$, is said to \textbf{reduce norm} if $\norm{Y(h)}\leq\norm{h}$ for every $h\in L^2(G)$. 
	 \end{definition}
	 
	 In Theorem \ref{main_conv00} we assume that \textbf{pooling is consistent with sampling} in the sense that for every layer $l=1,\ldots,L$ and $j=1,2$,
	 \begin{equation}
	     \forall f\in PW(\psi_{l}),\  
\norm{Y^{j,l} S^{\psi_{l}}_{j,l-1}f - S^{\psi_{l}}_{j,l} f} \leq \d \norm{f}
\label{eq:TEMP56h*t}
	 \end{equation}
	 Equation (\ref{eq:TEMP56h*t}) means that sampling to the graph $G^{j,l-1}$ and then pooling to the graph $G^{j,l}$ is approximately the same as sampling to graph $G^{j,l}$ directly.
	 
	 In the following, we consider normalizations of the components of the ConvNet. In particular, assuming that sampling and interpolation are approximately isometries, we may normalize them with asymptotically small error to $\norm{S_{j,l}^{\psi_l}}=1,\norm{R_{j,l}^{\psi_l}}= 1$. We also assume that pooling reduces norm,

Suppose that sampling asymptotically commutes with $\rho$ (Definition \ref{def:pointwise convergent}), and let $0<\d<1$ be some tolerance.  By (\ref{eq:tyutu7}), it is possible to choose a sequence of bands $\psi_l$, and fine enough discretizations,
guaranteeing
\[\begin{split}
& \forall f\in L^2(\mathcal{M}),\ j=1,2,\ l=1,\ldots,L, \\
 &   \norm{\rho(S_{j,l-1}^{\psi_{l-1}}P(\psi_{l-1})f)-S_{j,l-1}^{\psi_l}P(\psi_l)\rho(P(\psi_{l-1})f)} < \d \norm{f}.
\end{split}\]
 Note that the band $\psi^l$ increases in $l$, since the activation function $\rho$ gradually increases the complexity of the signal. 
This leads us to the non-asymptotic setting of Theorem \ref{main_conv00}. Note as well that the ConvNet is invariant to multiplying all filters $g^l_{k',k}$ by a constant $\a\in\RR$ and multiplying $A^l$ by $1/\a$. Thus, in Theorem \ref{main_conv00} we suppose that all filters are normalized to $\norm{g^l_{k',k}}_{\infty}=1$, and the norm of the convolution layers is controlled by $A^l$.


\begin{theorem}
\label{main_conv00}
Consider a ConvNet with Lipschitz filters $\{g^l_{k'k}\ |\ {\scriptstyle k=1\ldots K_{l-1},\ k'=1\ldots K_{l}}\}$ with Lipschitz constant $D$ at each layer $l$, normalized to $\norm{g^l_{k',k}}_{\infty}=1$, and with $A^l$ satisfying $\norm{A^l}_{\infty}\leq A$, for some constant $A>0$. Suppose that the biases satisfy $\norm{b_{k'}^l}_2\leq B$ for some constant $B>0$, for $\norm{b_{k'}^l}_2$ denoting the norm of the constant signal $b_{k'}^l$ both in $L^2(G^{j,l})$ and in $L^2(\mathcal{M})$. Consider a contractive activation function $\rho$ (Definition \ref{def:contractive}). Suppose that $S_{j,l}^{\psi_l}$ and $R_{j,l}^{\psi_l}$ are normalized to $\norm{S_{j,l}^{\psi_l}}=1,\norm{R_{j,l}^{\psi_l}}= 1$. Let $0<\d<1$ and suppose that for every $j=1,2$

\begin{equation}
  \begin{split}
\forall l=0,\ldots,L-1 , &\quad  \norm{S_{j,l}^{\psi_{l}}\mathcal{L} P(\psi_{l}) - \DD_{j,l}S_{j,l}^{\psi_{l}} P(\psi_{l})} \leq \d \\
 & \quad   \norm{P(\psi_L) - R_{j,L}^{\psi_L}S_{j,L}^{\psi_L} P(\psi_L)} \leq \d \\ 
\forall f\in PW(\psi_{l-1}),\ \  \forall l=1,\ldots,L, & \quad   \norm{\rho(S_{j,l-1}^{\psi_{l-1}}P(\psi_{l-1})f)-S_{j,l-1}^{\psi_l}P(\psi_l)\rho(P(\psi_{l-1})f)} < \d \norm{f}\\
     \forall f\in PW(\psi_{l}), \ \  \forall l=1,\ldots,L,  & \quad  
\norm{Y^{j,l} S^{\psi_{l}}_{j,l-1}f - S^{\psi_{l}}_{j,l} f} \leq \d \norm{f}
\end{split}  
\label{eq:Th14}
\end{equation}
 Suppose that pooling reduces norm (Definition \ref{def:reduce_norm}).

Then, if $A> 1$,
\begin{equation}
\begin{split} & \norm{R_{1,L}^{\psi_L}\mathcal{N}^{1,L}_k S_{1,0}^{\psi_0}P(\psi_0)f -R_{2,L}^{\psi_L}\mathcal{N}^{2,L}_k S_{2,L}^{\psi_0}P(\psi_0)f}  \\
 & \ \ \ \leq  \Big(LD\sqrt{\#\{\l_m\leq \psi_L\}_m} +2L+2 \Big)\Big(A^L\norm{f} + B\frac{A^L-1}{A-1}\Big)\d 
 \end{split}
\label{eq:main_conv0}
\end{equation}
and, if $A=1$,
\begin{equation}
	\begin{split}
 & \norm{R_{1,L}^{\psi_L}\mathcal{N}^{1,L}_kS_{1,0}^{\psi_0} P(\psi_0)f- R_{2,L}^{\psi_L}\mathcal{N}^{2,L}_kS_{2,0}^{\psi_0} P(\psi_0)f}\\
 & \ \ \ \leq  \Big(LD\sqrt{\#\{\l_m\leq \psi_L\}_m} +2L+2\Big)(\norm{f}+LB)\d.
\end{split}
\label{eq:main_conv01}
\end{equation}
\end{theorem}

The proof of this theorem is in the appendix.
Theorem \ref{main_conv00} \textcolor{black}{may hint to the importance} of regularizing the convolution operators in the infinity norm. The next corollary  shows that adding bias increases instability with respect to the depth $L$, from linear to quadratic. 
\begin{corollary}
\label{main_conv0}
Consider the setting of Theorem \ref{main_conv00},  with $A^l$ normalized to $\norm{A^l}_{\infty}=1$, without biases, namely, $b_{k'}^l=0$. 
Then 
	\begin{equation}
\norm{R_{1,L}^{\psi_L}\mathcal{N}^{1,L}_kS_{1,0}^{\psi_0} P(\psi_0)- R_{2,L}^{\psi_L}\mathcal{N}^{2,L}_kS_{2,0}^{\psi_0} P(\psi_0)}
\leq  \Big(LD\sqrt{\#\{\l_m\leq \psi_L\}_m} +2L +2\Big)\d.
\label{eq:main_conv02}
\end{equation}
\end{corollary}

	The assumptions of Corollary \ref{main_conv0} imply that the ConvNet is contractive. For non-contractive ConvNets, we can simply consider a contractive ConvNet and multiply it by a constant $C>1$. For such a ConvNet, the bound in (\ref{eq:main_conv0}) is simply multiplied by $C$.

\subsection{Transferability experiments}
In this subsection we showcase transferability of spectral graph methods in practice.
We first mention two papers that showcase the transferability of spectral graph ConvNets. In \cite{ARMA_ConvNet} graph spectral ConvNets are based on rational function filters. The task is graph classification on datasets consisting of many graphs and graph signals. Each graph represents a molecule and its signal represents some node features. \textcolor{black}{The results reported in that paper}  show that the proposed spectral method obtains state of the art results on these multi-graph settings.

In \cite{TransSuperPixel} different types of graph ConvNets are tested on a machine learning tasks in imaging. Inputs are represented by superpixel images, namely some graphs representing the images. Here, the setting is multi-graph, where different images are represented by different graphs. The reported results suggest that spectral methods are more transferable, dealing better with the multi-graph setting than spatial methods.


Next we present simple experiments that demonstrate transferability.
In figure \ref{Fig:all3} we showcase transferability under coarsening on the Bunny mesh. Here, the graph $\mathcal{M}$ consist of all mesh edges with weight $1$, and we consider the normalized Laplacian $\mathcal{L}$. The coarsened version $G$ of $\mathcal{M}$ is computed by the Graclus algorithm \cite{Graclus}. We consider the coarsening operator $S:L^2(\mathcal{M})\rightarrow L^2(G)$ defined as follows. Given a signal $f\in L^2(\mathcal{M})$, for each pair of nodes $x,y\in\mathcal{M}$ which collapse to the node $z\in G$, we define 
\[[Sf](z)=\big(f(x)+f(z)\big)/\sqrt{2}.\]
We define the piecewise constant interpolation operator $R:L^2(G)\rightarrow L^2(\mathcal{M})$ by $R=S^*$. 
The \textbf{coarsened Laplacian} $\DD$ in $L^2(G)$ is defined by $\DD = S\mathcal{L}R$. This definition of $\DD$ is natural when the goal is to promote transferability.
We consider a signal $f\in L^2(\mathcal{M})$ and a filter $g$. The figure compares $f$, $\mathcal{L}f$, and $g(\mathcal{L})f$ with $RSf$, $R\DD Sf$, and $Rg(\DD)Sf$ respectively.

In Figure \ref{Fig:all2} we showcase the transferability formula Thm.\ref{main00}(\ref{eq:TE_Fmode_G}) on the Bunny graph of Figure \ref{Fig:all3}. We consider a filter $g$ with Lipschitz bound $D$. On the left we plot the Laplacian transferability $\norm{\boldsymbol{\Delta}S^{\lambda_M}\phi_m - S^{\lambda_M} \mathcal{L} \phi_m}$ as a function of the eigenvalue of the eigenvector $\phi_m$ of $\mathcal{L}$. In the middle we plot the filter transferability $\norm{S^{\lambda_M} g(\mathcal{L}) \phi_m - g(\boldsymbol{\Delta})S^{\lambda_M}\phi_m}$ as a function of the Laplacian transferability, with the theoretical bound $y=Dx$ in red. On the right we plot the filter transferability divided by the Laplacian transferability of eigenvectors $\phi_m$ as a function of the eigenvalues $\l_m$. The theoretical bound $V_g(\l_m)$ is given in red.

In top-left of Figure \ref{Fig:all} we isolate principle transferability form concept-based transferability in MNIST, and compare a spectral graph ConvNet method with a spatial graph ConvNet method. We consider a simple ConvNet architecture based on three convolution layers with max pooling, where the max pooling in the third layer collapses each graph to one node, and two fully connected layers. In CayleyNet, the Cayley polynomial order of all three convolutional layers is 9, and they produce 32, 32, and 64 output features, respectively. In MoNet, all three convolutional layers contain 18 Gaussian kernels, and produce 32, 32 and 64 output features, respectively. Both two models contain 10K parameters. We train the network on MNIST images of one fixed fine resolution ($56X56$) and test on images of various coarse resolutions. The graph Laplacian is given by the central difference approximating second derivative. In this setting, the spectral method, CayleyNet, has higher principle transferability than the spatial method, MoNet. Indeed, its performance degrades slower as we coarsen the grid.

In top-middle and right of Figure \ref{Fig:all} we test transferability between the Citeseer graph $\mathcal{M}$ and its coarsened version $G$. We take the coarsening and interpolation operators $S$ and $R=S^*$ as before. We consider the normalized Laplacian $\mathcal{L}$ on $\mathcal{M}$, and the coarse Laplacian $\boldsymbol{\Delta} = S\mathcal{L} R$ on $G$. We use low-pass (top-middle) and high-pass (top-right) filters with Lipschitz constant 1. To show transferability, we plot $\norm{Sf(\mathcal{L}) \phi_m- f(\boldsymbol{\Delta})S\phi_m}$ as a function of $\norm{S\mathcal{L} \phi_m- \boldsymbol{\Delta}S\phi_m}$ for various eigenvectors $\phi_m$ of $\mathcal{L}$ (some corresponding eigenvalues are displayed). All values lie below $y=Dx$, where $D$ is the Lipschitz constant of the corresponding filter. This accords with the transferability inequality Thm.\ref{main00}(\ref{eq:TE_Fmode_G}).

In Figure \ref{Fig:all} bottom, we test the stability of spectral graph filters in the Cora graph with the normalized Laplacian, for different models of graph perturbations and sub-sampling. We consider three filters: low, mid and high pass. In bottom-left we randomly remove edges, in bottom-middle we randomly add edges, and in bottom-right randomly delete vertices, and compare the filters on the sub-sampled graph. The markers indicate the percentage of edges/vertices that were removed/added. The $x$ axis is the relative error in the Laplacian, and the $y$ axis is the relative error in the filter. The experimental results support the theoretical results on linear stability. All errors are given in Frobenius norm. The Frobenius norm can be seen as the average pointwise error, where the Laplacians and filters are applied on the signals of the standard basis.

\begin{figure}[!htb]
	\begin{minipage}{1\textwidth}
		\centering		
		\includegraphics[width=0.99\linewidth]{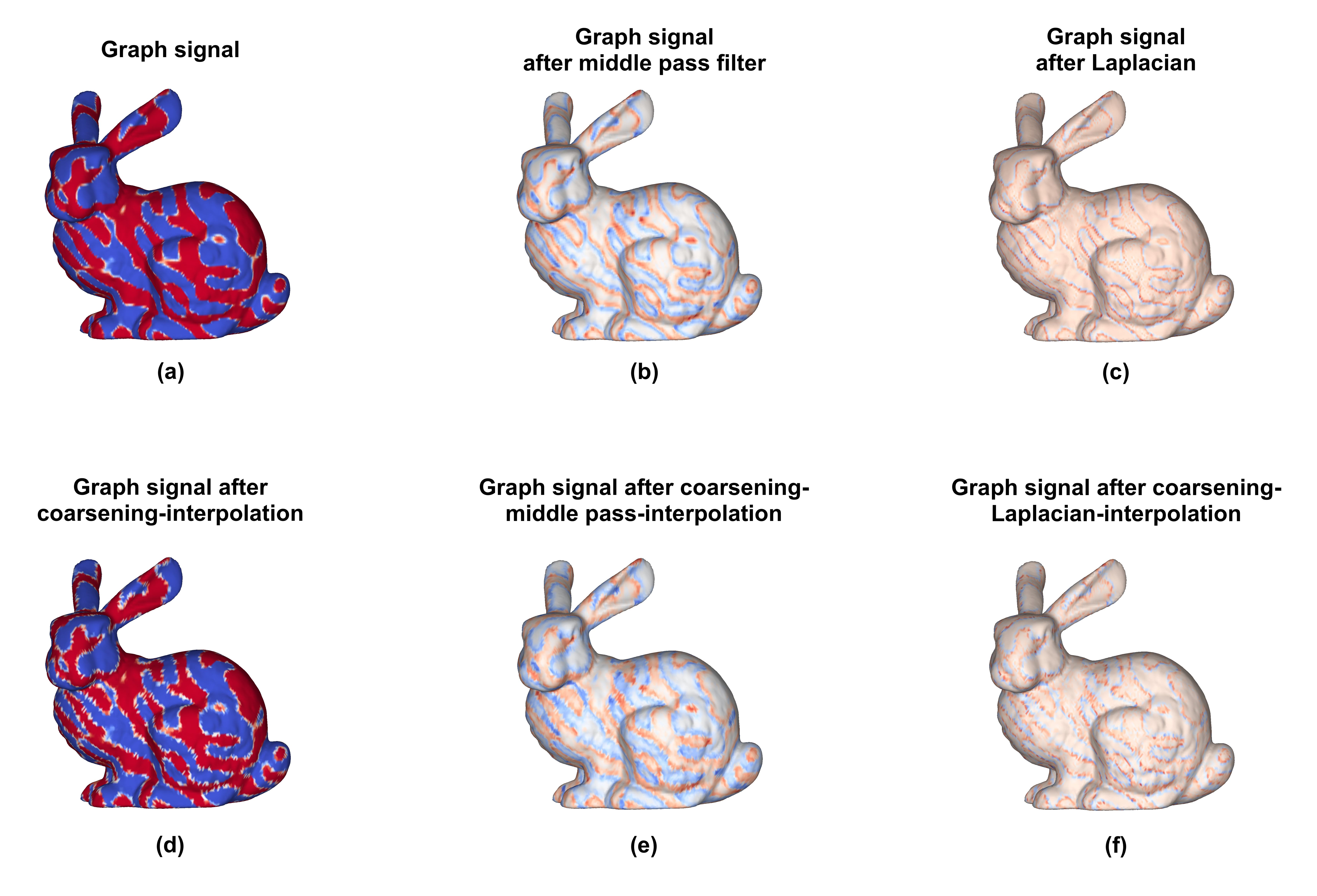}
		\caption{Transferability under coarsening on the Bunny mesh}
		\label{Fig:all3}
	\end{minipage}\hfill
\end{figure}

\begin{figure}[!htb]
	\begin{minipage}{1\textwidth}
		\centering		
		\includegraphics[width=0.99\linewidth]{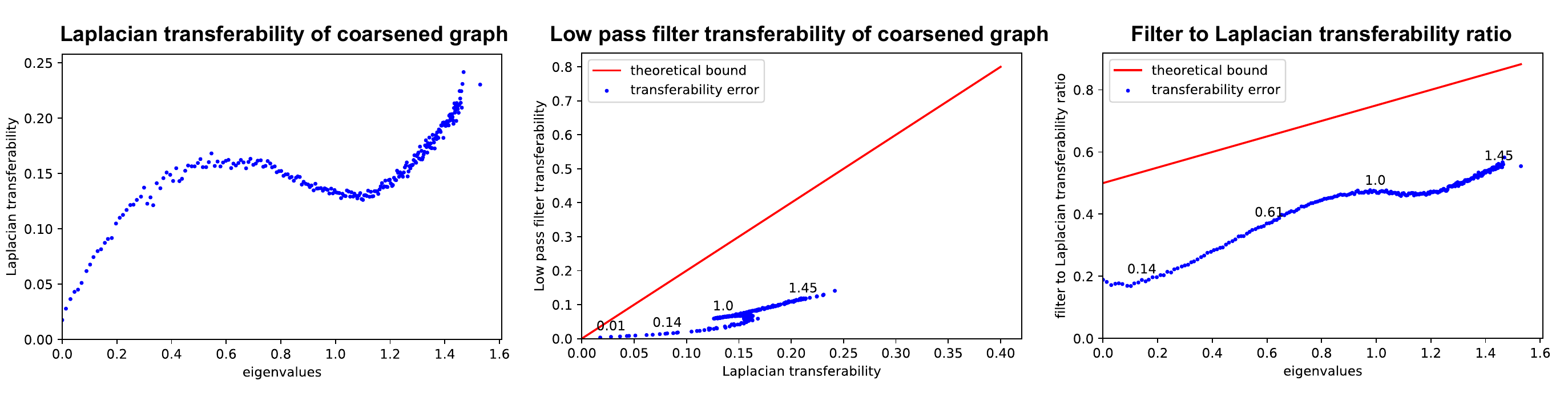}
		\caption{Transferability of $\mathcal{L}$-Fourier modes evaluated in $G$}
		\label{Fig:all2}
	\end{minipage}\hfill
\end{figure}

\vspace{-2.5mm}
\begin{figure}[!htb]
	\begin{minipage}{1\textwidth}
		\centering		
		\includegraphics[width=0.99\linewidth]{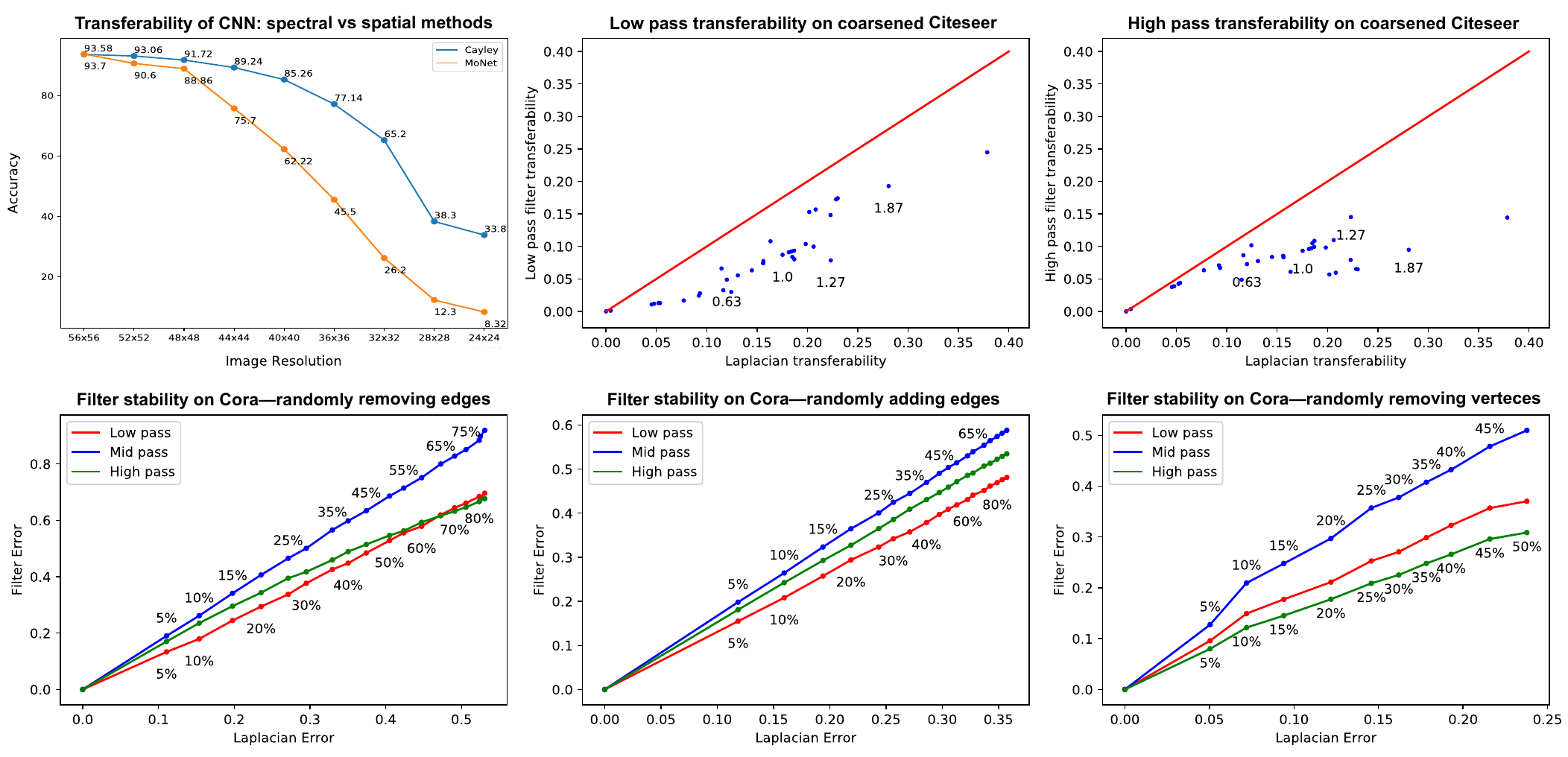}
		\caption{Trasferability experiments}
		\label{Fig:all}
	\end{minipage}\hfill
\end{figure}

\section{Transferability of graph discretizing topological spaces}
\label{Sampling and interpolation}

In this section we develop the DSP setting of transferability, in which graphs are sampled from continuous spaces, as described in Subsections \ref{Theoretical settings of transferability} and \ref{The basic assumption of graphs discretizing metric spaces}.
In the classical Nyquist--Shannon approach to digital signal processing, band-limited signals in $L^2(\RR)$ are discretized to $L^2(\ZZ)$ by sampling them on a grid of appropriate spacing. The original continuous signal can be reconstructed from the discrete signal via interpolation, which is explicitly given as the convolution of the delta train corresponding to the discrete signal with a sinc function.
Our goal is to formulate an analogous framework for graphs, where graphs are seen as discretizations of continuous entities, namely topological spaces with Borel measure.

Previous work studied sampling and interpolation in the context of graph signal processing, where the space that is sampled is a discrete graph itself. In \cite{G_Samp1,G_Samp2,G_Samp3,G_Samp4} sampling is defined by evaluating the graph signal on a subset of vertices, and in \cite{G_ag1,G_ag2} sampling is defined by evaluating the signal on a single vertex, and using repeated applications of the shift operator to aggregate the signal information on this node. In the context of discretizing continuous spaces to graphs, considering graph Laplacians of meshes as discretizations of Laplace-Beltrami operators on Riemannian manifolds is standard. However, manifolds are too restrictive to model the continuous counterparts of general graphs. A more flexible model are more general topological spaces with Borel measure.  Treating graph Laplacians as discretizations of metric space Laplacians was considered from a pure mathematics point of view in \cite{M_Laplace1}. In that work, the convergence of the spectrum of the graph Laplacian to that of the metric space Laplacian was shown under some conditions. However, for our needs, the explicit notion of convergence of Definition \ref{As_convS} is required, and the convergence of the spectrum alone is not sufficient. 
In \cite{graph_limit1}, a continuous limit object of graphs was proposed. More accurately, graph vertices are sampled from the continuous space $[0,1]$, and graph weights are sampled from a measurable weight function $W:[0,1]^2\rightarrow[0,1]$. 
In contract to this, in our analysis there is a special emphasis on Laplacians, which implicitly model the ``geometry'' of graphs and topological spaces. We thus bypass the analysis via edge weights, and study directly the discretization of topological Laplacians to graph Laplacians, from an operator theory point of view.

In this section we introduce a discrete signal processing setting, where analog domains are topological spaces, and digital domains are graphs. We present natural conditions, from a signal processing point of view, sufficient for the convergence of the graph Laplacian to the topological space Laplacian in the sense of Definition \ref{As_convS}. We also prove asymptotic reconstruction, boundedness, convergence, and asymptotic commutativity of sampling with activation function (Definitions \ref{As_RS}, \ref{As_RS2}, \ref{As_convS} and \ref{def:pointwise convergent}), under these conditions. All proofs are based on quadrature assumptions, stating that certain sums approximate certain integrals. In Subsection \ref{Random sample sets satisfy the quadrature definitions} we prove that the quadrature assumptions are satisfied in high probability, in case graphs are sampled randomly from topological spaces.

\subsection{Sampling and interpolation}
\label{Sampling and interpolation0}
We now proceed to give an explicit construction of the sampling and interpolation operators, under which they are asymptotically reconstructive and bounded (Definitions \ref{As_RS} and \ref{As_RS2}). The approach is similar to the classical Nyquist--Shannon approach to sampling and interpolation.

We start with basic notations and definitions.
Let $\mathcal{M}$ be a topological space with a Borel measure $\mu$, such that the volume $\mu(\mathcal{M})$ is finite. We call such a space a \textbf{topological-measure space}.
Let the Laplacian $\cL$  be a normal operator in $L^2(\mathcal{M})$, having discrete spectrum, with eigenvalues $\{\l_n\}_{n=1}^{\infty}$ and corresponding eigenvectors $\{\l_n,\phi_n\}_{n=1}^{\infty}$. Here, the eigenvectors $\l_n$ are in increasing order of $\abs{\l_n}$ and have repetitions if the corresponding eigenspace is more than one dimensional. Denote the Paley-Wiener spaces by $PW(\l)$, with projections $P(\l)$. Denote by $M_{\l}$ the index such that $\l_{M_{\l}}$ is the largest eigenvalue in its absolute value satisfying $\l_{M_{\l}}\leq \abs{\l}$.

Let
\[G_n=\{x^n_k\}_{k=1}^{N_n} \subset\mathcal{M}\quad , \quad n\in\NN\]
be a sequence of \textbf{sample sets}, 
where $N_n\in\NN$ for every $n\in\NN$. For the following analysis let us fix $n\in\NN$. We see $G_n$ as the nodes of a graph. Instead of analyzing the graph Laplacian through the graph adjacency matrix, we directly analyze the graph Laplacian.
Consider a diagonalizable operator $\DD_n$ in each $L^2(G_n)$, that we call graph Laplacian. The graph Laplacian represents the diffusion or shift kernel in $L^2(G_n)$, and hence encapsulates some notion of geometry in $L^2(G_n)$. A non symmetric Laplacian indicates that the space $L^2(G_n)$ samples $L^2(\mathcal{M})$ non-uniformly, as described in Subsection \ref{Continuous and sampled Laplacians}. Denote the eigendecomposition of $\DD_n$ with eigenvalues $\k^n_j$ and eigenvector $\boldsymbol{\gamma}^n_j$. Let  $\boldsymbol{\Gamma}_n$ be the eigenvector matrix with columns $\boldsymbol{\gamma}^n_j$. Consider the inner product $\ip{{\bf u}}{{\bf v}}_{L^2(G_n)}= {\bf v}^{\rm H} {\bf B}_n {\bf u}$ as defined in (\ref{eq:self1}), with ${\bf B}_n=\boldsymbol{\Gamma}^{-{\rm H}}_n\boldsymbol{\Gamma}^{-1}_n$. When writing $L^2(G_n)$ we mean the space with the inner product $\ip{{\bf u}}{{\bf v}}_{L^2(G_n)}$.
Here, for normal $\DD_n$, ${\bf B}_n={\bf I}$, and $\ip{{\bf u}}{{\bf v}}_{L^2(G_n)}$ is the standard dot product.

The following construction is defined for a fixed Paley-Wiener space $PW(\l)$.
We start by defining the \textbf{evaluation operator}, that evaluates signals in $PW(\l)$ at the sample set $G_n$.
Since general signals in $L^2(\mathcal{M})$ are only defined up to a set of finite measure, to define the evaluation of signals at points, we restrict ourselves to continuous signals. Let $C(\mathcal{M})$ be the Banach space of continuous functions with the infinity norm. The space $C(\mathcal{M})$ is dense in $L^2(\mathcal{M})$. Note that delta functionals that evaluate at a point are well defined on $C(\mathcal{M})$, as elements of the continuous dual $C(\mathcal{M})^*$. Thus, the sampling operator $S_n$ that evaluates at the sample points $\{x_k^n\}_k$ is a well defined bounded operator from $C(\mathcal{M})$ to $L^2(G_n)$. Since in our analysis we work in Paley-Wiener spaces, we consider the natural assumption that the Laplacian respects continuity.
\begin{definition}
\label{def:respectCont}
The Laplacian
$\cL$ is said to \textbf{respect continuity}
 if $PW(\l)$ is a subspace of $C(\mathcal{M})$ for every $\l>0$. 
\end{definition}
Note that Laplace-Beltrami operators on compact manifolds respect continuity, since their domain ($L^2$ functions with distributional Laplacian in $L^2$) is a subspace of $C(\mathcal{M})$.

We define the evaluation operator.
\begin{definition}
Under the above construction, the \textbf{evaluation operator}
$\Phi_n^{\l}:PW(\l)\rightarrow L^2(G_n)$ is defined by
\begin{equation}
\Phi_n^{\l}f=\big(\frac{1}{\sqrt{h_n}}f(x^n_k)\big)_{k=1}^{N_n},
\label{eq:samp_def_ev}
\end{equation}
where 
\begin{equation}
h_n=\frac{N_n}{\mu(\mathcal{M})}
\label{eq:hn1}
\end{equation} 
is the density of $G_n$ in $\mathcal{M}$.
\end{definition}
Consider the Fourier basis $\{\phi_m\}_{m=1}^{M_{\l}}$ of $PW(\l)$. Note that (\ref{eq:samp_def_ev}) can be written in this basis in the matrix form
$\boldsymbol{\Phi}_n^{\l}$,  with entries 
\begin{equation}
\phi_{k,m}=\frac{1}{\sqrt{h_n}}\phi_m(x^n_k).
\label{eq:phi_km}
\end{equation}
For a column vector ${\bf c}=(c_m)_{m=1}^{M_{\l}}$ and $f=\sum_{m=1}^{M_{\l}}c_m\phi_{m}$,  observe that
\[\Phi_n^{\l}f=\boldsymbol{\Phi}_n^{\l}{\bf c}.\]

When defining sampling and interpolation, one should address the non-uniform density of the sample set entailed by the inner product (\ref{eq:self1}).
We thus consider the following definitions of sampling and interpolation.
\begin{definition}
\label{def:samp_int}
Under the above construction,
\textbf{sampling}  $S^{\l}_n:PW(\l)\rightarrow L^2(G_n)$ is defined to be the evaluation operator, with the matrix representation,  where the input is in the Fourier basis $\{\phi_m\}_{m=1}^{M_{\l}}$ and the output in the standard basis of $L^2(G_n)$,
\begin{equation}
{\bf S}_n^{\l}=\boldsymbol{\Phi}_n,
\label{eq:samp_def002}
\end{equation}
where $\boldsymbol{\Phi}_n$ is a matrix with entries (\ref{eq:phi_km}). 
\textbf{Interpolation}  $R_n^{\l}: L^2(G_n)\rightarrow PW(\l)$ is defined as the operator with matrix representation, where the input is in the standard basis of $L^2(G_n)$ and the output is in the Fourier basis of $PW(\l)$,
\begin{equation}
{\bf R}_n^{\l}=\boldsymbol{\Phi}_n^{\rm H}{\bf B}_n.
\label{eq:int_def00002}
\end{equation}
\end{definition}

\begin{claim}
\label{c:RisSstar}
The interpolation operator satisfies
\begin{equation}
R_n^{\l} = (S_n^{\l})^*.
\label{eq:RisSstar}
\end{equation}
\end{claim}

\begin{proof}
Let us derive a general formula for the adjoint of a linear mapping $PW(\l)\rightarrow L^2(G_n)$, represented as a matrix operator $\mathbf{A}$, where $PW(\l)$ is represented in the Fourier basis, and $L^2(G_n)$ in the standard basis. Note that the inner product in $PW(\l)$, represented in the Fourier basis, is the standard dot product. Thus, for any $\mathbf{c}\in \CC^{M_{\l}}$ and $\mathbf{q}\in L^2(G_n)\cong\CC^{N_n}$,
\[\ip{\mathbf{A}\mathbf{c}}{\mathbf{q}}_{L^2(G_n)} = \mathbf{q}^{\rm H}\mathbf{B}_n\mathbf{A}\mathbf{c}= (\mathbf{A}^H\mathbf{B}_n\mathbf{q})^{\rm H}\mathbf{c} =  \ip{\mathbf{c}}{\mathbf{A}^H\mathbf{B}_n\mathbf{q}}_{\CC^{M_{\l}}}.\]
Therefore
\[\mathbf{A}^*= \mathbf{A}^H\mathbf{B}_n.\]
From this, (\ref{eq:RisSstar}) follows as a special case.
\end{proof}

Next, we would like to find a condition, for $f=\sum_{m=1}^{M_{\l}}c_m\phi_m$, 
 that guarantees
\begin{equation}
\label{eq:temp4y375347l}
    {\bf R}_n^{\l}{\bf S}_n^{\l}{\bf c} \xrightarrow[n \to \infty]{} {\bf c}.\
\end{equation}
By requiring (\ref{eq:temp4y375347l}) for all elements of the Fourier basis,  writing (\ref{eq:temp4y375347l}) using entry-wise limits, and arranging all limits as the entries of  a matrix, we obtain the condition
\begin{equation}
\Big(\frac{\mu(\mathcal{M})}{N_n}\ip{\Big(\phi_{m}(x^n_k)\Big)_k}{\Big(\phi_{m'}(x^n_k)\Big)_k}_{L^2(G_n)}\Big)_{m,m'}\xrightarrow[n \to \infty]{} {\bf I}.
\label{eq:quad0000}
\end{equation}
The left hand side of (\ref{eq:quad0000}) is interpreted as a quadrature approximation of the inner product $\ip{\phi_m}{\phi_{m'}}_{L^2(\mathcal{M})}$, based on the sample points $\{x^n_k\}_{k=1}^{N_n}$ and their density. We summarize this in a definition.

\begin{definition}
\label{quadrature_sequence2}
Consider the above construction and notations. Denote by $\ip{\boldsymbol{\Phi}_n}{\boldsymbol{\Phi}_n}\in\CC^{M_{\l}\times M_{\l}}$ the matrix with entries $\ip{\Phi^{\l}_n\phi_m}{\Phi^{\l}_n\phi_{m'}}_{L^2(G_n)}$.
The pair $\{G_n,\DD_n\}_{n=1}^{\infty}$ is called a \textbf{quadrature sequence with respect to reconstruction}, if 
\begin{equation}
\ip{\boldsymbol{\Phi}_n}{\boldsymbol{\Phi}_n} \xrightarrow[n \to \infty]{}  {\bf I}.
\label{eq:quad2456}
\end{equation}
\end{definition}

Next we prove that quadrature sequences are asymptotically reconstructive and bounded.
\begin{proposition}
\label{P:quad1}
Consider the above construction and notations, with $\{G_n,\DD_n\}_{n=1}^{\infty}$ a quadrature sequence and $\mathcal{L}$ that respects continuity. Then sampling and interpolation are asymptotically reconstructive and bounded (Definitions \ref{As_RS} and \ref{As_RS2}), with bound $C=1$ in Definition \ref{As_RS2}.
\end{proposition}

\begin{proof}
The proof of Definition \ref{As_RS} is given by the above analysis. For Definition \ref{As_RS2}, Definition \ref{quadrature_sequence2}  asserts that $\mathbf{S}_n^{\l}$   approximates an isometric embedding.
More accurately, for two vectors $\mathbf{c}_1,\mathbf{c}_2$ of Fourier coefficients,  by (\ref{eq:quad2456})
\[\ip{\mathbf{S}^{\l}_n \mathbf{c}_1}{\mathbf{S}^{\l}_n \mathbf{c}_2} =   \mathbf{c}_2^{\rm H} (\mathbf{S}^{\l}_n)^{\rm H}\mathbf{B}_n \mathbf{S}^{\l}_n \mathbf{c}_1\]
\[ =  \mathbf{c}_2^{\rm H} \ip{\boldsymbol{\Phi}_n}{\boldsymbol{\Phi}_n}  \mathbf{c}_1  .\]
For $\mathbf{c}_1=\mathbf{c}_2=\mathbf{c}$ we have
\[\norm{\mathbf{S}^{\l}_n} = \norm{\ip{\boldsymbol{\Phi}_n}{\boldsymbol{\Phi}_n}^{1/2}} .\]
Thus, since $PW(\l)$ has a fixed finite dimension with-respect-to $n$, and since convergence in matrix norm is equivalent to entry-wise convergence, by Definition \ref{quadrature_sequence2} we have
\[\norm{\mathbf{S}^{\l}_n} = \norm{\ip{\boldsymbol{\Phi}_n}{\boldsymbol{\Phi}_n}^{1/2}} \xrightarrow[n \to \infty]{} 1.\]
Finally, by Claim \ref{c:RisSstar}, $\mathbf{R}_n^{\l}= (\mathbf{S}_n^{\l})^*$, and thus $\norm{\mathbf{R}_n^{\l}}=\norm{\mathbf{S}_n^{\l}}$.
\end{proof}

\subsection{Asymptotic commutativity of sampling and activation functions}
\label{Pointwise convergence of sampling with respect to continuous activation functions}

In this section we prove that Sampling asymptotically commutes with the activation function (Definition \ref{def:pointwise convergent}) under some quadrature conditions. 
Definition \ref{def:pointwise convergent} involves a term of the form
\begin{equation}
\norm{\rho(S^{\l}_nP(\l)f)-S_n^{\l'}P(\l')\rho(P(\l)f)}.
\label{eq:rrrg6}
\end{equation}
Let us first show how to swap the order between sampling and $\rho$ in $\rho(S^{\l}_nP(\l)f)$. 
For every continuous $\rho:\CC\rightarrow\CC$ and $f\in C(\mathcal{M})$, we also have $\rho(f)\in C(\mathcal{M})$. Moreover, $S_n \rho(f)=\rho(S_n f)$ for every continuous $f$.
Thus, assuming that $\cL$ respects continuity, sampling $S_n^{\l}=S_n$ does not depend on $\l$, and $\rho(S_n^{\l}P(\l)f)=\rho(S_n P(\l)f)=S_n \rho(P(\l)f)$ for any continuous activation function $\rho$. 
As a result, for continuous $\rho$, (\ref{eq:rrrg6}) takes the form
\begin{equation}
\norm{\rho(S_n^{\l}P(\l)f)-S_n^{\l'}P(\l')\rho(P(\l)f)}=\norm{S_n\rho(P(\l)f)-S_n P(\l')\rho(P(\l)f)}.
\label{eq:34534hhhhhh}
\end{equation}
The right hand side of (\ref{eq:34534hhhhhh}) can be seen as a quadrature approximation of \newline 
$\norm{\rho(P(\l)f)-P(\l')\rho(P(\l)f)}$, which leads us to the following assumption.
\begin{definition}
\label{quadratureWRTrho}
The sampling operators $\{S_n^{\l}\}_{\l>0}$ are said to be \textbf{quadrature with respect to the continuous activation function} $\rho$, if $\cL$ respects continuity, and for every  $f\in L^2(\mathcal{M})$ and $\l'>\l>0$,
\[\lim_{n\rightarrow\infty}\norm{S_n\rho(P(\l)f)-S_n P(\l')\rho(P(\l)f)}=\norm{\rho(P(\l)f)-P(\l')\rho(P(\l)f)}.\]
\end{definition}

Next, we focus on a common class of activation functions, that include ReLU, absolute value, and absolute value or ReLU of the real or imaginary part of a complex number.
\begin{definition}
Consider the field $\RR$ or $\CC$, and denote it by $\mathbb{F}$.
The continuous activation function $\rho:\mathbb{F}\rightarrow\mathbb{F}$ is called \textbf{positively homogeneous} of degree $1$, if for every $z\in\mathbb{F}$ and every real $c\geq 0$,
\[\rho(cz)=c\rho(z).\]
\end{definition}

\begin{proposition}
\label{Prop_comm4}
Consider a DSP framework, quadrature with respect to reconstruction. Consider a contractive positively homogeneous activation function $\rho$ of degree $1$. Suppose that $\cL$ respects continuity and that the sampling operators are quadrature with respect to the continuous activation function $\rho$. Then  sampling asymptotically commutes with $\rho$ (Definition \ref{def:pointwise convergent}).
\end{proposition}

The proof is in Appendix \ref{Ap3}.

\subsection{Convergence of sampled Laplacians to topological space Laplacians}
\label{Continuous and sampled Laplacians}

In this subsection we discuss different definitions of topological Laplacians and their discretizations to graph Laplacians via sampling. We show convergence of the graph Laplacians to the topological-measure Laplacians, in the sense of Definition \ref{As_convS}, under a quadrature assumption.

Assume that $\mathcal{M}$ is a compact metric space with finite Borel measure $\mu(\mathcal{M})<\infty$. Since such a measure space is a probability space up to  normalization, we assume that $\mu(\mathcal{M})=1$.
Let $S_r(x_0),B_r(x_0)$ denote the sphere and ball or radius $r$ about $x_0$ respectively.
One definition of the Laplacian in the Euclidean space of dimension $d$ is
\[\cL f(x_0) := \lim_{r\rightarrow 0} \frac{2d}{r^2}\Big(A\big(S_r(x_0)\big)^{-1}\int_{S_r(x_0)}f(x)dx - f(x_0)\Big).\]
By integrating on the radius $r'$, from $0$ to $r$, with weights $V\big(S_{r'}(x_0)\big)^{-1}A\big(S_{r'}(x_0)\big)$, and using the mean value theorem for integrals, we obtain the equivalent definition
\[\cL f(x_0) = \lim_{r\rightarrow 0} V\big(B_r(x_0)\big)^{-1}\int_{B_r(x_0)}\frac{2d}{\abs{x-x_0}^2}\big(f(x)-f(x_0)\big)dx.\]
Another equivalent definition for the Laplace-Beltrami operator on manifolds of dimension $d$ is
\[\cL f(x_0)=\lim_{r\rightarrow 0} (2d+2)V\big(B_r(x_0)\big)^{-1}r^{-2}\int_{B_r(x_0)}\big(f(x)-f(x_0)\big)dx.\]
This motivates two classes of Laplacians in general metric-measure spaces. First, an infinitesimal definition
\begin{equation}
\cL f(x_0) = \lim_{r\rightarrow 0} V\big(B_r(x_0)\big)^{-1}r^{-2}\int_{B_r(x_0)}H(x_0,x)\big(f(x)-f(x_0)\big)dx,
\label{eq:IntOp1}
\end{equation}
where a prototypical example is $H(x_0,x)=1$, for which (\ref{eq:IntOp1}) are termed \textbf{Korevaar-Schoen type energies} \cite{M_Laplace0}.
Second, a non-infinitesimal definition
\begin{equation}
\cL f(x_0) = \int_{\mathcal{M}}H(x_0,x)\big(f(x)-f(x_0)\big)dx,
\label{eq:IntOp2}
\end{equation}
where a prototypical example is $H(x_0,x)= V\big(B_r(x_0)\big)^{-1}r^{-2}\chi_{B_r(x_0)}$ for some fixed radius $r$. Here, $\chi_{B_r(x_0)}$ is the characteristic function of the ball $B_r(x_0)$. 
Formulas (\ref{eq:IntOp1}) and (\ref{eq:IntOp2}) define symmetric operators in case $H(x,x_0)=H(x_0,x)$. Indeed, (\ref{eq:IntOp2}) is a sum of an integral and a multiplicative operator, both symmetric. Moreover, the symmetric property is preserved under limits in (\ref{eq:IntOp1}), since the limit commutes with the inner product.

In \cite{M_Laplace1} it was shown, under some mild conditions, that (\ref{eq:IntOp2}) with $H(x,x_0)=r^{-2}\chi_{B_r(x_0)}$ is a self-adjoint operator with spectrum supported in $[0,2r]$. Moreover, the part of the spectrum in $\left.\left[0,\right.r\right)$ is discrete, and the eigenvalues of the sampled Laplacian in $\left.\left[0,\right.r\right)$ converge to the eigenvalues of the continuous Laplacian, assuming that sampling becomes denser in $n$ in some sense. 

The advantage of Laplacians of the form (\ref{eq:IntOp2}) is that they are readily discretizable on sample sets, by approximating the integral in (\ref{eq:IntOp2}) by a sum over the sample set. Suppose that $H$ is symmetric ($H(x,x_0)=H(x_0,x)$), and consider a continuous weight function $w:\mathcal{M}\rightarrow \RR_+$. For a detailed explanation of the role of $w$ we refer to Subsection \ref{Random sample sets satisfy the quadrature definitions}.
Given a sample set $G_n=\{x_k^n\}_{k-1}^{N_n}$, define the discrete Laplacian $\DD_n$ acting on a vector $\mathbf{q}$ by
\begin{equation}
[\DD_n \mathbf{q}]_k = \frac{1}{\sqrt{N_n}}\sum_{k'=1}^{N_m}\frac{1}{w(x^n_{k'})}H(x^n_k,x^n_{k'})q_{k'}.
\label{eq:DLap2}
\end{equation}
For $q_{k'}=f(x_{k'}^n)$, (\ref{eq:DLap2}) is interpreted as a quadrature approximation of (\ref{eq:IntOp2}).
It is easy to show that the inner product (\ref{eq:self1}) under which $\DD_n$ is self-adjoint is based on
\begin{equation}
\mathbf{B}_n={\rm diag}\{\frac{1}{N_n w(x^n_k)}\}_{k=1}^{N_n},
\label{eq:Bdiag}
\end{equation}
where $\mathbf{A}={\rm diag}\{v_j\}_{j=1}^N$ is the diagonal matrix with diagonal entries $a_{j,j}=v_j$.

 For our analysis, we relax the assumption that $\mathcal{M}$ is a compact metric space to a compact topological space. We further assume that the Laplacian $\cL$ has discrete spectrum in the sense of Definition~\ref{discreteL}. 
 However, for continuous $H$ on a compact topological space $\mathcal{M}$, any Laplacian (\ref{eq:IntOp2}) is bounded, and thus has a discrete spectrum in the sense of Definition \ref{discreteL}  only if the range of $\cL$ is finite-dimensional. We thus approximate Laplacians $\cL$ having discrete spectrum in two steps. First, by a finite-dimensional Laplacian of the form (\ref{eq:IntOp2}), and then, by the discretization (\ref{eq:DLap2}).

 The approximation of $\cL$ by a finite-dimensional Laplacian works as follows. \newline
Let $\{\l_m,\phi_m\}_{m=1}^{\infty}$ be the eigendecomposition of $\cL$, and $\overline{\l}$ be some large band. Denote $\overline{M}=M_{\overline{\l}}$. We define the integral operator
\begin{equation}
\cL^{\overline{\l}}f(x_0) = \int_x H(x_0,x)f(x) dx
\label{eq:NALF}
\end{equation}
based on the kernel
\begin{equation}
H_{\overline{\l}}(x_0,x)= \sum_{m=1}^{\overline{M}} \phi_m(x_0)\l_m\overline{\phi_m(x)}.
\label{eq:NALF2}
\end{equation}
It is easy to see that 
\begin{equation}
\cL^{\overline{\l}} = \cL P(\overline{\l}).
\label{eq:Dapp23}
\end{equation}
Therefore, for every $f\in L^2(\mathcal{M})$, we have $\lim_{\overline{\l}\rightarrow\infty}\cL^{\overline{\l}}f = \cL f.$
Moreover, by (\ref{eq:Dapp23}) for every $f\in PW(\l)$ with $\l<\overline{\l}$, we have
$\cL^{\overline{\l}}f = \cL f$.

We then treat the total approximation of $\cL$ by a graph Laplacian using some sort of a diagonal extraction method. This is explained in Theorem \ref{Theo:probMC} of Section \ref{Random sample sets satisfy the quadrature definitions}. 
For now, let us focus on the non-asymptotic Laplacian $\cL^{\overline{\l}}$ of (\ref{eq:NALF}) with discrete spectrum, denoted by abuse of notation by $\cL$ where $\l$ is fixed.
To guarantee that the sequence of graph Laplacians as sampling operators are convergent (Definition \ref{As_convS}) we consider the following quadrature assumption.

\begin{definition}
\label{Def_quad3}
Under the above construction, $G_n=\{x_k^n\}_{k-1}^{N_n}$ is a \textbf{quadrature sequence with respect to} $\cL$, if for every $P(\l) f \in PW(\l)$
\[\lim_{n\rightarrow\infty}\norm{ S_n^{\l}\cL P(\l) f  - \DD_n S_n^{\l}P(\l) f }_{L^2(G_n)} = 0. \]
\end{definition}

\begin{proposition}
\label{PropDef3l}
Consider the above construction, with radon space $\mathcal{M}$, Laplacian $\cL$ with discrete spectrum, and Paley-Wiener projections $P(\l)$. Consider a sampling sequence $\{S^{\l}_n\}_{n,\l}$ based on the sample points $G_n$, $n=1,\ldots,\infty$, where $G_n$ is quadrature sequence with respect to $\cL$.
Then $\DD_n$ converges to $\cL$ in the sense of Definition \ref{As_convS}.
\end{proposition}

\begin{proof}

The operator $A_n=S_n^{\l}\cL - \DD_n S_n^{\l}$ maps the $M_{\l}$ dimensional space $PW(\l)$ to an $M_{\l}$ dimensional  space $W_n\subset L^2(G_n)$ containing the space $A_n PW(\l)$.
Consider an isometric isomorphism $Q_n:W_n\rightarrow PW(\l)$. 
The operators $Q_nA_n:PW(\l)\rightarrow PW(\l)$ converge to zero as $n\rightarrow\infty$ in the strong topology, and since  $PW(\l)$ is finite-dimensional, $Q_nA_n$ converge to zero also in the operator norm topology. Thus, since $Q_n$ preserves norm, $A_n$ converges to zero in the operator norm topology, which proves convergence as defined in Definition \ref{As_convS}.
\end{proof}

\subsection{Transferability of random graph Laplacians}
\label{Random sample sets satisfy the quadrature definitions}

In this section we show that under some setting of random sampling of Laplacians $\cL$ that respect continuity, graph Laplacian, sampling operators, and interpolation operators are asymptotically reconstructive, bounded, convergent, and sampling asymptotically commutes with the activation function (Definitions \ref{As_RS},\ref{As_RS2},\ref{As_convS} and \ref{def:pointwise convergent}). 
To model the arbitrariness in which graphs can be sampled from topological-measure spaces, we suppose that the sample points  $\{x_k^n\}_{k-1}^{N_n}$ are chosen at random. This allows us to treat the graph Laplacians as Monte-Carlo approximations of the topological-measure Laplacian.

 Let $f=P(\l)f\in PW(\l)$. Consider a weighted $\mu$ measure, $\mu_w$, defined for measurable sets $X\subset\mathcal{M}$ by
\begin{equation}
\mu_w(X) := \int_X w(x)d\mu(x).
\label{eq:mu_w}
\end{equation}
Here, the weight function $w:\mathcal{M}\rightarrow \RR$ is positive, continuous, and satisfies
\[\int_{\mathcal{M}} w(x)d\mu(x)=1.\]
 We take $\{x_k^n\}_{k-1}^{N_n}$ as random points in the probability space $\{\mathcal{M},\mu_w\}$. 

\begin{definition}
Let $\{\mathcal{M},\mu\}$ be a compact topological-measure space with $\mu(\mathcal{M})=1$. Let the weighted measure $\mu_w$ satisfy (\ref{eq:mu_w}). Let $\cL$ be a symmetric Laplacian  of the form (\ref{eq:IntOp2}), such that $H\in L^2(\mathcal{M}^2)$.
Suppose that $\cL$ respects continuity and has discrete spectrum. 
Let $\{x_k^n\}_{k-1}^{N_n}$ be $N_n$ random points from the probability space $\{\mathcal{M},\mu_w\}$.
The \textbf{random sampled Laplacian} $\DD_n$ is a random variable $\{\mathcal{M}^{N_n}; \mu_w^{N_n}\}\rightarrow \CC^{N_n\times N_n}$,  defined by (\ref{eq:DLap2}) for the random samples $\{x_k^n\}_{k-1}^{N_n}$.
The \textbf{random sampling and interpolation operators} $S_n^{\l},R_n^{\l}$ are defined as in Definition \ref{def:samp_int} on the random points $\{x_k^n\}_{k-1}^{N_n}$, with the inner product structure (\ref{eq:Bdiag}) of $L^2(G_n)$. 
\end{definition}

For Theorem \ref{Theo:probMC} below, we need one more assumption on $\rho$ and $\cL$. 
Let us consider for motivation the standard Laplacian $\cL$ on the unit circle, and the ReLU activation function. Consider the classical Fourier basis $\{\phi_n\}_{n=-\infty}^{\infty}$. Any $f\in PW(\l)$ is smooth, and $\rho(f)$ is piecewise smooth and continuous. Thus $\rho(f)$ can be differentiated term-by-term, and
\[\norm{\partial_x\rho(f)}_2^2 = 4\pi^2\sum_{n=-\infty}^{\infty} n^2\abs{\ip{\rho(f)}{\phi_n}}^2.\]
On the other hand, observe that for ReLU
\begin{equation}
\norm{\rho(f)}_2 \leq \norm{f}_2 \ , \quad \norm{\partial_x\rho(f)}_2 \leq \norm{\partial_x f}_2.
\label{eq:SMe}
\end{equation}
Thus
\begin{equation}
\sum_{n=-\infty}^{\infty} n^2\abs{\ip{\rho(f)}{\phi_n}}^2 \leq \sum_{n=-M_{\l}}^{M_{\l}} n^2\abs{\ip{f}{\phi_n}}^2 \leq M_{\l}^2\norm{f}_2^2. 
    \label{eq_yrmp86003}
\end{equation}
We can now show  the following claim
\begin{claim}
\label{ReLU_rexpects1}
The ReLU function $\rho$ is a continuous mapping of signals from $PW(\l)$ to signals in the norm 
\begin{equation}
\norm{h}_{1+\k,2}=  \sqrt{\abs{\ip{h}{\phi_0}}^2+\sum_{n=-\infty}^{\infty} \abs{n}^{1+\k}\abs{\ip{h}{\phi_n}}^2}
\label{eq:PSDex}
\end{equation}
for any $0<\k<1$.
\end{claim}
The proof of this claim in in Appendix \ref{ReLU_rexpects1A}.

This analysis motivates the following definition in the general case.

\begin{definition}
\label{def:PSD}
The activation function $\rho$ is said to \textbf{preserve spectral decay} if there exists $\k>0$ such that for every $\l$, the activation function $\rho$ applied on signals from $PW(\l)$ is continuous in the norm 
\begin{equation}
\norm{h}_{\k,2}=  \sqrt{\sum_{n=1}^{\infty} \abs{n}^{1+\k}\norm{\phi_n}_{\infty}^2\abs{\ip{h}{\phi_n}}^2}.
\label{eq:PSD}
\end{equation}
\end{definition}

Note that in the finite-dimensional domain $PW(\l)$, all norms are equivalent. Thus, for $\rho$ that preserves spectral decay,
\begin{equation}
\lim_{\norm{f-g}_{2}\rightarrow 0}\sqrt{\sum_{n=1}^{\infty} \abs{n}^{1+\k}\norm{\phi_n}_{\infty}^2\abs{\ip{\rho(f)-\rho(g)}{\phi_n}}^2}=0,
\label{eq:PSD2}
\end{equation}
where the limit is over $f,g\in PW(\l)$.

Preservation of spectral decay is interpreted as follows. Applying $\rho$ on a band-limited signal $f\in PW(\l)$ results in a continuous signal which is not band-limited and in general has frequency coefficients in all frequencies. Namely, after applying $\rho$ on $f$, which decays rapidly in the frequency domain, $\rho(f)$ is not guaranteed to decay rapidly. However, under Definition \ref{def:PSD}, $\rho(f)$ is guaranteed to have some decay rate in the frequency domain, since the weighted sum (\ref{eq:PSD}), with weights increasing to $\infty$ in frequency, is finite.

The following notation is used in the asymptotic analysis in Theorem \ref{Theo:probMC}.
For any $M\in\NN$ denote
\begin{equation}
\norm{\boldsymbol{\l}^{M}}_1 = \sum_{m=1}^{M}\abs{\l_m}.
\label{eq:norm_l}
\end{equation}

\begin{theorem}
\label{Theo:probMC}
Let $\{\mathcal{M},\mu\}$ be a probability topological-measure space, and $\mu_w$ another measure satisfying (\ref{eq:mu_w}) with positive and continuous $w$. 
Let $\cL$ be a topological-measure Laplacian with discrete spectrum that respects continuity. Let $\rho$ be a contractive positively homogeneous of degree 1 activation function that preserves spectral decay.
Consider a sequence of random $\mu_w$ sample sets $\{x^n_k\}_{n=1}^{N_n}$, $n\in\NN$, with $N_n\xrightarrow[n \to \infty]{} \infty$.  Then, for every series of bands $\overline{\l}_n\xrightarrow[n \to \infty]{} \infty$, such that $\norm{\boldsymbol{\l}^{M_{\overline{\l}_n}}}_1=o(N_n^{1/2})$, and random sampled Laplacians $\DD_n=\DD_n^{\overline{\l}_n}$ with $\cL^{\overline{\l}_n}$ defined by (\ref{eq:NALF}) and (\ref{eq:NALF2}), and for every $\d>0$, in probability 1 there exists a subsequence $n_m\subset \NN$ such that the followind holds:
\begin{enumerate}
	\item[\textbf{i}] for every $n\in\NN$ we have $n\in\{n_m\}_{m\in\NN}$ in probability more than $(1-\d)$, and
	\item[\textbf{ii}]
	the sampled Laplacians $\{\DD_{n_m}\}_m$ satisfy Definitions \ref{As_RS},\ref{As_RS2},\ref{As_convS} and \ref{def:pointwise convergent}.
\end{enumerate}
\end{theorem}

\begin{remark}
The sequence of random sample sets is treated formally in the following fashion.
The basis of the topology of a sequence of topological spaces is defined as follows. A generic set in the basis of the topology is constructed by choosing finitely many indexes and picking an open set for each of the corresponding spaces. For each of the rest of the indexes we pick the whole corresponding probability space. The measure of such sets is the product of the measures of the sets of the finite subsequence.
\end{remark}

By Theorems \ref{main00} and \ref{main_conv0}, Theorem \ref{Theo:probMC} is interpreted as follows. 
If $\DD_n$ are sampled from $\cL$ by drawing $N_n$ random sample points and sampling band-limited approximations of $\cL$, where the bands do not increase too fast with respect to $N_n$, then graph filters and ConvNets approximate topological-measure filters and ConvNets. Therefore, graph filters and ConvNets are transferable. \textcolor{black}{The explicit bounds on different transferability terms in (\ref{eq:Th14}) of Theorem \ref{main_conv00} are given in Appendix \ref{Proof of Theorem Theo:probMC}.}

\textcolor{black}{Last, let us use the results in Theorem \ref{Theo:probMC} and  Appendix \ref{Proof of Theorem Theo:probMC} to derive non-asymptotic bounds for the transferability error of filters.}
\textcolor{black}{
\begin{proposition}
\label{Prop:MC_rate}
Consider the setting of Theorem \ref{Theo:probMC}, where we choose $\l_n$ such that
\[\norm{\boldsymbol{\l}^{M_{\overline{\l}_n}}}_1 \leq B N_n^{1/2-\alpha},\]
where $B>0$ is some constant, and $\alpha\in (0,1/2]$. Let $g$ be a Lipschitz  continuous filter, with Lipschitz  constant $D$, and  let $\norm{g}_{\mathcal{L},M}$ be as defined in (\ref{eq:glm}). Denote $w_{\min}=\min_{x\in\mathcal{M}}w(x)$.   Then, for each $n$, with probability more than $1-2\d$,
\begin{equation}
\label{eq:nonAsBoundLast}
    \begin{split}
     & \norm{g(\mathcal{L})P(\lambda) - R_n^{\lambda}g(\boldsymbol{\Delta}_n) S_n^{\lambda} P(\lambda)} \\
     & \leq M_{\l}\Big(2DB w_{\min}^{-1}\max_{m\leq M_{\l}}\norm{\phi_m}_{\infty} N_n^{-\alpha}+\norm{g}_{\mathcal{L},M}w_{\min}^{-1/2}\max_{m\leq M_{\l}}\norm{\phi_m}_{\infty}^2N_n^{-1/2}\Big)\delta^{-1/2}.
    \end{split}
\end{equation}
\end{proposition}
In Proposition \ref{Prop:MC_rate}, different choices of $\alpha\in(0,1/2]$ correspond to different choices of the Laplacian discretization. The choice $\alpha=1/2$ means that we discretize a fixed Paley-Wiener projection of $\cL$, and the closer $\alpha$ is to $0$, the faster the band of $\cL$ that we approximate goes to infinity in $n$.
}

\section{\textcolor{black}{Conclusion}}

\textcolor{black}{In this paper, we proved that spectral graph filters and ConvNets are transferable. We took the philosophical point of view in which a ConvNet is called transferable, if, whenever two graphs represent the same phenomenon, the ConvNet has approximately the same repercussion on both graphs. We modeled mathematically ``graph representing a phenomenon'' as a graph which is sampled from an underlying ``continuous'' Borel space. Here, sampling is treated very broadly, and two examples are sampling by evaluating at sample points, and graph coarsening. We modeled mathematically ``ConvNet having approximately the same repercussion'' via the sampling-interpolation approach. Using this model, we were able to prove that spectral ConvNets are transferable. It is interesting to note that, after the publication of the current paper, \cite{experimental} tested ChebNet, a spectral ConvNet, on a set of multi-graph benchmark problems, with the goal of testing our results experimentally. The results showed that ChebNet outperforms vanilla spatial methods, especially in settings where the graphs are synthetically generated from an underlying continuous model. This validates that spectral methods indeed have competitive transferability capabilities in practice.}

\textcolor{black}{We believe that our paradigm of treating transferability by modeling ``graphs representing the same phenomenon'' and ``ConvNet having the same repercussion'' is a good starting point for any future research on graph ConvNet transferability. Such research should focus on modeling these concepts mathematically, justifying the model experimentally or heuristically, and proving corresponding transferability error bounds.}



\subsection*{Acknowledgments}
R.L. acknowledges support by the DFG SPP 1798 “Compressed Sensing
in Information Processing” through Project Massive MIMO-II.

M.B. is partially supported by the ERC Consolidator grant No. 724228 (LEMAN)

G.K. acknowledges partial support by the Berlin Mathematics Research
Center MATH+ through Project EF1x1, the DFG SPP 1798 “Compressed Sensing
in Information Processing” through Project Massive MIMO-II, and the BMBF
through Project MaGriDo.


\appendix

\section{Laplacians of directed graphs as normal operators}
\label{Laplacians of directed graphs as normal operators}

Next we explain how functional calculus applies as-is to non-normal matrices, even though the theory is defined only for normal operators. As a result, spectral filters can be defined on directed graphs represented by non-symmetric adjacency matrices.

Every finite-dimensional normal operator has an eigendecomposition with complex eigenvalues and orthonormal eigenvectors. Functional calculus applies to finite-dimensional normal operators by (\ref{eq:FC1}), and is canonical in the sense that it is equivalent to compute a rational function of a normal operator by (\ref{eq:FC1}), or by compositions, linear combinations, and inversions by (\ref{eq:FC_ratio}). 
On the other hand, any diagonalizable matrix can be seen as a normal operator, considering an appropriate inner product. Moreover, almost any matrix is diagonalizable. Eigendecomposition and functional calculus are theories of self-adjoint/unitary/normal operators, which need not be represented by symmetric/orthonormal/normal matrices. Thus, spectral graph theory applies also to directed graphs. Note that no eigendecomposition is ever calculated in practice, and all computations in applying filters (compositions, linear combinations, and inversions) are algebraic and do not depend on the inner product structure. Thus, the theory applies as-is on directed graphs, with no extra considerations.
We thus focus on finite-dimensional normal Laplacian operators, which can represent non-symmetric Laplacian matrices on directed graphs. 

Given an $N\times N$ diagonalizable matrix ${\bf A}$ with eigenvectors $\{\boldsymbol{\gamma}_k\}_{k=1}^N$, consider the matrix $\boldsymbol{\Gamma}$ comprising the eigenvectors as columns. Define the inner product 
\begin{equation}
\ip{{\bf u}}{{\bf v}}= {\bf v}^{\rm H} {\bf B} {\bf u},
\label{eq:self100}
\end{equation}
where ${\bf B}=\boldsymbol{\Gamma}^{-{\rm H}}\boldsymbol{\Gamma}^{-1}$ is symmetric,  ${\bf u}$ and ${\bf v}$ are given as column vectors, and for a matrix ${\bf C}=(c_{m,k})_{n,m}\in \CC^{N\times N}$, the Hermitian transpose ${\bf C}^{\rm H}$ is the matrix consisting of entries $c^{\rm H}_{m,k}=\overline{c_{k,m}}$. It is easy to see that (\ref{eq:self100}) defines an inner product for which ${\bf A}$ is normal. Consider an operator $A$ represented by the matrix ${\bf A}$. The adjoint $A^*$ of an operator $A$ is defined to be the unique operator such that
\[\forall {\bf u},{\bf v}\in\CC^{d}, \quad \ip{A{\bf u}}{{\bf v}} = \ip{{\bf u}}{A^*{\bf v}}.\]
 By the equality
\[{\bf v}^{\rm H} {\bf B} {\bf A}{\bf u} = {\bf v}^{\rm H} {\bf B} {\bf A} {\bf B}^{-1}{\bf B}{\bf u} = \big({\bf B}^{-1}{\bf A}^{\rm H}{\bf B}{\bf v}\big)^{\rm H} {\bf B}{\bf u},\]
the matrix representation of the adjoint $A^*$ is given by 
\begin{equation}
{\bf A}^*={\bf B}^{-1}{\bf A}^{\rm H}{\bf B}.
\label{eq:mat_adj100}
\end{equation}
Thus, an operator is self-adjoint if ${\bf B}^{-1}{\bf A}^{\rm H}{\bf B}={\bf A}$, unitary if  ${\bf B}^{-1}{\bf A}^{\rm H}{\bf B}={\bf A}^{-1}$, and normal if 
\[{\bf A}{\bf B}^{-1}{\bf A}^{\rm H}{\bf B} = {\bf B}^{-1}{\bf A}^{\rm H}{\bf B}{\bf A}.\] 

Note the difference between transpose and adjoint, and between 
 symmetric/orthonormal matrices and self-adjoint/unitary operators: a non-symmetric matrix may represent a self-adjoint operator. To emphasize this difference,
we opt in this paper for a Hilbert space formulation of inner products and basis expansions, over the more commonly used formulation in the graph signal processing community of matrix products and dot products.

The eigenvalues and eigenspaces of a diagonalizable matrix, and the eigenvalues and eigenspaces of the corresponding normal operator, are identical. Indeed, eigenvalues and eigenspaces are defined algebraically, independently of the inner product structure. If the eigenvalues of the matrix are real or in $e^{i\RR}$, then the corresponding operator is self-adjoint or unitary respectively.

\section{Proofs}
\label{Proofs}

\subsection{Proof of Theorem \ref{main00}}

By linearity and finite dimension of $PW(\lambda)$, we start with a signal $\phi_{m}\in PW(\l_M)$ which is an eigenvector of $\mathcal{L}$ corresponding to the eigenvalue $\lambda_j$, and then generalize to linear combinations.
Let $Q_k$ be the projection upon the eigenspace of $\boldsymbol{\Delta}$ corresponding to the eigenvalue $\kappa_k$. Then, by $\mathcal{L} \phi_m=\l_m\phi_m$,
\[\boldsymbol{\Delta}S^{\l_M}\phi_m - S^{\l_M}\mathcal{L} \phi_m = \sum_{k}\kappa_kQ_kS^{\l_M}\phi_m - \lambda_m S^{\l_M}\sum_{k}Q_k\phi_m=\sum_{k}(\kappa_k-\lambda_m)Q_kS^{\l_M}\phi_m .\]
By orthogonality of the projections $\{Q_k\}_k$,
\begin{equation}
\norm{\sum_{k}\kappa_kQ_kS^{\l_M}\phi_m - \lambda_m S^{\l_M}\phi_m}^2 =\sum_{k}\abs{\kappa_k-\lambda_m}^2\norm{Q_kS^{\l_M}\phi_m}^2
\label{eq:main001}
\end{equation}
Now, similarly to the derivation of (\ref{eq:main001}), by functional calculus and by (\ref{eq:V_fl}),  
\begin{equation}
\begin{split}
\norm{f(\boldsymbol{\Delta})S^{\l_M}\phi_m - S^{\l_M} f(\mathcal{L}) \phi_m}^2 = & \sum_{k}\abs{f(\kappa_k)-f(\lambda_m)}^2\norm{Q_kS^{\l_M}\phi_m}^2 \\
 = & \sum_{k}\abs{\frac{f(\kappa_k)-f(\lambda_m)}{\kappa_k-\lambda_m}}^2\abs{\kappa_k-\lambda_m}^2\norm{Q_kS^{\l_M}\phi_m}^2\\
\leq &  {\rm V}_f(\l_m)^2\sum_{k}\abs{\kappa_k-\lambda_m}^2\norm{Q_kS^{\l_M}\phi_m}^2\\
= & {\rm V}_f(\l_m)^2 \norm{\sum_{k}\kappa_kQ_kS^{\l_M}\phi_m - \lambda_m S^{\l_M}\phi_m}^2 \\
= & {\rm V}_f(\l_m)^2\norm{\boldsymbol{\Delta}S^{\l_M}\phi_m - S^{\l_M}\mathcal{L} \phi_m}^2,
\end{split}
\label{eq:1}
\end{equation}
which proves Thm.\ref{main00}(\ref{eq:TE_Fmode_G}).

Now,  for $q=\sum_{m}c_m\phi_m$,
we have
\[ \norm{f(\boldsymbol{\Delta})S^{\l_M}P(\lambda_M)q - S^{\l_M}f(\mathcal{L})P(\lambda_M) q} =  \norm{\sum_{m=1}^M c_m\Big(f(\boldsymbol{\Delta})S^{\l_M} - S^{\l_M}f(\mathcal{L})\Big) \phi_m}. \]
By the triangle inequality and Thm.\ref{main00}(\ref{eq:TE_Fmode_G}),
\begin{equation}
\begin{split}
\norm{f(\boldsymbol{\Delta})S^{\l_M} P(\lambda_M)q -S^{\l_M}f(\mathcal{L})P(\lambda_M) q} & \leq   \sum_{m=1}^M\abs{c_m}\norm{\Big(f(\boldsymbol{\Delta})S^{\l_M}\phi_m - S^{\l_M}f(\mathcal{L}) \phi_m\Big)}\\
 & \leq \sum_{m=1}^M\abs{c_m}{\rm V}_f(\l_m)\norm{\boldsymbol{\Delta}S^{\l_M}\phi_m - S^{\l_M}\mathcal{L} \phi_m}
\end{split}
\label{eq:temp2}
\end{equation}
which proves Thm.\ref{main00}(\ref{eq:TE_point_G}). Moreover, by (\ref{eq:temp2}), by ${\rm V}_f(\l_m) \leq D$ by $\norm{\phi_m}=1$ and by H\"older's inequality,
\begin{equation}
\norm{f(\boldsymbol{\Delta})S^{\l_M} P(\lambda_M)q -S^{\l_M}f(\mathcal{L})P(\lambda_M) q} 
\leq \norm{q}_1 D \norm{\boldsymbol{\Delta}S^{\l_M}P(\lambda_M) - S^{\l_M}\mathcal{L} P(\lambda_M) }.
\label{eq:temp2.2}
\end{equation}
Here, $\norm{q}_1:=\sum_{m=1}^M\abs{c_{m}}$. By Cauchy–Schwarz inequality we have
\[\norm{q}_1 \leq \norm{q}_2\sqrt{M},\]
which proves Thm.\ref{main00}(\ref{eq:TE_worst_G}).

By the triangle inequality 
\[
\begin{split}
 & \norm{f(\mathcal{L})P(\lambda_M) - R^{\lambda_M}f(\boldsymbol{\Delta}) S^{\l_M}P(\lambda_M)} \\
 & \leq \norm{f(\mathcal{L})P(\lambda_M) - R^{\lambda_M}S^{\l_M}f(\mathcal{L})P(\lambda_M)} \\
 & \quad +\norm{R^{\lambda_M}S^{\l_M}f(\mathcal{L})P(\lambda_M) - R^{\lambda_M}f(\boldsymbol{\Delta}) S^{\l_M}P(\lambda_M)}\\
 & \leq \norm{P(\lambda_M)-R^{\lambda_M}S^{\l_M}P(\lambda_M)}\norm{f(\mathcal{L})P(\lambda_M)} \\
 & \quad + \norm{R^{\lambda_M}}\norm{S^{\l_M}f(\mathcal{L})P(\lambda_M) - f(\boldsymbol{\Delta}_n) S^{\l_M}P(\lambda_M)}.
\end{split}
\]
Note that by assumption
\[\norm{R^{\lambda_M}}\leq C\]
and, by the diagonal form of $f(\mathcal{L})P(\lambda_M)$,
\[\norm{f(\mathcal{L})P(\lambda_M)} \leq \norm{f}_{\mathcal{L},M},\]
which gives Thm.\ref{main00}(\ref{eq:TE_worst_M}). A similar use of the triangle inequality gives Thm.\ref{main00}(\ref{eq:TE_point_M}).

\subsection{Proof of Theorem \ref{main_conv00}}


First we show that, if $A\neq 1$, then
\[\begin{split}
 \norm{\mathcal{N}^{l}_k P(\psi^0)f}\leq A^l\norm{P(\psi^0)f} + \frac{A^{l}-1}{A-1}B  & {\quad \rm for\ all\ }f\in L^2(\mathcal{M})\\
    \norm{\mathcal{N}^{j,l}_k\tf}\leq A^l\norm{\tf} + \frac{A^{l}-1}{A-1}B & {\quad \rm for\ all\ }  \tf\in L^2(G^{j,0})
	\end{split}
	\]
	and if $A=1$
	\[\begin{split} \norm{\mathcal{N}^{l}_k P(\psi_0)f}\leq \norm{P(\psi_0)f} + (l-1)B &   {\quad \rm for\ all\ }  f\in L^2(\mathcal{M}) \\
   \norm{\mathcal{N}^{j,l}_k\tf}\leq \norm{\tf}+ (l-1)B &  {\quad \rm for\ all\ } \tf\in L^2(G^{j,0})
	\end{split}\]
 for every $l$, $k$ and $j=1,2$. We next focus on $\mathcal{N}^{l}_k$, and remark that we can use similar arguments for $\mathcal{N}^{j,l}_k$.
Note that $\norm{g^l_{k',k}(\mathcal{L})}\leq \norm{g^l_{k',k}}_{\infty}\leq 1$ for every $l,k,k'$. Moreover, 
\[
\begin{split}
\norm{\sum_{k=1}^{K_{l-1}}a^l_{k'k}\ g^l_{k'k}(\mathcal{L})f^{l-1}_k + b^l_{k'}}  & \leq \sum_{k=1}^{K_{l-1}}\abs{a^l_{k'k}}\norm{g^l_{k'k}(\mathcal{L})f^{l-1}_k}+\norm{b^l_{k'}}\\
 & \leq \sum_{k=1}^{K_{l-1}}\abs{a^l_{k'k}}\norm{f^{l-1}_k}+B \\
 & \leq A \max_k\norm{f^{l-1}_k} +B . 
\end{split}
\]
Moreover,
\[\norm{\rho\Big(\sum_{k=1}^{K_{l-1}}a^l_{k'k}\ g^l_{k'k}(\mathcal{L})f^{l-1}_k +b^l_{k'}\Big)} \leq  \norm{\sum_{k=1}^{K_{l-1}}a^l_{k'k}\ g^l_{k'k}(\mathcal{L})f^{l-1}_k + b^l_{k'}}  \leq A \max_k\norm{f^{l-1}_k} +B .\]
Fianlly, using the fact that pooling and projection decreases norm by assumption implies
This shows that 
\[\max_k\norm{f^{l}_k}\leq A\max_k\norm{f^{l-1}_k}+B.\]
 Thus, by solving this recursive sequence for $A\neq 1$ we get
\begin{equation}
\max_k\norm{f^{l}_k} \leq A^l\norm{P(\psi_0)f} + \frac{A^{l}-1}{A-1}B.
\label{eq:cont_instead}
\end{equation}
or for $A=1$
\[\max_k\norm{f^{l}_k}\leq \norm{P(\psi_0)f} + (l-1)B.\]

Let us now prove (\ref{eq:main_conv0}) and (\ref{eq:main_conv01}), starting with $f\in L^2(\mathcal{M})$ at the input of Layer $0$.
The error in one convolution $g^l_{k'k}$, between the continuous and the discrete signals $j=1,2$, satisfies
\[
\begin{split}
 & \norm{S_{j,l-1}^{\psi_{l-1}}g^l_{k'k}(\mathcal{L})P(\psi_{l-1})f^{l-1}_k-g^l_{k'k}(\DD_{j,l-1}) \tf^{j,l-1}_k} \\ 
 & \leq \norm{S_{j,l-1}^{\psi_{l-1}}g^l_{k'k}(\mathcal{L})P(\psi_{l-1})f^{l-1}_k-g^l_{k'k}(\DD_{j,l-1})S_{j,l-1}^{\psi_{l-1}} P(\psi_{l-1}) f^{l-1}_k} \\ 
 & \ \ \  +\norm{g^l_{k'k}(\DD_{j,l-1})S_{j,l-1}^{\psi_{l-1}} P(\psi_{l-1}) f^{l-1}_k - g^l_{k'k}(\DD_{j,l-1}) \tf^{j,l-1}_k } \\
 & \leq \norm{S_{j,l-1}^{\psi_{l-1}}g^l_{k'k}(\mathcal{L})P(\psi_{l-1}) - g^l_{k'k}(\DD_{j,l-1})S_{j,l-1}^{\psi_{l-1}} P(\psi_{l-1}) }\norm{f^{l-1}_k} \\
 & \ \ \  +\norm{g^l_{k'k}(\DD_{j,l-1})}\norm{S_{j,l-1}^{\psi_{l-1}} P(\psi_{l-1})f^{l-1}_k - \tf^{j,l-1}_k}.  
\end{split}
\]
Thus, by Thm.\ref{eq:TE_worst_G}.(\ref{main00}), and by $\norm{g^l_{k'k}(\DD_{j,l-1})}\leq\norm{g^l_{k'k}}_{\infty}=1$,
\begin{equation}
\begin{split}
 & \norm{S_{j,l-1}^{\psi_{l-1}}g^l_{k'k}(\mathcal{L})P(\psi_{l-1})f^{l-1}_k-g^l_{k'k}(\DD_{j,l-1}) \tf^{j,l-1}_k}  \\
 & \leq  D(\psi_L)\d\norm{f^{l-1}_k}+\norm{S_{j,l-1}^{\psi_{l-1}} P(\psi_{l-1})f^{l-1}_k - \tf^{j,l-1}_k},
\end{split}
\label{eq:dd,m7q}
\end{equation}
where $D(\psi_{L})=D\sqrt{\#\{\l_m\leq \psi_{L}\}_m}$.

Now, the error in the output of the network, before pooling, is
\[
\begin{split}
 &      \norm{S_{j,l-1}^{\psi_l}P(\psi_l)\rho\Big(\sum_{k=1}^{K_{l-1}}a^l_{k'k}\ g^l_{k'k}(\mathcal{L})P(\psi_{l-1})f^{l-1}_k\Big)  -  \rho\Big(\sum_{k=1}^{K_{l-1}}a^l_{k'k} g^l_{k'k}(\DD_{j,l-1})\tf^{j,l-1}_k\Big)}      \\
 &      \leq \norm{\rho\Big(S_{j,l-1}^{\psi_{l-1}}\sum_{k=1}^{K_{l-1}}a^l_{k'k}\ g^l_{k'k}(\mathcal{L})P(\psi_{l-1})f^{l-1}_k\Big) - S_{j,l-1}^{\psi_l}P(\psi_l)\rho\Big(\sum_{k=1}^{K_{l-1}}a^l_{k'k}\ g^l_{k'k}(\mathcal{L})P(\psi_{l-1})f^{l-1}_k\Big)}      \\
 &    \ \ \  +\norm{\rho\Big(S_{j,l-1}^{\psi_{l-1}}\sum_{k=1}^{K_{l-1}}a^l_{k'k}\ g^l_{k'k}(\mathcal{L})P(\psi_{l-1})f^{l-1}_k\Big)  -  \rho\Big(\sum_{k=1}^{K_{l-1}}a^l_{k'k} g^l_{k'k}(\DD_{j,l-1})\tf^{j,l-1}_k\Big)}       \\
 &    \leq  \d \norm{\sum_{k=1}^{K_{l-1}}a^l_{k'k}\ g^l_{k'k}(\mathcal{L})P(\psi_{l-1})f^{l-1}_k}        \\
 &    \ \ \  + \norm{S_{j,l-1}^{\psi_{l-1}}\sum_{k=1}^{K_{l-1}}a^l_{k'k}\ g^l_{k'k}(\mathcal{L})P(\psi_{l-1})f^{l-1}_k  -  \sum_{k=1}^{K_{l-1}}a^l_{k'k}\ g^l_{k'k}(\DD_{j,l-1})\tf^{j,l-1}_k}       \\
 &      \leq \d A^l\norm{P(\psi_0)f} + \d\frac{A^{l}-1}{A-1}B + \sum_{k=1}^{K_{l-1}}\abs{a^l_{k'k}}\norm{ S_{j,l-1}^{\psi_{l-1}}g^l_{k'k}(\mathcal{L})P(\psi_{l-1})f^{l-1}_k -  g^l_{k'k}(\DD_{j,l-1})\tf^{j,l-1}_k}      \end{split}
\]
or for $A=1$
\[
\begin{split}
 &      \norm{S_{j,l-1}^{\psi_l}P(\psi_l)\rho\Big(\sum_{k=1}^{K_{l-1}}a^l_{k'k}\ g^l_{k'k}(\mathcal{L})P(\psi_{l-1})f^{l-1}_k\Big)  -  \rho\Big(\sum_{k=1}^{K_{l-1}}a^l_{k'k} g^l_{k'k}(\DD_{j,l-1})\tf^{j,l-1}_k\Big)}      \\
& \leq \d \norm{P(\psi_0)f} + \d (l-1) B + \sum_{k=1}^{K_{l-1}}\abs{a^l_{k'k}}\norm{ S_{j,l-1}^{\psi_{l-1}}g^l_{k'k}(\mathcal{L})P(\psi_{l-1})f^{l-1}_k -  g^l_{k'k}(\DD_{j,l-1})\tf^{j,l-1}_k}
\end{split}\]
Therefore, by (\ref{eq:dd,m7q})
\[\begin{split}
 &      \norm{S_{j,l-1}^{\psi_l}P(\psi_l)\rho\Big(\sum_{k=1}^{K_{l-1}}a^l_{k'k}\ g^l_{k'k}(\mathcal{L})P(\psi_{l-1})f^{l-1}_k\Big)  -  \rho\Big(\sum_{k=1}^{K_{l-1}}a^l_{k'k} g^l_{k'k}(\DD_{j,l-1})\tf^{j,l-1}_k\Big)}      \\
& \leq   \d A^l\norm{P(\psi_0)f} + \d\frac{A^{l}-1}{A-1}B 
  + A \max_{k}\Big\{D(\psi_{L})\d\norm{f^{l-1}_k}+\norm{S_{j,l-1}^{\psi_{l-1}}P(\psi_{l-1})f^{l-1}_k - \tf^{j,l-1}_k}\Big\} .
\end{split}\]

The error after pooling takes the form
\[\begin{split}
 &      \norm{S_{j,l}^{\psi_l}P(\psi_l)\rho\Big(\sum_{k=1}^{K_{l-1}}a^l_{k'k}\ g^l_{k'k}(\mathcal{L})P(\psi_{l-1})f^{l-1}_k\Big)  -  Y^{j,l}\rho\Big(\sum_{k=1}^{K_{l-1}}a^l_{k'k} g^l_{k'k}(\DD_{j,l-1})\tf^{j,l-1}_k\Big)}      \\
& \leq   \left\| S_{j,l}^{\psi_l}P(\psi_l)\rho\Big(\sum_{k=1}^{K_{l-1}}a^l_{k'k}\ g^l_{k'k}(\mathcal{L})P(\psi_{l-1})f^{l-1}_k\Big)   \right. \\
 & \quad\quad\quad\quad\quad\quad\quad\quad\quad\quad\quad\quad \left. -Y^{j,l}S_{j,l-1}^{\psi_l}P(\psi_l)\rho\Big(\sum_{k=1}^{K_{l-1}}a^l_{k'k}\ g^l_{k'k}(\mathcal{L})P(\psi_{l-1})f^{l-1}_k\Big)\right\| \\
 & + \norm{Y^{j,l}S_{j,l-1}^{\psi_l}P(\psi_l)\rho\Big(\sum_{k=1}^{K_{l-1}}a^l_{k'k}\ g^l_{k'k}(\mathcal{L})P(\psi_{l-1})f^{l-1}_k\Big)  -  Y^{j,l}\rho\Big(\sum_{k=1}^{K_{l-1}}a^l_{k'k} g^l_{k'k}(\DD_{j,l-1})\tf^{j,l-1}_k\Big)}
 \end{split}\]
 \[\begin{split}
 & \leq    \d \norm{\rho\Big(\sum_{k=1}^{K_{l-1}}a^l_{k'k}\ g^l_{k'k}(\mathcal{L})P(\psi_{l-1})f^{l-1}_k\Big)} \\
 & + \norm{S_{j,l-1}^{\psi_l}P(\psi_l)\rho\Big(\sum_{k=1}^{K_{l-1}}a^l_{k'k}\ g^l_{k'k}(\mathcal{L})P(\psi_{l-1})f^{l-1}_k\Big)  -  \rho\Big(\sum_{k=1}^{K_{l-1}}a^l_{k'k} g^l_{k'k}(\DD_{j,l-1})\tf^{j,l-1}_k\Big)} \\
& \leq   2\d A^l\norm{P(\psi_0)f} + 2\d\frac{A^{l}-1}{A-1}B 
  + A \max_{k}\Big\{D(\psi_{L})\d\norm{f^{l-1}_k}+\norm{S_{j,l-1}^{\psi_{l-1}}P(\psi_{l-1})f^{l-1}_k - \tf^{j,l-1}_k}\Big\} .
\end{split}\]
Thus, 
\[\begin{split}
& \norm{S_{j,l}^{\psi_l}P(\psi_{l})f^{l}_{k'} -\tf^{j,l}_{k'}} \\
 & \leq  (D(\psi_{L})+2)\d\Big( A^l\norm{P(\psi_0)f} + \frac{A^{l}-1}{A-1}B\Big) + A \max_{k}\norm{S_{j,l-1}^{\psi_{l-1}}P(\psi_{l-1})f^{l-1}_k - \tf^{j,l-1}_k} .
\end{split}\]
By solving this recurrent sequence, we obtain for $A>1$
\[\norm{S_{j,L}^{\psi_L}\mathcal{N}^L_k P(\psi_{0})f -\mathcal{N}^{j,L}_k S_{j,L}^{\psi_0}P(\psi_0)f}\leq  L(D(\psi_{L})+2)\d\Big( A^{L}\norm{f} + B\frac{A^L-1}{A-1}\Big).\]

For $A=1$ we get
\[\norm{S_{j,L}^{\psi_L}\mathcal{N}^L_k P(\psi_{0})f -\mathcal{N}^{j,L}_k S_{j,0}^{\psi_0}P(\psi_0)f}\leq  L(D(\psi_{L})+2)\d\Big(\norm{f} + LB\Big)\]

Finally,
\[
\begin{split}
&\norm{\mathcal{N}^L_k P(\psi_0)f -R_{j,L}^{\psi_L}\mathcal{N}^{j,L}_k S_{j,1}^{\psi_0}P(\psi_0)f} \\
 & \leq \norm{\mathcal{N}^L_k P(\psi_0)f-R_{j,L}^{\psi_L}S_{j,L}^{\psi_L}\mathcal{N}^L_k P(\psi_0)f}    +\norm{R_{j,L}^{\psi_L}S_{j,L}^{\psi_L}\mathcal{N}^L_k P(\psi_0)f -R_{j,L}^{\psi_L}\mathcal{N}^{j,L}_k S_{j,L}^{\psi_0}P(\psi_0)f}     \\
 &   \leq \norm{P(\psi_L)-R_{j,L}^{\psi_L}S_{j,L}^{\psi_L}P(\psi_L)}\norm{\mathcal{N}^L_k P(\psi_0)f}     +\norm{R_{j,L}^{\psi_L}}\norm{S_{j,l}^{\psi_L}\mathcal{N}^L_k P(\psi_0)f -\mathcal{N}^{j,L}_k S_{j,L}^{\psi_1}P(\psi_0)f}     \\
 &   \leq  L\Big(D\sqrt{\#\{\l_m\leq \psi_L\}_m} +2 \Big)\d\Big(A^L\norm{f} + B\frac{A^L-1}{A-1}\Big)+ \Big(A^L\norm{f} + \frac{A^L-1}{A-1}B\Big)\d. 
\end{split}
\]
This shows that
\[ \begin{split} & \norm{R_{1,L}^{\psi_L}\mathcal{N}^{1,L}_k S_{1,0}^{\psi_0}P(\psi_0)f -R_{2,L}^{\psi_L}\mathcal{N}^{2,L}_k S_{2,L}^{\psi_0}P(\psi_0)f}  \\
 & \leq  \Big(LD\sqrt{\#\{\l_m\leq \psi_L\}_m} +2L+2 \Big)\Big(A^L\norm{f} + B\frac{A^L-1}{A-1}\Big)\d ,
 \end{split}
\]
and similarly for $A=1$.


\subsection{Proof of Proposition \ref{Prop_comm4}}
\label{Ap3}

\begin{lemma}
\label{Lem1}
Consider the setting of Proposition \ref{Prop_comm4}. Then
\begin{equation}
\lim_{n\rightarrow\infty}\sup_{f\neq 0}\frac{\norm{S_n\rho(P(\l)f)-S_n P(\l')\rho(P(\l)f)}  - \norm{\rho(P(\l)f)-P(\l')\rho(P(\l)f)}}{\norm{P(\l)f}}=0
\label{eq:Lem11}
\end{equation}
\begin{equation}
\lim_{\l'\rightarrow\infty }\sup_{f\neq 0} \frac{\norm{\rho(P(\l)f)-P(\l')\rho(P(\l)f)}}{\norm{P(\l)f}}=0
\label{eq:Lem12}
\end{equation}
\end{lemma}

\begin{proof}
We first prove (\ref{eq:Lem11}). Observe that any nonzero vector in $PW(\l)$ can be written as $cf$, where $c>0$ is a real scalar, and $f\in PW(\l)$ has norm $1$. Now, by the positive homogeneity of $\rho$,
\[\frac{\norm{S_n\rho(cP(\l)f)-S_n P(\l')\rho(cP(\l)f)}  - \norm{\rho(cP(\l)f)-P(\l')\rho(cP(\l)f)}}{\norm{cP(\l)f}}\]
\[=\norm{S_n\rho(P(\l)f)-S_n P(\l')\rho(P(\l)f)}  - \norm{\rho(P(\l)f)-P(\l')\rho(P(\l)f)}.\]
Thus, (\ref{eq:Lem11}) is equivalent to
\[
\lim_{n\rightarrow\infty}\sup_{P(\l)f\in\mathcal{S}(\l)}\norm{S_n\rho(P(\l)f)-S_n P(\l')\rho(P(\l)f)}  - \norm{\rho(P(\l)f)-P(\l')\rho(P(\l)f)}=0
\]
where $\mathcal{S}(\l)$ is the unit sphere in $PW(\l)$. Note that the mapping
$F_n: \mathcal{S}(\l) \rightarrow \RR$ defined by
\[
\begin{split}
F_n\big(P(\l)f\big) & =\norm{S_n\rho(P(\l)f)-S_n P(\l')\rho(P(\l)f)}  - \norm{\rho(P(\l)f)-P(\l')\rho(P(\l)f)}\\
 & =\norm{S_n\big(I- P(\l')\big)\rho(P(\l)f)}  - \norm{\big(I-P(\l')\big)\rho(P(\l)f)}
\end{split}
\]
is Lipschitz continuous in $P(\l)f$ for large enough $n$. Indeed, by $\norm{I- P(\l')}=1$ and contraction of $\rho$,
\[
\begin{split}
 \abs{F_n\big(P(\l)f_1\big)-F_n\big(P(\l)f_2\big)}  &       \leq \norm{S_n\big(I- P(\l')\big)\rho(P(\l)f_1)  -  S_n\big(I- P(\l')\big)\rho(P(\l)f_2)}       \\
	 & \ \ \ + \norm{\big(I-P(\l')\big)\rho(P(\l)f_1)   -   \big(I-P(\l')\big)\rho(P(\l)f_2)}             \\
	 &     \leq (C+1)\norm{P(\l)f_1 - P(\l)f_2},         
\end{split}
\]
where $C$ is the bound of $\norm{S_n^{\l}}$, guaranteed by Proposition \ref{P:quad1}, and can be chosen $C=2$ for large enough $n$. Note that the Lipschitz constants of $F_n$ are uniformly bounded by $D=3$.

Observe that by Definition \ref{quadratureWRTrho}, $F_n$ converges to 0 pointwise as $n\rightarrow\infty$. Our goal is to show uniform convergence.
Since the domain $\mathcal{S}(\l)$  of $F_n$ is compact, $F_n$ obtains a maximum for each $n$. Denote 
\[P(\l)f_n = \argmax_{P(\l)f\in \mathcal{S}(\l)} F_n(P(\l)f).\]
Suppose that $\lim_{n\rightarrow\infty}F_n(P(\l)f_n)$ does not exist, or converges to a nonzero limit.
Since $\mathcal{S}(\l)$ is compact, and $F_n$ uniformly bounded by $2D$, there is a subsequence $P(\l)f_{n_m}$ converging to some $P(\l)f_{\infty}\in \mathcal{S}(\l)$, such that 
\[\lim_{m\rightarrow\infty}F_{n_m}(P(\l)f_{n_m}) = A >0.\] 
Now, for every $\e>0$ there is a large enough $M$, such that, for every $m>M$, 
\[
\begin{split}
  \abs{F_{n_m}(P(\l)f_{\infty}) - A }  &     \leq \abs{F_{n_m}(P(\l)f_{\infty}) - F_{n_m}(P(\l)f_{n_m})} + \e/2         \\
	 &     \leq  D \norm{P(\l)f_{\infty}-P(\l)f_{n_m}} + \e/2<\e.         
\end{split}
\]
By picking $\e=A/3$, this contradicts the fact that $\lim_{n\rightarrow\infty}F_n(P(\l)f_{\infty})=0$, guaranteed by Definition \ref{quadratureWRTrho}.

Similarly, for (\ref{eq:Lem12}), 
\[   \sup_{f\neq 0}\frac{\norm{\rho(P(\l)f)-P(\l')\rho(P(\l)f)}}{\norm{P(\l)f}}        = \sup_{P(\l)f\in\mathcal{S}(\l)}\norm{\big(I-P(\l')\big)\rho(P(\l)f)}.      \]
For a fixed $f$, the fact that $\big(I-P(\l')\big)\rho(P(\l)f)$ is the tail in the expansion of $\rho(P(\l)f)$ in the eigenbasis of $\cL$, we have
\begin{equation}
\lim_{\l'\rightarrow\infty}\norm{\big(I-P(\l')\big)\rho(P(\l)f)}=0 \quad  {\rm for\ all \ }P(\l)f\in\mathcal{S}(\l).
\label{eq:pont21}
\end{equation}
The uniform convergence of (\ref{eq:Lem12}) is derived from the pointwise convergence of (\ref{eq:pont21}) in the same procedure as above.

\end{proof}

\begin{proof}[Proof of Proposition \ref{Prop_comm4}]
By Lemma \ref{Lem1}
\[
\begin{split}
   &     \lim_{\l'\rightarrow\infty }\lim_{n\rightarrow\infty}\sup_{f\neq 0}\frac{\norm{S_n\rho(P(\l)f)-S_n P(\l')\rho(P(\l)f)} }{\norm{P(\l)f}}         \\
	 &     \leq \lim_{\l'\rightarrow\infty }\lim_{n\rightarrow\infty}\sup_{f\neq 0}\frac{\norm{S_n\rho(P(\l)f)-S_n P(\l')\rho(P(\l)f)}  - \norm{\rho(P(\l)f)-P(\l')\rho(P(\l)f)}}{\norm{P(\l)f}}         \\
	 &       \ \ \ + \lim_{\l'\rightarrow\infty }\sup_{f\neq 0} \frac{\norm{\rho(P(\l)f)-P(\l')\rho(P(\l)f)}}{\norm{P(\l)f}}  =0.      
\end{split}
\]
Now, the limit as $\l\rightarrow\infty$ follows trivially.
\end{proof}

\subsection{Proof of Theorem \ref{Theo:probMC}}
\label{Proof of Theorem Theo:probMC}

We prove Theorem \ref{Theo:probMC} using three lemmas.

\begin{lemma}
\label{T:quad1}
Let $f\in PW(\l)$.
Let $\{\mathcal{M},\mu\}$ be a compact topological-measure space with $\mu(\mathcal{M})=1$. Consider the weighted measure $\mu_w$ satisfying (\ref{eq:mu_w}). Let $\cL$ be a Laplacian of the form (\ref{eq:IntOp2}), such that $H\in L^2(\mathcal{M}^2; \mu\times\mu)$.
Suppose that $\cL$ respects continuity. Let $\DD_n$ be a random sampled Laplacian.
Let
\begin{equation}
C= \frac{1}{w_{\min}}\norm{H}_{L^2(\mathcal{M}^2; \mu\times\mu)}  C_{\l}
\label{eq:C01}
\end{equation}
for $w_{\min}=\min_{x\in\mathcal{M}}w(x)$, and $C_{\l}$ is the constant such that
\begin{equation}
    \label{eq_Cl}
    \forall g\in PW(\l). \quad \norm{g}_{\infty}  \leq C_{\l}\norm{g}_{2},
\end{equation}
guaranteed by the fact that $PW(\l)$ is finite-dimensional.

Then for every $\d>0$, in probability more than $(1-\d)$,
\begin{equation}
\norm{S_n^{\l}\cL P(\l) - \DD_n S_n^{\l}P(\l)}_{L^2(G_n)} \leq C \d^{-1/2} N_n^{-1/2} .
\label{eq:dLerror}
\end{equation}
where the induced norm is for operators $L^2(\mathcal{M};\mu)\rightarrow L^2(G_n)$.

\end{lemma}

\begin{proof}
Let $f\in PW(\l)$, and note that $f$ is continuous since $\cL$ respects continuity.
For a fixed $x_0\in\mathcal{M}$, consider the random variable $F_{x_0}:\{\mathcal{M};\mu_w\}\rightarrow \CC$ defined by
\begin{equation}
F_{x_0}(x)= \frac{1}{w(x)}H(x_0,x)f(x).
\label{eq:Fx0}
\end{equation}
By (\ref{eq:IntOp2}) and (\ref{eq:mu_w}), the expected value of $F_{x_0}$ is 
\begin{equation}
{\rm E}(F_{x_0})=\cL f (x_0).
\label{eq:exp11}
\end{equation}

Consider $N_n$ i.i.d random variables (\ref{eq:Fx0}), denoted by 
\[F_{x_0;k'} = \frac{1}{w(x^n_{k'})}H(x_0,x^n_{k'})f(x^n_{k'})  , \quad {k'}=1,\ldots,N_n.\]
Let
\begin{equation}
F^{N_n}_{x_0} = \frac{1}{N_n}\sum_{{k'}=1}^{N_n}F_{x_0;k'}.
\label{eq:Fn12}
\end{equation}
By (\ref{eq:exp11})  we have
\[{\rm E}\Big(F^{N_n}_{x_0}\Big)=\cL f (x_0)\]
On the other hand, the realization of the sum in (\ref{eq:Fn12}) can be written for $x_0=x^n_k$ as
\begin{equation}
F^{N_n}_{x^n_k}  = \sum_{k'=1}^{N_m}\frac{1}{w(x^n_{k'})}H(x^n_k,x^n_{k'})f(x^n_{k'})dx = [\DD_n S_n^{\l}f]_ k.
\label{eq:FisDn}
\end{equation}
This shows that the graph Laplacians coincide on average with the topological-measure Laplacian. 

Next let us analyze the average mean square error over $x_0\in\mathcal{M}$. In the following, Fubini's theorem follows the fact that $\mathcal{M}$ is compact and all integrands are continuous. Hence,
\[
\begin{split}
   &      {\rm E}\norm{F^{N_n}_{(\cdot)} - \cL f}^2_{L^2(\mathcal{M})}        \\
	 &     = \iint_{x_1,\ldots,x_n} \int_{x_0} \abs{F^{N_n}_{x_0}(x^n_1,\ldots,x^n_{N_n}) - [\cL f](x_0)}^2 dx_0 \ {w(x^n_1)}dx^n_1 \cdot {w(x^n_{N_n})}dx^n_{N_n}         \\
	 &       = \int_{x_0}\iint_{x_1,\ldots,x_n}  \abs{F^{N_n}_{x_0}(x^n_1,\ldots,x^n_{N_n}) - [\cL f](x_0)}^2  {w(x^n_1)}dx^n_1 \cdot {w(x^n_{N_n})}dx^n_{N_n}\ dx_0        \\
	 &         = \int_{x_0} {\rm Var}F^{N_n}_{x_0}  dx_0      = \int_{x_0} \frac{{\rm Var}F_{x_0}}{N_n}  dx_0 = \frac{\norm{{\rm Var}F_{(\cdot)}}_1}{N_n} 
\end{split}
\]

Next, we prove that prove ${\rm Var}F_{(\cdot)}\in L^1(\mathcal{M})$, and bound $\norm{{\rm Var}F_{(\cdot)}}_1$.  We have
\[{\rm Var}F_{x_0} \leq \int_{x}  \abs{F_{x_0}(x)}^2   {w(x)}dx. \]
This yields
\[
\begin{split}
  \norm{{\rm Var}F_{(\cdot)}}_1  &     \leq \int_{x_0}\int_{x}  \abs{F_{x_0}(x)}^2   {w(x)}dx dx_0          \\
	 &     =\int_{x_0}\int_{x} \frac{1}{w(x)}\abs{H(x_0,x) }^2\abs{f(x)}^2 dx dx_0.                    
\end{split}
\]
Thus
\[
\begin{split}
  \norm{{\rm Var}F_{(\cdot)}}_1  &      \leq \norm{\frac{1}{\sqrt{w(\cdot)}}H(\cdot,\cdot\cdot) }_{L^2(\mathcal{M}^2)}^2  \norm{f}_{\infty}^2         \\
	 &     \leq \frac{1}{w_{\min}}\norm{H}_{L^2(\mathcal{M}^2)}^2 \norm{f}_{\infty}^2                   
\end{split}
\]
This proves that the expected mean square error satisfies
\begin{equation}
{\rm E}\norm{F^{N_n}_{(\cdot)} - \cL f}^2_{L^2(\mathcal{M})} \leq  \frac{1}{w_{\min}}\norm{H}_{L^2(\mathcal{M}^2)}^2 \norm{f}_{\infty}^2 \frac{1}{N_n}.
\label{eq:rrrgrggg}
\end{equation}

To obtain a convergence result in high probability, we can use theorems on concentration of measure, like Markov's, Chebyshev's or Bernstein's inequalities. For Lemma \ref{T:quad1}, we consider Markov's inequality, that states that for a random variable $X$ with finite non-zero expected value
\[{\rm Pr}\Big( X \geq \frac{{\rm E}(X)}{\d} \Big) \leq \d\]
for any $0<\d<1$.
In our case,  by (\ref{eq:rrrgrggg}), Markov's inequality states that in probability more than $(1-\d)$
\begin{equation}
\norm{F^{N_n}_{(\cdot)} - \cL f}_{L^2(\mathcal{M})} \leq \frac{1}{\sqrt{w_{\min}}}\norm{H}_{L^2(\mathcal{M}^2)} \norm{f}_{L^{\infty}(\mathcal{M})} \frac{1}{\sqrt{N_n}}\frac{1}{\sqrt{\d}}.
\label{eq:inf_err}
\end{equation}
This means that for every $k$, 
\begin{equation}
\abs{F^{N_n}_{x^n_k} - \cL f (x^n_k)} \leq C_{\l}\frac{1}{\sqrt{w_{\min}}}\norm{H}_{L^2(\mathcal{M}^2)} \norm{f}_{L^{\infty}(\mathcal{M})} \frac{1}{\sqrt{N_n}}\frac{1}{\sqrt{\d}}.
\label{eq:inf_errk}
\end{equation}
We finally conclude that, by the inner product structure (\ref{eq:Bdiag}) of $L^2(G_n)$,  and by (\ref{eq:FisDn})
\[\norm{\DD_n S_n^{\l}f - S_n^{\l}\cL f}_{L^2(G_n)} = \sqrt{\frac{1}{N_n}\sum_{k=1}^{N_n} \frac{1}{w(x^n_k)}\abs{F^{N_n}_{x^n_k} - \cL f (x^n_k)}^2} \leq  C N_n^{-1/2}\d^{-1/2} \norm{f}_{L^{2}(\mathcal{M})}\]
where $C$ is given in (\ref{eq:C01}).

\end{proof}

Denote by $\norm{A}_{F(\CC^{M\times M})}$ the Frobenius norm of the matrix $A\in \CC^{M\times M}$.
\begin{lemma}
\label{T:quad2}
Let $\{\mathcal{M},\mu\}$ be a compact topological-measure space with $\mu(\mathcal{M})=1$. Let $\mu_w$ be a weighted measure  satisfying (\ref{eq:mu_w}). Let $\cL$ be a Laplacian  of the form (\ref{eq:IntOp2}), such that $H\in L^2(\mathcal{M}^2)$.
Suppose that $\cL$ respects continuity. Let $\mathbf{S}_n^{\l}$ and $\mathbf{R}_n^{\l}$ be random sampling and interpolation operators. Consider the corresponding random variable $\ip{\boldsymbol{\Phi}_n}{\boldsymbol{\Phi}_n}$ given in Definition \ref{quadrature_sequence2} on the random sample points.
Then for every $\d>0$, in probability more than $(1-\d)$
\begin{equation}
\norm{\ip{\boldsymbol{\Phi}_n}{\boldsymbol{\Phi}_n} - \mathbf{I}}_{F(\CC^{M_{\l}\times M_{\l}})} \leq C \d^{-1/2} N_n^{-1/2} .
\label{eq:dLerror2}
\end{equation}
Here,
\[C = \frac{M_{\l}}{\sqrt{w_{\min}}}\max_{m\leq M_{\l}}\norm{\phi_m}_{\infty}^2, \]
and $M_{\l} = {\rm dim}(PW(\l))$ as before.
\end{lemma}

\begin{proof}
For fixed $m,m'\in\mathcal{M}$, consider the random variable $F_{m,m'}:\{\mathcal{M};\mu_w\}\rightarrow \CC$ defined by
\begin{equation}
F_{m,m'}(x)= \frac{1}{w(x)}\phi_m(x)\overline{\phi_{m'}(x)}.
\label{eq:Fx01}
\end{equation}
By (\ref{eq:Fx01}) and (\ref{eq:mu_w}), the expected value of $F_{x_0}$ is 
\begin{equation}
{\rm E}(F_{x_0})=\ip{\phi_m}{\phi_{m'}} = \d_{m,m'},
\label{eq:exp111}
\end{equation}
where the Kronecker delta $\d_{m,m'}$ is $1$ if $m=m'$ and $0$ otherwise.

Consider $N_n$ i.i.d random variables (\ref{eq:Fx01}), denoted by 
\[F_{m,m';k'} = \frac{1}{w(x^n_{k'})}\phi_m(x^n_{k'})\overline{\phi_{m'}(x^n_{k'})}  , \quad {k'}=1,\ldots,N_n.\]
Let
\begin{equation}
F^{N_n}_{m,m'} = \frac{1}{N_n}\sum_{{k'}=1}^{N_n}F_{m,m';k'}.
\label{eq:Fn121}
\end{equation}
By (\ref{eq:exp111})  we have
\[{\rm E}\Big(F^{N_n}_{m,m'}\Big)=\ip{\phi_m}{\phi_{m'}}.\]
On the other hand, the realization of the sum in (\ref{eq:Fn121}) can be written  as
\begin{equation}
F^{N_n}_{m,m'}  = [\ip{\mathbf{\Phi_n}}{\mathbf{\Phi_n}}]_{m,m'}.
\label{eq:FisDn2}
\end{equation}
This shows that $\ip{\mathbf{\Phi_n}}{\mathbf{\Phi_n}}$ coincide on average with $\mathbf{I}$. 

Next let us analyze the average mean square error over $m,m'\in\mathcal{M}$. For a matrix $\mathbf{A} = (a_{m,m'})_{m,m'}$, denote 
\[\norm{\mathbf{A}}_{\rm F} = \sqrt{\sum_{m,m'}\abs{a_{m,m'}}^2} \ , \quad \norm{\mathbf{A}}_{{\rm F},1} = \sum_{m,m'}\abs{a_{m,m'}}.\]
We have
\[
\begin{split}
   &      {\rm E}\norm{\ip{\mathbf{\Phi_n}}{\mathbf{\Phi_n}}-\mathbf{I}}^2_{{\rm F}}        \\
	 &       = \iint_{x_1,\ldots,x_n} \sum_{m,m'} \abs{F^{N_n}_{m,m'}(x^n_1,\ldots,x^n_{N_n}) - \d_{m,m'}}^2 {w(x^n_1)}dx^n_1 \cdot {w(x^n_{N_n})}dx^n_{N_n}        \\
	 &       = \sum_{m,m'}\iint_{x_1,\ldots,x_n}  \abs{F^{N_n}_{m,m'}(x^n_1,\ldots,x^n_{N_n}) - \d_{m,m'}}^2  {w(x^n_1)}dx^n_1 \cdot {w(x^n_{N_n})}dx^n_{N_n}       \\
	 &        = \sum_{m,m'} {\rm Var}F^{N_n}_{m,m'}  = \sum_{m,m'} \frac{{\rm Var}F_{m,m'}}{N_n}  = \frac{\norm{{\rm Var}F_{(\cdot)}}_{{\rm F},1}}{N_n}     
\end{split}
\]
Next, we bound $\norm{{\rm Var}F_{(\cdot)}}_{{\rm F},1}$.  We have
\[{\rm Var}F_{m,m'} \leq \int_{x}  \frac{1}{w(x)}\abs{\phi_m(x)\phi_{m'}(x)}^2   {w(x)}dx, \]
so
\[\norm{{\rm Var}F_{(\cdot)}}_{{\rm F},1} \leq \frac{M_{\l}^2}{w_{\min}}\max_m\norm{\phi_m}_{\infty}^4  \]
This proves that the expected mean square error satisfies
\begin{equation}
{\rm E}\norm{\ip{\mathbf{\Phi_n}}{\mathbf{\Phi_n}}-\mathbf{I}}^2_{{\rm F}}\leq \frac{M_{\l}^2}{w_{\min}}\max_{m\leq M_{\l}}\norm{\phi_m}_{\infty}^4  \frac{1}{N_n}.
\label{eq:rrrgrggg2}
\end{equation}

Finally, by Markov's inequality,
in probability more than $(1-\d)$
\begin{equation}
\norm{\ip{\mathbf{\Phi_n}}{\mathbf{\Phi_n}}-\mathbf{I}}_{{\rm F}} \leq \frac{M_{\l}}{\sqrt{w_{\min}}}\max_{m\leq M_{\l}}\norm{\phi_m}_{\infty}^2 \frac{1}{\sqrt{N_n}}\frac{1}{\sqrt{\d}}.
\label{eq:inf_err2}
\end{equation}

\end{proof}

Before formulating the last Monte-Carlo lemma, we require two more lemmas.

\begin{lemma}
\label{Lem:5}
Let $\mathcal{S}(\l)$ be the unit $L^2(\mathcal{M})$ sphere  in $PW(\l)$, and let $\rho$ be a contractive positively homogeneous of order 1 activation function that preserves spectral decay. Then
\[\mathcal{S}(\l) \ni f \mapsto \big(I-P(\l')\big)\rho(f) \]
is continuous as a mapping $\mathcal{S}(\l)\rightarrow L^{\infty}(\mathcal{M})$.
\end{lemma}

\begin{proof}

Let $f,g\in PW(\l)$.
Consider the following calculation for any $M_2>M_1>M_{\l'}$.
\begin{equation}
\begin{split}
   &      \norm{\sum_{m=M_1}^{M_2} \ip{\rho(f) - \rho(g)}{\phi_m} \phi_m}_{\infty}         \\
	 &       \leq \sum_{m=M_1}^{M_2} \abs{\ip{\rho(f) - \rho(g)}{\phi_m} } \norm{\phi_m}_{\infty}       \\
	 &        =\sum_{m=M_1}^{M_2} \norm{\phi_m}_{\infty}\abs{\ip{\rho(f)}{\phi_m} - \ip{\rho(g)}{\phi_m}}      \\
   &        =\sum_{m=M_1}^{M_2} m^{-1/2-\k/2}\norm{\phi_m}_{\infty}\abs{m^{1/2+\k/2}\ip{\rho(f)}{\phi_m} - m^{1/2+\k/2}\ip{\rho(g)}{\phi_m}}      \\
	 &        \leq R \sqrt{\sum_{m=M_1}^{\infty} \norm{\phi_m}_{\infty}^2 m^{1+\k}\abs{\ip{\rho(f)}{\phi_m} - \ip{\rho(g)}{\phi_m}}^2},      
\end{split}
\label{eq:tempf7}
\end{equation}
where 
\[R=\sqrt{\sum_{m=1}^{\infty} m^{-1-2\k}}.\]
By (\ref{eq:PSD2}), 
\[\lim_{M_1\rightarrow\infty}\sqrt{\sum_{m=M_1}^{\infty} \norm{\phi_m}_{\infty}^2 m^{1+\k}\abs{\ip{\rho(f)}{\phi_m} - \ip{\rho(g)}{\phi_m}}^2}=0.\]
Therefore
\begin{equation}
\Big\{\sum_{m=M_{\l'}}^{M} \ip{\rho(f) - \rho(g)}{\phi_m} \phi_m\Big\}_{M=M_{\l'}}^{\infty}
\label{eq:Cauch}
\end{equation}
is a Cauchy sequence in $L^{\infty}(\mathcal{M})$, and thus converges in $L^{\infty}(\mathcal{M})$ to a limit we denote by 
\begin{equation}
\sum_{m=M_{\l'}}^{\infty} \ip{\rho(f) - \rho(g)}{\phi_m} \phi_m.
\label{eq:ser4}
\end{equation}
The series (\ref{eq:Cauch}) also converges in $L^2(\mathcal{M})$, to $(I-P(\l'))(\rho(f) - \rho(g))$. Since convergence in $L^2(\mathcal{M})$ implies pointwise convergence of a subsequence almost everywhere, we must have
\[\sum_{m=M_{\l'}}^{\infty} \ip{\rho(f) - \rho(g)}{\phi_m} \phi_m=(I-P(\l'))(\rho(f) - \rho(g)),\]
with convergence in $L^{\infty}(\mathcal{M})$.
By  conservation of bounds under limits, and by (\ref{eq:tempf7}), we now have
\begin{equation}
\begin{split}
   &      \norm{(I-P(\l'))\rho(f) - P(\l_{M})(I-P(\l'))\rho(g)}_{\infty}        \\
	 &       =\norm{(I-P(\l'))(\rho(f) - \rho(g))}_{\infty}       \\
	 &      \leq R \sqrt{\sum_{m=M_{\l'}}^{\infty} m^{1+\k}\norm{\phi_m}_{\infty}^2\abs{\ip{\rho(f)-\rho(g)}{\phi_m}}^2}.        
\end{split}
\label{eq:PSD3}
\end{equation}
Last, the continuity of $(I-P(\l'))\rho(f)$ as a mapping $\mathcal{S}(\l)\rightarrow L^{\infty}(\mathcal{M})$ follows from (\ref{eq:PSD3}) and (\ref{eq:PSD2}).
\end{proof}

By Lemma \ref{Lem:5},
$\norm{\big(I-P(\l')\big)\rho(f)}_{\infty}$
has a maximal value in the compact domain $\mathcal{S}(\l)$ that we denote by $C_{\l'}$.
For the next proposition we also need the following simple observation.
\begin{lemma}
\label{Lem_simpe1}
Let $A,B\geq 0$ such that $\abs{A^2-B^2} <\k$. Then $\abs{A-B}<\sqrt{\k}$.
\end{lemma}

\begin{proof}
The equation $\abs{A^2-B^2} <\k$ is equivalent to
\[B^2 - \k <A^2 < B^2 + \k\]
or
\begin{equation}
\sqrt{B^2 - \k} <A < \sqrt{B^2 + \k}.
\label{eq:temp45}
\end{equation}
As a result
\[\sqrt{B^2} - \sqrt{\k} <A < \sqrt{B^2} + \sqrt{\k} \]
or equivalently
\[\abs{A-B}<\sqrt{\k}.\]
\end{proof}

\begin{lemma}
\label{T:quad3}
Let $\{\mathcal{M},\mu\}$ be a compact topological-measure space with $\mu(\mathcal{M})=1$. Consider the weighted measure $\mu_w$ satisfying (\ref{eq:mu_w}), and a random sample set $\{x^n_k\}_{n=1}^{N_n}$ from $\{\mathcal{M},\mu_w\}$. Consider a Laplacian $\cL$ with eigenbasis $\{\phi_m\}$ as before. Suppose that the activation function $\rho$ is contractive,  positively homogeneous of order 1, and preserves spectral decay.
Suppose that $\cL$ respects continuity.  
Then for every $\d>0$, in probability more than $(1-\d)$
\begin{equation}
\max_{f\in PW(\l)}
\frac{\norm{S_n\rho(f) - S_n^{\l}P(\l')\rho(f)}_{L^2(G_n)} -\norm{\rho(f) - P(\l')\rho(f)}_{L^2(G_n)} }{\norm{P(\l)f}_{L^2(\mathcal{M} ;\mu)}}\leq \frac{1}{w_{\min}^{1/4}}C_{\l'} \frac{1}{N_n^{1/4}}\frac{1}{\d^{1/4}},
\label{eq:dLerror3}
\end{equation}
where 
\[C_{\l'}=\max_{f\in \mathcal{S}(\l)}\norm{\big(I-P(\l')\big)\rho(f)}_{\infty}\]
and $\mathcal{S}(\l)$ is the unit sphere in $PW(\l)$.
\end{lemma}

\begin{proof}

First, since $\rho$ is positively homogeneous of order 1,
the maximum in (\ref{eq:dLerror3}) is equal to
\begin{equation}
\max_{f\in \mathcal{S}(\l)}
\norm{S_n\rho(f) - S_n^{\l}P(\l')\rho(f)}_{L^2(V^2)} -\norm{\rho(f) - P(\l')\rho(f)}_{L^2(V^2)} .
\label{eq:dLerror0}
\end{equation}

Consider the random variable $F:\{\mathcal{M};\mu_w\}\rightarrow \CC$ defined by
\begin{equation}
F(x)= \frac{1}{w(x)}\abs{\big(  \rho(f(x)) - P(\l')\rho(f(x)) \big)}^2.
\label{eq:Fx03}
\end{equation}
By (\ref{eq:Fx03}) and (\ref{eq:mu_w}), the expected value of $F$ is 
\begin{equation}
{\rm E}(F)=\norm{\rho(f)- P(\l')\rho(f)}_2^2.
\label{eq:exp113}
\end{equation}

Consider $N_n$ i.i.d random variables (\ref{eq:Fx03}), denoted by 
\[F_{k'} = \frac{1}{w(x^n_{k'})}\abs{\big(  \rho(f(x_{k'}^n)) - P(\l')\rho(f(x_{k'}^n)) \big)}^2  , \quad {k'}=1,\ldots,N_n.\]
Let
\begin{equation}
F^{N_n} = \frac{1}{N_n}\sum_{{k'}=1}^{N_n}F_{k'}.
\label{eq:Fn123}
\end{equation}
By (\ref{eq:exp113})  we have
\[{\rm E}\Big(F^{N_n}\Big)=\norm{\rho(f)- P(\l')\rho(f)}_2^2.\]

On the other hand, the realization of the sum in (\ref{eq:Fn123}) can be written as
\begin{equation}
F^{N_n}  = \norm{S_n\rho(P(\l)f) - S_n P(\l')\rho(P(\l)f)}_{L^2(V^2)}^2.
\label{eq:FisDn3}
\end{equation}
This shows that on average (\ref{eq:dLerror0}) is zero.

Next let us analyze the expected error of (\ref{eq:dLerror0}).
\[
\begin{split}
   &       {\rm E}\abs{F^{N_n} - \norm{\rho(f)- P(\l')\rho(f)}_2^2}^2       \\
	 &         = \iint_{x_1,\ldots,x_n}  \abs{F^{N_n}(x^n_1,\ldots,x^n_{N_n}) - \norm{\rho(f)- P(\l')\rho(f)}_2^2}^2 {w(x^n_1)}dx^n_1 \cdot {w(x^n_{N_n})}dx^n_{N_n}     \\
	 &     = {\rm Var}F^{N_n}    =  \frac{{\rm Var}F}{N_n} .      
\end{split}
\]
 We have
\begin{equation}
\begin{split}
 {\rm Var}F   &     \leq \int_{x}  \abs{F(x)}^2   {w(x)}dx         \\
	 &      =\int_x\frac{1}{w(x)}\abs{\big(  \rho(f(x)) - P(\l')\rho(f(x)) \big)}^4 dx        \\
	 &     \leq \frac{1}{w_{\min}}\norm{\big(  I - P(\l') \big)\rho(f(x))}_4^4         \\
	 &     \leq \frac{1}{w_{\min}}\norm{\big(  I - P(\l') \big)\rho(f(x))}_{\infty}^4 \leq \frac{1}{w_{\min}} C_{\l'}^4.         
\end{split}
\label{eq:VarF3}
\end{equation}

By (\ref{eq:VarF3}), Markov's inequality states that in probability more than $(1-\d)$
\begin{equation}
\abs{F^{N_n} - \norm{\rho(f)- P(\l')\rho(f)}_2^2} \leq \frac{1}{\sqrt{w_{\min}}}C_{\l'}^2 \frac{1}{\sqrt{N_n}}\frac{1}{\sqrt{\d}}.
\label{eq:inf_err3}
\end{equation}
This shows, by Lemma \ref{Lem_simpe1} and (\ref{eq:FisDn3}), that
\[      \max_{f\in PW(\l)}
\frac{\norm{S_n\rho(f) - S_n^{\l}P(\l')\rho(f)}_{L^2(V^2)} -\norm{\rho(f) - P(\l')\rho(f)}_{L^2(V^2)} }{\norm{f}_2}       
	    \leq 
\frac{1}{w_{\min}^{1/4}}C_{\l'} \frac{1}{N_n^{1/4}}\frac{1}{\d^{1/4}}.        
\]

\end{proof}

\begin{proof}[Proof of Theorem \ref{Theo:probMC}]
We apply Lemmas \ref{T:quad1},\ref{T:quad2} and \ref{T:quad3} with failure probability $\d/3$. Then, with probability more than $(1-\d)$ the bounds (\ref{eq:dLerror}), (\ref{eq:dLerror2}) and (\ref{eq:dLerror3}) are satisfied simultaneously. We thus consider the subsequence $n_m$ that contains any $n$ independently in probability more than $(1-\d)$,  for which the bounds (\ref{eq:dLerror}), (\ref{eq:dLerror2}) and (\ref{eq:dLerror3}) are deterministic. Note that the sequence $n_m$ is infinite in probability 1.

Denote $\overline{M}_n=M_{\overline{\l}_n}$. By assumption $\norm{\boldsymbol{\l}^{M_{\overline{\l}_n}}}_1=o(N_n^{1/2})$, where  $\norm{\boldsymbol{\l}^{\overline{M}_n}}_1$ is defined in (\ref{eq:norm_l}). Let us analyze the dependency of the bounds (\ref{eq:dLerror}), (\ref{eq:dLerror2}) and (\ref{eq:dLerror3}) on $\overline{M}_n$ and $N_n$.
Note that the dependency of (\ref{eq:dLerror}), (\ref{eq:dLerror2}) and (\ref{eq:dLerror3}) on $\l$ does not affect the validity of Definitions \ref{As_convS} and \ref{def:pointwise convergent}, and \ref{quadrature_sequence2}. The asymptotic analysis in $\overline{M}_n$ and $N_n$ in these definitions is for fixed $\l$.

The bound (\ref{eq:dLerror}) depends on $\overline{M}_n$ as follows:
\[
\begin{split}
 \norm{H}_2^2  &       = \int_x\int_{x_0} \abs{\sum_{m=1}^{\overline{M}_n}\phi_m(x_0)\l_m\phi_m(x) }^2 dx_0 dx        \\
	 &       \leq  \Big(\sum_{m=1}^{\overline{M}_n}\abs{\l_m}\sqrt{\int_x\int_{x_0}\abs{\phi_m(x_0)\phi_m(x)}^2 dx_0 dx} \Big)^2       \\
	 &      = \Big(\sum_{m=1}^{\overline{M}_n}\abs{\l_m}\sqrt{\int_{x_0}\abs{\phi_m(x_0)}^2dx_0} \sqrt{\int_x\abs{\phi_m(x)}^2 dx  } \Big)^2        \\
	 &      = \Big(\sum_{m=1}^{\overline{M}_n}\abs{\l_m} \Big)^2   =\norm{\boldsymbol{\l}^{\overline{M}_n}}_1^2.       
\end{split}
\]
Thus, since the bound (\ref{eq:dLerror}) also depend multiplicatively on $N_n^{-1/2}$, any choice of $M_{n}$ such that  $\norm{\boldsymbol{\l}^{M_{\overline{\l}_n}}}_1=o(N_n^{1/2})$  makes the bound converge to zero as $n\rightarrow\infty$, and guarantees Definition \ref{As_convS}. 

Note that the bounds (\ref{eq:dLerror2}) and (\ref{eq:dLerror3}) do not depend on $M_{n}$.
The bound (\ref{eq:dLerror2}) proves that Definition \ref{quadrature_sequence2} is satisfied for the subsequence $n_m$, which proves Definitions \ref{As_RS} and \ref{As_RS2}.
For the relation between the bound (\ref{eq:dLerror3}) and Definition \ref{def:pointwise convergent},  we use Lemma \ref{Lem1}, where (\ref{eq:Lem11}) converges to zero in the subsequence $n_m$, and (\ref{eq:Lem12}) converges to zero deterministically. This proves Definition \ref{def:pointwise convergent} for the subsequence $n_m$.

\end{proof}

\subsection{\textcolor{black}{Proof of Proposition \ref{Prop:MC_rate}}}

By the proof of Proposition \ref{Prop10},  
	\begin{equation}
	\label{eq:temp3NE4}
\begin{split}
\norm{g(\mathcal{L})P(\lambda) - R_n^{\lambda}g(\boldsymbol{\Delta}_n) S_n^{\lambda} P(\lambda)} 
 \leq  & DC \sqrt{\#\{\l_j\leq\l\}_j} \norm{S_n^{\lambda_M}\mathcal{L} P(\lambda_M) -\boldsymbol{\Delta} S_n^{\lambda_M} P(\lambda_M)} \\
 &  + \norm{g}_{\mathcal{L},M}\norm{P(\lambda_M) - R_n^{\lambda_M} S_n^{\lambda_M}P(\lambda_M)}.
\end{split}
\end{equation}
By the proof of Theorem \ref{Theo:probMC}, 
\[\norm{H_{\l_n}}_2 \leq \norm{\boldsymbol{\l}^{\overline{M}_n}}_1 \leq B N_n^{1/2-\alpha}.\]
By Lemma \ref{T:quad1}, in probability more than $(1-\d)$,
\[\norm{S_n^{\l}\cL P(\l) - \DD_n S_n^{\l}P(\l)}_{L^2(G_n)} \leq 
\frac{1}{w_{\min}}\norm{H_{\l_n}}_{L^2(\mathcal{M}^2; \mu\times\mu)}  C_{\l}
\d^{-1/2} N_n^{-1/2}. \]
where $C_{\l}$ is defined by (\ref{eq_Cl})
\[\forall g\in PW(\l). \quad \norm{g}_{\infty}  \leq C_{\l}\norm{g}_{2}.\]
So
\[\norm{S_n^{\l}\cL P(\l) - \DD_n S_n^{\l}P(\l)}_{L^2(G_n)} \leq 
\frac{1}{w_{\min}}B N_n^{1/2-\alpha}  C_{\l}
\d^{-1/2} N_n^{-1/2} \]
\begin{equation}
\label{eq:temp456bd}
    =\frac{B}{w_{\min}} C_{\l}
\d^{-1/2}  N_n^{-\alpha}.
\end{equation}

We can bound $C_{\l}$ as follows. Let $g=\sum_{n=1}^{M_{\l}}c_n\phi_n$.
For every $x$ in $\mathcal{M}$,
\[
\begin{split}
  \norm{g}_2 & =   \sqrt{\sum_{n=1}^{M_{\l}} \abs{c_n}^2}
\geq M_{\l}^{-1/2}\sum_{n=1}^{M_{\l}} \abs{c_n} \\
 &  \geq \frac{1}{\max_{x,n\in\{1,\ldots,M_{\l}\}}\abs{\phi_n(x)}} M_{\l}^{-1/2}\sum_{n=1}^{M_{\l}} \abs{c_n{\phi_n(x)} }\\
 & \geq\frac{1}{\max_{m\leq M_{\l}}\norm{\phi_m}_{\infty}} M_{\l}^{-1/2}\abs{\sum_{n=1}^{M_{\l}} c_n{\phi_n(x)} }\\
 &  =\frac{1}{\max_{m\leq M_{\l}}\norm{\phi_m}_{\infty}} M_{\l}^{-1/2}\abs{g(x)}.
\end{split}
\]
Hence,
\[\norm{g}_2 \geq \frac{1}{\max_{m\leq M_{\l}}\norm{\phi_m}_{\infty}} M_{\l}^{-1/2}\norm{g}_{\infty}.\]
Therefore, the optimal bound $C_{\l}$ satisfies
\begin{equation}
\label{eq:temp_dmp3kjdfbg}
  C_{\l} \leq \max_{m\leq M_{\l}}\norm{\phi_m}_{\infty} M_{\l}^{1/2}.  
\end{equation}
Note that for classical Fourier bases we have $\max_{m\leq M_{\l}}\norm{\phi_m}_{\infty}=1$.

Next we analyze the second term in the bound of Proposition \ref{Prop10}.
By  (\ref{eq:temp4y375347l}), we can write in the basis $\{\phi_m\}_{m=1}^{M_{\l}}$
\begin{equation}
   {\bf R}_n^{\l}{\bf S}_n^{\l}\mathbf{c} = \ip{\boldsymbol{\Phi}_n}{\boldsymbol{\Phi}_n} \mathbf{c}. 
   \label{eq:temp_46_5}
\end{equation}
By Lemma \ref{T:quad2},
in probability more than $(1-\d)$,
\begin{equation}
    \label{temp_5475h}
\norm{\ip{\boldsymbol{\Phi}_n}{\boldsymbol{\Phi}_n} - \mathbf{I}}_{F(\CC^{M_{\l}\times M_{\l}})} \leq C' \d^{-1/2} N_n^{-1/2} .
\end{equation}
Here,
\[C' = \frac{M_{\l}}{\sqrt{w_{\min}}}\max_{m\leq M_{\l}}\norm{\phi_m}_{\infty}^2, \]
and $M_{\l} = {\rm dim}(PW(\l))$ as before.

By the fact that the induced $l_2$ norm is bounded by the Frobenius norm, 
\begin{equation}
\label{eq:temp_345hg}
   \norm{P(\l) - R_n^{\l}S_n^{\l}P(\l)}_2 \leq \norm{\ip{\boldsymbol{\Phi}_n}{\boldsymbol{\Phi}_n} - \mathbf{I}}_{F(\CC^{M_{\l}\times M_{\l}})}. 
\end{equation}

We now plug (\ref{eq:temp456bd}), (\ref{eq:temp_dmp3kjdfbg}), (\ref{temp_5475h}) and (\ref{eq:temp_345hg}) in  (\ref{eq:temp3NE4}).
The bound $C$ on the norm of interpolation $\norm{R_n^{\l}}$ is close to 1 by (\ref{eq:temp_46_5}), (\ref{temp_5475h}), and the fact that $\norm{R_n^{\l}}=\norm{(S_n^{\l })^*}=\norm{S_n^{\l}}$. Therefore, for large enough $n$,  $C<2$ (in the same event of (\ref{temp_5475h})). This leads to (\ref{eq:nonAsBoundLast}) in probability more than $1-2\d$, since, in the worst case, not satisfying the bound (\ref{eq:temp456bd}) and not satisfying the bound (\ref{temp_5475h}) are disjoint events.

\subsection{Proof of Claim \ref{ReLU_rexpects1}}
\label{ReLU_rexpects1A}

Let $\e>0$ and $f\in PW(\l)$. Let $g\in PW(\l)$ such that $\norm{f-g}_2<1$. For any $N\in\NN$
\[
\begin{split}
\sum_{\abs{n}> N} n^{1+\k}\abs{\ip{\rho(g)}{\phi_n}}^2 & = \sum_{\abs{n}>N} \abs{n}^{-1+\k} n^{2}\abs{\ip{\rho(g)}{\phi_n}}^2\\
& \leq  N^{-1+\k}\sum_{n=-\infty}^{\infty}  n^{2}\abs{\ip{\rho(g)}{\phi_n}}^2  \\
&  \leq  N^{-1+\k} M_{\l}^2\norm{g}_2^2 \leq N^{-1+\k} M_{\l}^2(\norm{f}_2^2+1).
\end{split}
\]
Similarly,
\[ \sum_{\abs{n}> N} \abs{n}^{1+\k}\abs{\ip{\rho(g)}{\phi_n}}^2  \leq N^{-1+\k} M_{\l}^2(\norm{f}_2^2+1). \]
Now, choose $N=N(\e)$ such that $N^{-1+\k} M_{\l}^2(\norm{f}_2^2+1)< \e/8$.
Moreover, choose $\d < \frac{\e}{2N(\e)^{1+\k}}$. Now, if $\norm{f-g}<\min\{\d,1\}$ we have 
\[\sum_{n=-N}^N n^{1+\k} \abs{\ip{\rho(f)-\rho(g)}{\phi_n}}^2 \leq  N^{1+\k} \sum_{n=-\infty}^{\infty}  \abs{\ip{\rho(f)-\rho(g)}{\phi_n}}^2 = N^{1+\k} \norm{\rho(f)-\rho(g)}_2^2\]
and by the fact that $\rho$ is contractive, 
\[\sum_{n=-N}^N n^{1+\k} \abs{\ip{\rho(f)-\rho(g)}{\phi_n}}^2 \leq  N^{1+\k} \norm{f-g}_2^2 < \e/2. \]
Altogether,
\[
\begin{split}
\norm{\rho(f)-\rho(g)}_{1+\k,2}^2 & \leq \sum_{n=-N}^N \abs{n}^{1+\k} \abs{\ip{\rho(f)-\rho(g)}{\phi_n}}^2 \\
 & \ \ \ + 4 \max\left\{\sum_{\abs{n}> N} \abs{n}^{1+\k}\abs{\ip{\rho(f)}{\phi_n}}^2,\sum_{\abs{n}> N} \abs{n}^{1+\k}\abs{\ip{\rho(g)}{\phi_n}}^2\right\} < \e,
\end{split}
\]
which proves continuity.

\vskip 0.2in
\bibliographystyle{plain}
\bibliography{ref_Transferability,misconception}


\end{document}